\documentclass{article}
\PassOptionsToPackage{numbers, compress}{natbib}
\usepackage{fullpage}
\usepackage{natbib}

\usepackage{thmtools,thm-restate}
\usepackage{mathtools}
\usepackage{amssymb}
\usepackage{bbm}
\usepackage{graphicx}
\usepackage{hyperref}
\usepackage{microtype}
\usepackage{todonotes}
\usepackage{comment}
\usepackage{subcaption}
\usepackage{wrapfig}
\usepackage{nicefrac}
\usepackage{booktabs,siunitx}
\usepackage{cleveref}
\usepackage{amsthm}
\usepackage{amartya_ltx}
\usepackage{paralist, tabularx}
\usepackage{floatrow}
\usepackage{enumitem}
\usepackage{xspace}
\usepackage{tikz}
\usetikzlibrary{arrows}
\usetikzlibrary{shapes.misc}
\usetikzlibrary{positioning}
\tikzset{cross/.style={cross out, draw, 
         minimum size=4,%
         inner sep=0pt, outer sep=0pt}}
     \definecolor{cadmiumgreen}{rgb}{0.0, 0.42, 0.24}

\graphicspath{{./figures/}}
\usepackage{pgfplots}
\pgfplotsset{compat=1.11}
\usepgfplotslibrary{fillbetween}
\usetikzlibrary{intersections}
\usetikzlibrary{patterns}
\usetikzlibrary{calc}

\tikzset{
    right angle quadrant/.code={
        \pgfmathsetmacro\quadranta{{1,1,-1,-1}[#1-1]}     %
        \pgfmathsetmacro\quadrantb{{1,-1,-1,1}[#1-1]}},
    right angle quadrant=1, %
    right angle length/.code={\def\rightanglelength{#1}},   %
    right angle length=2ex, %
    right angle symbol/.style n args={3}{
        insert path={
            let \p0 = ($(#1)!(#3)!(#2)$) in     %
                let \p1 = ($(\p0)!\quadranta*\rightanglelength!(#3)$), %
                \p2 = ($(\p0)!\quadrantb*\rightanglelength!(#2)$) in %
                let \p3 = ($(\p1)+(\p2)-(\p0)$) in  %
            (\p1) -- (\p3) -- (\p2)
        }
    }
}
\pgfdeclarelayer{bg}
\pgfsetlayers{bg,main}

\usepackage{algorithm}
\usepackage{algorithmicx}
\usepackage{algpseudocode}
\newcommand{\AT}{AT\xspace}

\newfloatcommand{capbtabbox}{table}[][\linewidth]
\title{How benign is benign overfitting?}
\author{}
\date{}

\usepackage{authblk}
\author[1,2]{Amartya Sanyal\thanks{amartya.sanyal@cs.ox.ac.uk.}}%
\author[3,4]{Puneet K. Dokania\thanks{puneet@robots.ox.ac.uk}}
\author[1,2]{Varun Kanade\thanks{varunk@cs.ox.ac.uk.}}%
\author[3]{Philip H.S. Torr\thanks{phst@robots.ox.ac.uk}}
\affil[1]{Department of Computer Science, University of Oxford}
\affil[2]{The Alan Turing Institute}
\affil[3]{Department of Engineering Science, University of Oxford}
\affil[4]{Five AI Ltd., UK}

\begin{document}
\maketitle
\begin{abstract}

We investigate two causes for adversarial vulnerability in deep neural
networks: bad data and (poorly) trained models. When trained with SGD, deep
neural networks essentially achieve zero training error, even in the presence
of label noise, while also exhibiting good generalization on natural
test data, something referred to as benign
overfitting~\cite{Bartlett2020,Chatterji2020}. However, these models
are vulnerable to \emph{adversarial attacks}. We identify 
\emph{label noise} as one of the causes for adversarial vulnerability, and
provide theoretical and empirical evidence in support of this. Surprisingly, we
find several instances of label noise in datasets such as MNIST and CIFAR, and
that robustly trained models incur training error on some of these, i.e. they
don't fit the noise. However, removing noisy labels alone does not suffice to
achieve adversarial robustness. Standard training procedures bias neural
networks towards learning ``simple'' classification boundaries, which may be
less robust than more complex ones. We observe that adversarial
training does produce more complex decision boundaries. We conjecture that in part the need
for complex decision boundaries arises from sub-optimal \emph{representation
learning}. By means of simple toy examples, we show theoretically how the
choice of representation can drastically affect adversarial robustness.
\end{abstract}
\section{Introduction}
\label{sec:introduction}

Modern machine learning methods achieve a very high accuracy on wide range of
tasks, e.g. in computer vision, natural language
processing, etc.~\citep{KSH:2012,graves2013speech,he2015delving, Zagoruyko_2016,vaswani2017attention,ren2015faster}, but especially in vision tasks, they have
been shown to be highly vulnerable to small adversarial 
perturbations that are imperceptible to the human
eye~\citep{Dalvi2004,Biggio2018,szegedy2013intriguing,goodfellow2014explaining,Carlini2017,Papernot2016,mosaavi2016}. This
vulnerability poses serious security concerns when these models are deployed in real-world tasks
(cf.~\citep{kurakin2016adversarial,kurakin2016,Papernot2017,Schoenherr2018,Hendrycks2019,li2019adversarial}). A
  large body of research has  been devoted to crafting
defences to protect neural networks from adversarial
attacks~(e.g.~\citep{goodfellow2014explaining,Papernot2015,cisse17a,Xu2017,he2017adversarial,Carlini2017a,tramer2018ensemble,madry2018towards,Zhang2019}).  However, such defences have usually been broken by future attacks~\citep{athalye2018obfuscated, Tramer2020}. This arms race between
attacks and defences suggests that to create a truly robust model
would require a deeper understanding of the source of this vulnerability. 

Our goal in this paper is not to propose new defences, but to provide better
answers to the question: what causes adversarial vulnerability? In doing so, we
also seek to understand how existing methods designed to achieve adversarial
robustness overcome some of the hurdles pointed out by our work.  We identify
two sources of vulnerability that, to the best of our knowledge, have not been
properly studied before: a) memorization of label noise, and b) the implicit
bias in the decision boundaries of neural networks trained with stochastic
gradient descent (SGD). 
\begin{figure}[!htb]
	\centering
  	\begin{subfigure}[b]{0.45\linewidth}
		\def\svgwidth{0.99\linewidth}
    	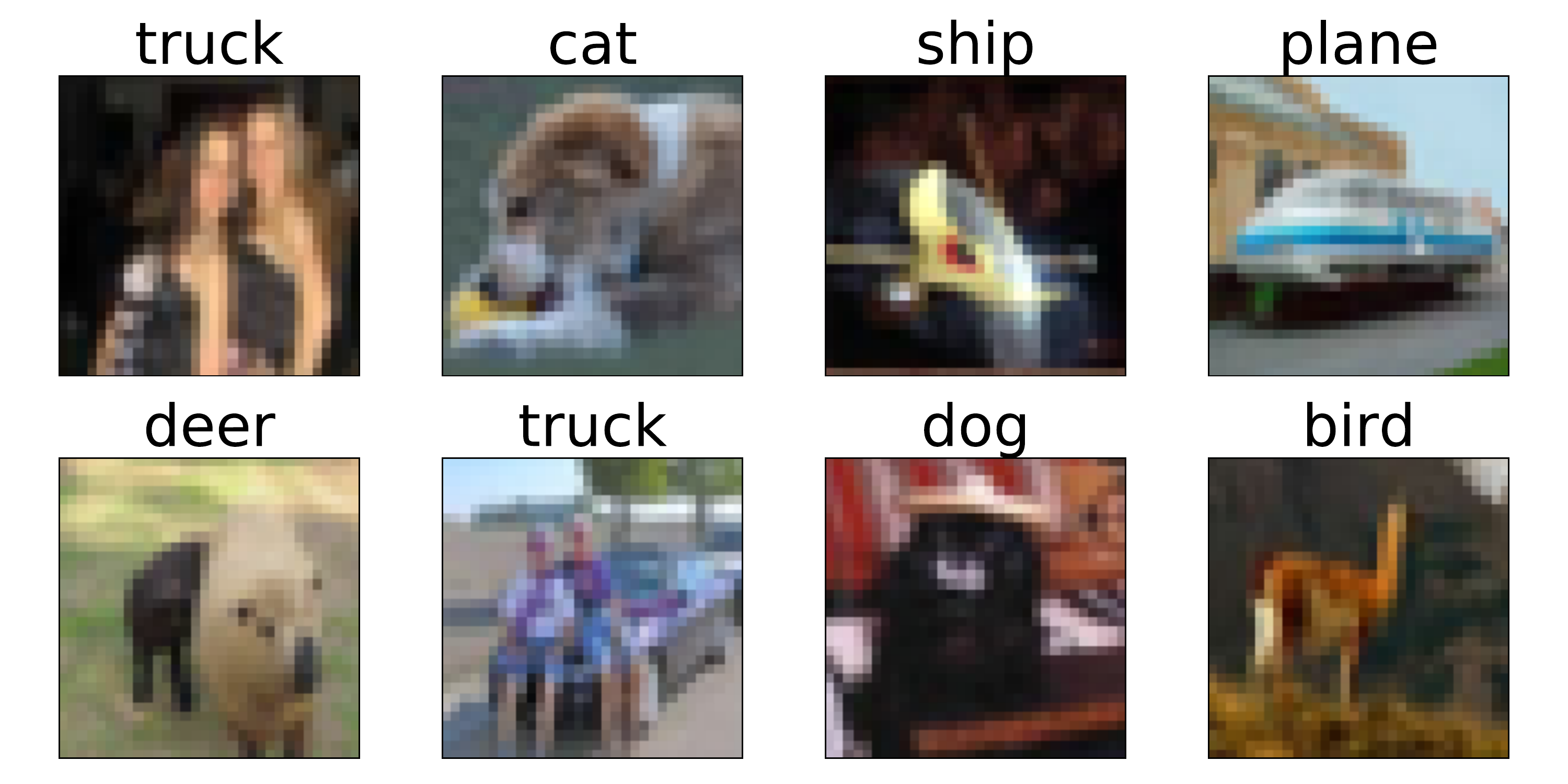
    	\caption*{CIFAR10}
  	\end{subfigure}
  	\begin{subfigure}[b]{0.45\linewidth}
   	\def\svgwidth{0.99\linewidth}
    	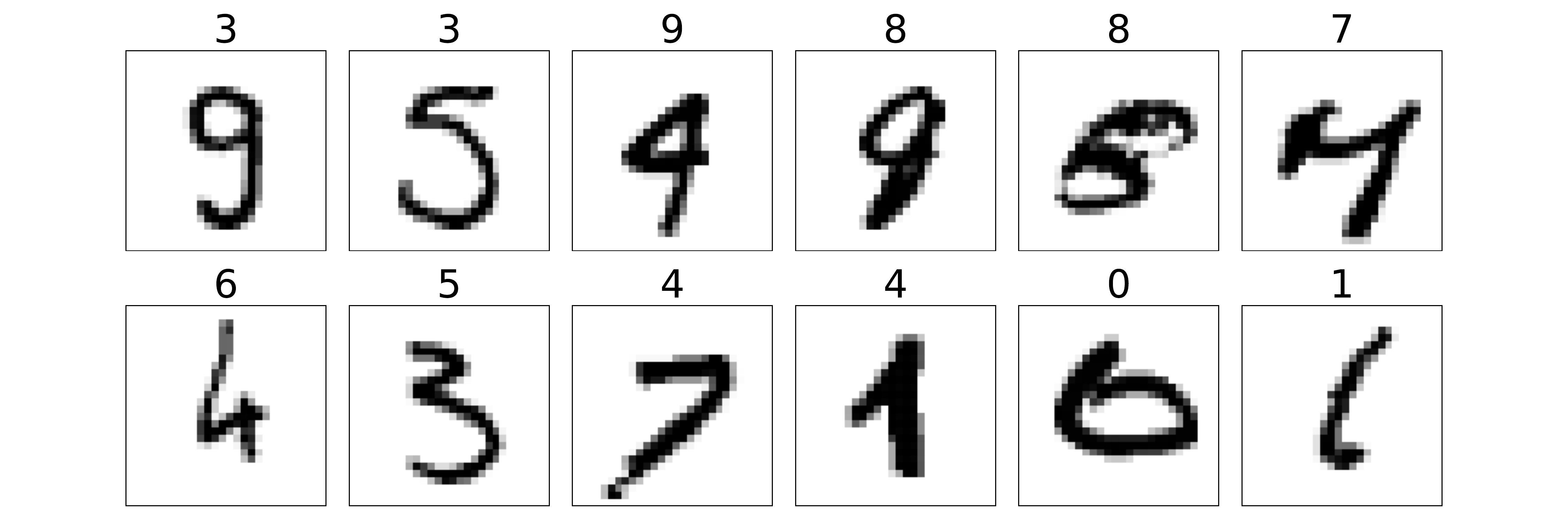
    	\caption*{MNIST}
  	\end{subfigure}
	 \caption{\label{fig:mislabelled_ds}\footnotesize Label Noise in CIFAR10 and MNIST. Text above the image indicates the training set label.}
     \end{figure}
     
First, in the case of label noise, starting with the celebrated work
of~\citet{Zhang2016} it has been observed that neural networks trained with SGD
are capable of memorizing large amounts of label noise. Recent theoretical
work~(e.g.~\citep{Liang2018,belkin18akernel,Belkin2018,Hastie2019,Belkin2019,belkin19a,Bartlett2020,Muthukumar2020,Chatterji2020})
has also sought to explain why fitting training data perfectly (also referred
to as \emph{memorization} or \emph{interpolation}) does not lead to a large
drop in test accuracy, as the classical notion of overfitting might suggest.
We show through simple theoretical models, as well as experiments, that there
are scenarios where label noise does cause significant adversarial
vulnerability, even when high natural~(test) accuracy can be achieved. Surprisingly,
we find that label noise is not at all uncommon in datasets such as MNIST and
CIFAR-10 (see~\cref{fig:mislabelled_ds}). Our experiments show that robust
training methods like Adversarial training~(\AT)~\citep{madry2018towards} and
TRADES~\citep{Zhang2019} produce models that incur training error on at least
some of the noisy examples,\footnote{We manually inspected all training set
errors of these models.} but also on atypical examples from the classes.
Viewed differently, robust training methods are unable
to differentiate between atypical correctly labelled examples~(rare dog) and a
mislabelled example~(cat labelled as dog) and end up not memorizing either;
interestingly, the lack of memorizing these atypical examples has been pointed
out as an explanation for slight drops in test accuracy, as the test set often
contains similarly atypical (or even identical) examples in some
cases~\cite{Feldman2019,Zhang2020}.

Second, the fact that adversarial learning may require more ``complex''
decision boundaries, and as a result may require more data has been pointed out
in some prior
work~\citep{schmidt2018adversarially,pmlr-v97-yin19b,Nakkiran2019,madry2018towards}.
However, the question of decision boundaries in neural networks is subtle as
the network learns a \emph{feature representation} as well as a decision
boundary on top of it. We develop theoretical examples that establish that
choosing one feature representation over another may lead to
\emph{visually} more complex decision boundaries
on the input space, though these are not necessarily more complex in terms of
statistical learning theoretic concepts such as VC dimension. One way
to evaluate whether more meaningful representations lead to better
robust accuracy is to use training data with more fine-grained labels
(e.g. subclasses of a class); for example, one would expect that if
different breeds of dogs are labelled differently the network will
learn features that are relevant to that extra information. We show
both using synthetic data and CIFAR100 that training on fine-grained
labels does increase robust accuracy.

\citet{tsipras2018robustness} and \citet{Zhang2019} have argued that the
trade-off between robustness and accuracy might be unavoidable. However, their
setting involves a distribution that is not robustly separable  by any
classifier. In such a situation there is indeed a trade-off 
between robustness and accuracy. In this paper, we focus on settings where
robust classifiers exist, which is a more realistic scenario for real-world
data. At least for vision, one may well argue that ``humans''
	are robust classifiers, and as a result we would expect that classes are
	well-separated at least in some representation space. In fact, \citet{Yang2020} show that classes are already well-separated in the input space. In
	such situations, there is no need for robustness to be at odds with
	accuracy. A more plausible scenario which we posit, and provide theoretical
	examples in support of, is that the trained models may not be using the
	``right'' representations. Recent empirical work has also established that
	modifying the training objective to favour certain properties in the learned
	representations can automatically lead to improved
robustness~\cite{Sanyal2020LR}. %

\textbf{Summary of Theoretical Contributions} \vspace{-1.3em} \\
\begin{enumerate}[leftmargin=0.5cm,itemsep=-0.3em]
	\item We provide simple sufficient conditions on the data distribution under which any classifier that fits the training data with label noise perfectly is adversarially vulnerable.
	\item The choice of the representation~(and hence the shape of the decision
		boundary) can be important for adversarial accuracy even when it doesn't
		affect natural test accuracy.
	\item There exists data distributions and training algorithms, which when
		trained with (some fraction of) random label noise have the following
		property: (i) using one representation, it is possible to have high
		natural and robust test accuracies but at the cost of having training error;
		(ii) using another representation, it is possible to have no training
		error (including fitting noise) and high test accuracy, but low robust
		accuracy. Furthermore, any classifier that has no training error must
		have low robust accuracy.
\end{enumerate}

\nocite{krizhevsky2009learning}\nocite{LBBH:1998}
The last example shows that the choice of representation matters significantly when it comes to adversarial accuracy, and that memorizing label noise directly leads to loss of robust accuracy. The proofs of the results are not technically complicated and are included in the supplementary material. We have focused on making conceptually clear statements rather than optimize the parameters to get the best possible bounds. We also perform experiments on synthetic data (motivated by the theory), as well as MNIST, CIFAR10/100 to test these hypotheses.

\textbf{Summary of Experimental Contributions} \vspace{-1em} \\
\begin{enumerate}[leftmargin=0.5cm,itemsep=0ex]
	\item  As predicted theoretically, neural nets trained to convergence with
		label noise have greater adversarial vulnerability.
	\item Robust training methods, such as \AT and TRADES that have higher
		robust accuracy, avoid overfitting (some) label noise. This behaviour is
		also partly responsible for their decrease in natural test accuracy.
	\item  Even in the absence of any label noise, methods like \AT and TRADES
		have higher robust accuracy due to more complex decision
		boundaries.
	\item When trained with more fine-grained labels, subclasses
          within each class, leads to higher robust accuracy.
\end{enumerate}

\section{Theoretical Setting}
\label{sec:theoretical-setting}

We develop a simple theoretical framework to demonstrate how overfitting, even
very minimal, label noise causes significant adversarial vulnerability. We also
show how the choice of representation can significantly affect robust accuracy.
Although we state the results for binary classification, they can easily be
generalized to multi-class problems. We formally define the notions of natural
(test) error and adversarial error.

\begin{defn}[Natural and Adversarial Error]\label{defn:gen_risk} 
  For any distribution $\cD$ defined over
  $\br{\vec{x},y}\in\reals^d\times\bc{0,1}$ and any binary classifier
  $f:\reals^d\rightarrow\bc{0,1}$,
  \begin{itemize}
	  \item the \emph{natural} error is 
		  \begin{equation} 
			  \risk{\cD}{f}=\bP_{\br{\vec{x},y}\sim\cD}\bs{f\br{\vec{x}}\neq y},
		  \end{equation} 
	  \item if $\cB_\gamma\br{\vec{x}}$ is a ball of radius
            $\gamma \ge 0$ around $\vec{x}$ under some
            norm\footnote{Throughout, we will mostly use the (most commonly used) $\ell_\infty$ norm, but the results hold for other norms.},
the $\gamma$-\emph{adversarial} error is 
		  \begin{equation} 
			  \radv{\gamma}{f;\cD}=\bP_{\br{\vec{x},y}\sim\cD}\bs{\exists \vec{z}\in\cB_{\gamma}\br{\vec{x}};f\br{\vec{z}}\neq y},
		  \end{equation} 
  \end{itemize}
\end{defn}
In the rest of the section, we provide theoretical results to show the effect of overfitting label
noise and choice of representations~(and hence simplicity of decision
boundaries)  on the robustness of classifiers.

\subsection{Overfitting Label Noise}
\label{sec:overf-label-noise}

The following result provides a sufficient condition under which even a
small amount of label noise causes any classifier that fits the training data
perfectly to have significant adversarial error. Informally,~\Cref{thm:inf-label}
states that if the data distribution has significant probability mass in a
union of (a relatively small number of, and possibly overlapping) balls, each of which has roughly the
same probability mass (cf. Eq.~\eqref{eq:balls_density}), then even a small
amount of label noise renders this entire region vulnerable to adversarial
attacks to classifiers that fit the training data perfectly.

\begin{restatable}{thm}{infectedballs}\label{thm:inf-label}
  Let $c$ be the target classifier, and let $\cD$ be a distribution over
  $\br{\vec{x},y}$, such that $y=c\br{\vec{x}}$ in its support. Using the
  notation $\bP_D[A]$ to denote $\bP_{(\vec{x}, y) \sim \cD}[ \vec{x} \in A]$
  for any measurable subset $A \subseteq \reals^d$, suppose that there
  exist $c_1 \geq c_2 > 0$, $\rho>0$, and a finite set $\zeta \subset
  \reals^d$ satisfying
  \begin{equation}
    \label{eq:balls_density}
	 \bP_\cD\bs{\bigcup_{\vec{s}\in\zeta}\cB_\rho^p\br{\vec{s}}}\ge c_1\quad\text{and}\quad\forall \vec{s}\in\zeta,~\bP_\cD\bs{\cB_{\rho}^p\br{\vec{s}}}\ge\frac{c_2}{\abs{\zeta}}
  \end{equation}
  where $\cB_\rho^p\br{\vec{s}}$ represents a $\ell_p$-ball of radius $\rho$
  around $\vec{s}$. Further, suppose that each of these balls contain points from a single class
  i.e. for all $\vec{s}\in\zeta$, for all
  $\vec{x},\vec{z}\in\cB_{\rho}^p\br{\vec{s}}:
  c\br{\vec{x}}=c\br{\vec{z}}$.

  Let $\cS_m$ be a dataset of $m$ i.i.d. samples  drawn from
  $\cD$, which subsequently has each label flipped independently with probability $\eta$. 
  For any classifier $f$ that \emph{perfectly} fits the training data $\cS_m$
  i.e. $\forall~\vec{x},y\in\cS_m, f\br{\vec{x}}=y$, 
  $\forall \delta>0$ and $m\ge\frac{\abs{\zeta}}{\eta
    c_2}\log\br{\frac{\abs{\zeta}}{\delta}}$, with probability at least $1-\delta$, $\radv{2\rho}{f;\cD}\ge c_1$.
\end{restatable}

The goal is to find a relatively small set $\zeta$ that satisfies the condition as this will mean that even for modest sample sizes, the trained models have significant adversarial error. We remark that it
is easy to construct concrete instantiations of problems that satisfy the
conditions of the theorem, e.g. each class represented by a spherical (truncated) Gaussian
with radius $\rho$, with the classes being well-separated satisfies Eq.~\cref{eq:balls_density}. The main idea of the
proof is that there is sufficient probability mass for points which are within
distance $2\rho$ of a training datum that was mislabelled. We note that the
generality of the result, namely that \emph{any} classifier~(including
neural networks) that fits the training
data must be vulnerable irrespective of its structure, requires a
result like~\Cref{thm:inf-label}. For instance, one could
construct the classifier $h$, where $h(\vec{x}) = c(\vec{x})$, if $(\vec{x}, b)
\not\in \cS_m$ for $b = 0, 1$, and $h(\vec{x}) = y$ if $(\vec{x}, y) \in \cS_m$. Note
that the classifier $h$ agrees with the target $c$ on \emph{every} point of
$\reals^d$ except the mislabelled training examples, and as a result these
examples are the only source of vulnerability. The complete proof is
presented in~\Cref{sec:proof-21}.

There are a few things to note about~\cref{thm:inf-label}. First, the
lower bound on adversarial error applies to any classifier $f$ that
fits the training data $\cS_m$ perfectly and is agnostic to the type
of model $f$ is. Second, for a given $c_1$, there maybe multiple
$\zeta$s that satisfy the bounds in~\cref{eq:balls_density} and the
adversarial risk holds for all of them. Thus, smaller the value of
$\abs{\zeta}$ the smaller the size of the training data it needs to
fit and it can be done by simpler classifiers. Third, if the
distribution of the data  is such that it is concentrated around some
points then for a fixed $c_1,c_2$, a smaller value of $\rho$ would be
required to satisfy~\cref{eq:balls_density} and thus a weaker
adversary~(smaller perturbation budget~$2\rho$) can cause a much
larger adversarial error.

In practice, classifiers exhibit much greater vulnerability than
purely arising from the presence of memorized noisy data. Experiments
in~\Cref{sec:exp-overfit-mislbl}~show how label noise  causes
vulnerability in a toy MNIST model, as well as the full MNIST.

\subsection{Bias towards simpler decision boundaries}
\label{sec:bias-towrds-simpler}
Label noise by itself is not the sole cause for adversarial vulnerability
especially in deep learning models trained with standard optimization
procedures like SGD. A second cause is the \emph{choice of representation} of
the data, which in turn affects the shape of the decision boundary. The choice
of model affects representations and introduces desirable and possibly even
undesirable (cf.~\citep{Liu2018}) invariances;
for example, training convolutional networks are invariant to (some) translations, while
training fully connected networks are invariant to permutations of input
features. This means that fully connected networks can learn even if
the pixels of each training image in the training set are permuted with a fixed
permutation~\citep{Zhang2016}. This invariance is worrying as it means
that such a network can effectively classify a matrix~(or tensor) that
is visually nothing like a real image into an image category. While CNNs don't have this particular
invariance,  as~\citet{Liu2018} shows, location invariance in CNNs
mean that they are unable to predict where in the image a particular
object is.

In
particular, it may be that the decision boundary for robust classifiers needs
to be ``visually'' more \emph{complex} as pointed out in prior work~\citep{Nakkiran2019}, but we
emphasize that this may be because of the choice of representation, and in
particular in standard measures of statistical complexity, such as VC
dimension, this may not be the case. We demonstrate this phenomenon by a simple
(artificial) example even when there is no label noise. Our example in
Section~\ref{sec:repr-learn-pres} combines the two causes and shows how
classifiers that are translation invariant may be worse for adversarial
robustness.

\begin{figure}[t]
  \begin{subfigure}[t]{0.67\linewidth}
    \scalebox{0.9}{\begin{tikzpicture}
      \draw (1.,0) -- (9,0); \filldraw [gray] (2,0) circle (2pt);
      \filldraw [gray] (4,0) circle (2pt); \filldraw [gray] (6,0)
      circle (2pt); \filldraw [gray] (8,0) circle (2pt);

      \draw[|-|, cadmiumgreen, line width=1pt] (1.35,0.5) --
      (1.75,0.5); \draw[|-|, cadmiumgreen, line width=1pt] (1.85,0.5)
      -- (2.35,0.5); \draw[|-|, cadmiumgreen, line width=1pt]
      (2.45,0.5) -- (2.65,0.5);

      \draw[|-|, cadmiumgreen, line width=1pt] (1.35,1.5) --
      (2.65,1.5);

      \draw (1.4,0 ) node[cross,red,line width=1pt] {}; \draw (1.6,0 )
      node[cross,red,line width=1pt] {}; \draw (2.2,0 )
      node[cross,red,line width=1pt] {}; \draw (2.6,0 )
      node[cross,red,line width=1pt] {};
      \draw[blue, line width=1pt] (1.8,0 ) circle (2.5pt); \draw[blue,
      line width=1pt] (2.4,0 ) circle (2.5pt); \node at (2,-0.5)
      {001};

      \draw[|-|, cadmiumgreen, line width=1pt] (3.55,0.5) --
      (3.65,0.5);

      \draw (2+1.4,0 ) node[cross,blue,line width=1pt] {}; \draw
      (2+1.8,0 ) node[cross,blue,line width=1pt] {}; \draw (2+2.2,0 )
      node[cross,blue,line width=1pt] {};
      \draw (2+2.4,0 ) node[cross,blue,line width=1pt] {}; \draw
      (2+2.6,0 ) node[cross,blue,line width=1pt] {}; \draw[red, line
      width=1pt] (2+1.6,0 ) circle (2.5pt); \node at (4,-0.5) {010};

      \draw[|-|, cadmiumgreen, line width=1pt] (2+2+1.35,1.5) --
      (2+2+2.65,1.5);

      \draw[|-|, cadmiumgreen, line width=1pt] (2+2+1.35,0.5) --
      (2+2+1.55,0.5); \draw[|-|, cadmiumgreen, line width=1pt]
      (2+2+1.65,0.5) -- (2+2+2.35,0.5); \draw[|-|, cadmiumgreen, line
      width=1pt] (2+2+2.45,0.5) -- (2+2+2.65,0.5);

      \draw (2+2+1.8,0 ) node[cross,red,line width=1pt] {}; \draw
      (2+2+1.4,0 ) node[cross,red,line width=1pt] {}; \draw (2+2+2.2,0
      ) node[cross,red,line width=1pt] {};
      \draw (2+2+2.6,0 ) node[cross,red,line width=1pt] {};
      \draw[blue, line width=1pt] (2+2+1.6,0 ) circle (2.5pt);
      \draw[blue, line width=1pt] (2+2+2.4,0 ) circle (2.5pt); \node
      at (6,-0.5) {011};

      \draw[|-|, cadmiumgreen, line width=1pt] (2+2+2+1.35,1.5) --
      (2+2+2+2.65,1.5);

      \draw[|-|, cadmiumgreen, line width=1pt] (2+2+2+1.35,0.5) --
      (2+2+2+2.15,0.5); \draw[|-|, cadmiumgreen, line width=1pt]
      (2+2+2+2.25,0.5) -- (2+2+2+2.65,0.5);

      \draw (2+2+2+1.4,0 ) node[cross,red,line width=1pt] {}; \draw
      (2+2+2+1.6,0 ) node[cross,red,line width=1pt] {}; \draw
      (2+2+2+1.8,0 ) node[cross,red,line width=1pt] {}; \draw
      (2+2+2+2.4,0 ) node[cross,red,line width=1pt] {}; \draw
      (2+2+2+2.6,0 ) node[cross,red,line width=1pt] {}; \draw[blue,
      line width=1pt] (2+2+2+2.2,0 ) circle (2.5pt); \node at (8,-0.5)
      { 100};

      \node[text width=1.5cm,color=cadmiumgreen,line width=2pt] at
      (0.5,0.5) {Union of Intervals}; \node[text
      width=1.5cm,color=cadmiumgreen,line width=2pt] at (0.5,1.5)
      {Parity Classifier};
    \end{tikzpicture}}\caption{Both Parity and Union of Interval
    classifier predicts {\color{red} red} if inside any
    {\color{cadmiumgreen}~green} interval and
    {\color{blue} blue} if outside all intervals. The $\times$-es are
 correctly labelled and the $\circ$-es are mis-labelled points. Reference integer points on the line labelled in \emph{binary}.}\label{fig:thm-3}
  \end{subfigure}\hfill
  \begin{subfigure}[t]{0.29\linewidth}
    \begin{tikzpicture}
      \draw[thin, dashed] (-0.5,0) -- (3,0);
      \draw[thin, dashed] (-0.5,1) -- (3,1);
      \draw[thin, dashed] (2,-0.6) -- (2,1.6);
      \draw[thin, dashed] (1,-0.6) -- (1,1.6); 
      \draw[thin, dashed] (0,-0.6) -- (0,1.6);

      \filldraw[color=red!60, fill=red!5, very thick](0,0) circle
      (.15);
      \filldraw[color=red!60, fill=red!5, very thick](1,1)
      circle (.15);
      \filldraw[color=blue!60, fill=blue!5, very
      thick](1,0) circle (.15);
      \filldraw[color=blue!60, fill=blue!5,
      very thick](2,1) circle (.15);
      \draw[orange, dashed, very thick, name  path=plane] (-0.1, -0.6) -- (1.5, 1) -- (2.1, 1.6);
      \draw[cadmiumgreen, thick] (0.5, -0.5) -- (0.5, 0.5) -- (1.5, 0.5) --
      (1.5, 1.5);
     \draw[|-|, cadmiumgreen, line width=1pt]
      (1.5,0.3) -- (1.9,0.3);
      \draw[|-|, orange, line width=1pt]
     (1.5,-0.2) -- (1.9,-0.2);

       \filldraw [gray] (0,0)
      circle (1pt);
      \filldraw [gray] (0,1)
      circle (1pt);
      \filldraw [gray] (1,0)
      circle (1pt);
      \filldraw [gray] (1,1)
      circle (1pt);
      \filldraw [gray] (2,0)
      circle (1pt);
      \filldraw [gray] (2,1)
      circle (1pt);
      \node[text width=2.3cm,color=cadmiumgreen] at (3.2,0.3) {\scriptsize Parity Classifier};
      \node[text width=2.3cm,color=orange] at (3.2,-0.2) {\scriptsize Linear Classifier};
    \end{tikzpicture}
   \caption{Robust generalization needs more complex boundaries}
   \label{fig:complex_simple}
  \end{subfigure}\caption{Visualization of the distribution and
    classifiers used in the Proof
    of~\Cref{thm:parity_robust_repre_all,thm:repre-par-inter}.~The {\color{red}Red}
  and {\color{blue}Blue} indicate the two classes.} 
\end{figure}
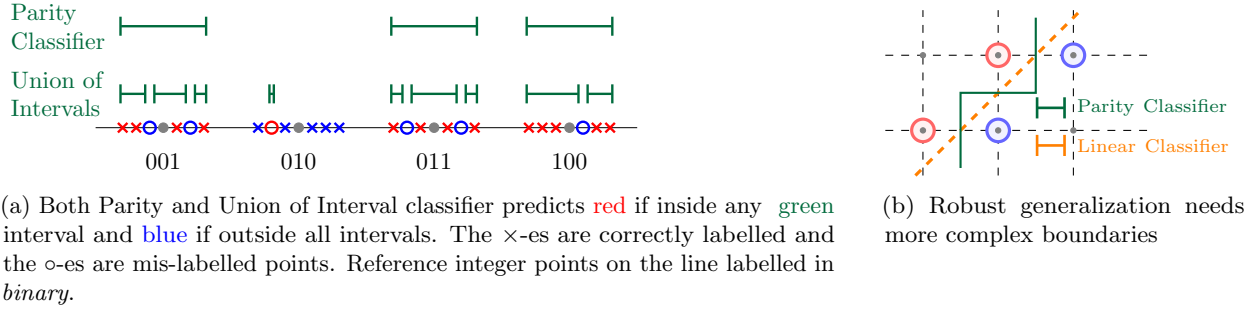

\begin{restatable}{thm}{parityrobustrepre}\label{thm:parity_robust_repre_all}
	For some universal constant $c$, and any $0 < \gamma_0 < 1/\sqrt{2}$, there exists a
  family of distributions $\cD$ defined on $\cX\times\bc{0,1}$ where
  $\cX\subseteq\reals^2$ such that for all distributions $\cP\in\cD$, and
  denoting by $\cS_m =\bc{\br{\vec{x}_1,y_1},\cdots,(\vec{x}_m,y_m)}$ a sample of
  size $m$ drawn i.i.d. from $\cP$, 
  \vspace{-0.5em}
  \begin{enumerate}[itemsep=-0.3em,leftmargin=*]
	  \item[(i)] For any $m \geq 0$, $\cS_m$ is linearly separable i.e.,
		  $\forall(\vec{x}_i, y_i) \in \cS_m$, there exist $\vec{w}\in\reals^2,
		  w_0\in\reals$ s.t.  $y_i\br{\vec{w}^\top\vec{x}_i+w_0}\ge 0$.
		  Furthermore, for every $\gamma > \gamma_0$, any linear separator $f$
		  that perfectly fits the training data $\cS_m$ has $\radv{\gamma}{f;
			  \cP} \geq 0.0005$, even though $\risk{\cP}{f}
                        \rightarrow 0$ as $m \rightarrow \infty $.
	  \item[(ii)] There exists a function class $\cH$ such that for some $m \in
		  O(\log(\delta^{-1}))$, any $h \in \cH$ that perfectly fits the $\cS_m$,
		  satisfies with probability at least $1 - \delta$, $\risk{\cP}{h} = 0$
		  and $\radv{\gamma}{h; \cP} = 0$, for any $\gamma \in [0, \gamma_0 + 1/8]$.
  \end{enumerate}
\end{restatable}

A complete proof of this result appears in~\Cref{sec:proof-22}, but
first, we provide a sketch of the key idea here.%
The distributions in family $\cD$ will
be supported on balls of radius at most $1/\sqrt{2}$ on the integer lattice in
$\reals^2$. The \emph{true} class label for any point $\vec{x}$ is provided by
the parity of $a + b$, where $(a, b)$ is the lattice point closest to
$\vec{x}$. However, the distributions in $\cD$ are chosen to be such that there
is also a linear classifier that can separate these classes, e.g. a
distribution only supported on balls centered at the points $(a, a)$ and $(a,
a+1)$ for some integer $a$~(See~\Cref{fig:complex_simple}). \emph{Visually} learning the classification problem
using the parity of $a + b$ results in a seemingly more complex decision
boundary, a point that has been made earlier regarding the need for more
complex boundaries to achieve adversarial robustness~\citep{Nakkiran2019,degwekar19a}. However, it is worth
noting that this complexity is not rooted in any \emph{statistical theory},
e.g. the VC dimension of the classes considered in
Theorem~\ref{thm:parity_robust_repre_all} is essentially the same (even lower for $\cH$
by $1$). This \emph{visual} complexity arises purely due to the fact
that the linear classifier looks at a geometric representation of the
data whereas the parity classifier looks at the binary representation
of the sum of the nearest integer of the coordinates. In the case of
neural networks, recent works~\citep{kamath2020invariance} have indeed
provided empirical results to support that excessive invariance
(eg. rotation invariance) increases adversarial error.

\subsection{Representation Learning in the presence of label noise}
\label{sec:repr-learn-pres}

In this section, we show how both causes of vulnerability can
interact. Informally, we show that if the \emph{correct}
representation is used, then in the presence of label noise, it will
be impossible to fit the training data perfectly, but the 
classifier that best fits the training data,%
\footnote{This is referred to as the Empirical Risk Minimization (ERM) in the statistical learning theory literature.}
will have good test accuracy and adversarial accuracy. However, using an
``incorrect'' representation, we show that it is possible to find a classifier
that has no training error, has good test accuracy, but has high
\emph{adversarial error}. We posit this as an (partial) explanation of why
classifiers trained on real data (with label noise, or at least atypical
examples) have good test accuracy, while still being vulnerable to adversarial
attacks.  

\begin{restatable}{thm}{robustpossibleful}[Formal version of~\Cref{thm:repre-par-inter}]~\label{thm:repre-par-inter}
	For any $n\in\bZ_+$, there exists a family of distributions $\cD^n$ over $\reals \times \{0, 1\}$ and function classes $\cC,\cH$, such that for any $\cP$ from 
	$\cD^n$, and for any $0 < \gamma < 1/4$, and $\eta \in (0, 1/2)$ if $\cS_m =
        \{(\vec{x}_i, y_i)\}_{i=1}^m$ denotes a sample of size $m$ where
        \[m=\bigO{\mathrm{max}\bc{
      n\log{\frac{n}{\delta}} 
   \br{\frac{\br{1-\eta}}{\br{1-2\eta}^2}+1},
   \frac{n}{\eta\gamma^2} 
   \log\br{\frac{n}{\gamma\delta}}}}\]
drawn from $\cP$, and if $\cS_{m, \eta}$ denotes the sample where each label is flipped independently with probability $\eta$.
  \begin{enumerate}[labelsep=-0.3em,leftmargin=*]
	  \item[(i)]~~the classifier $c \in \cC$ that minimizes the training error on $\cS_{m, \eta}$, has $\risk{\cP}{c} = 0$ and $\radv{\gamma}{c; \cP} = 0$ for $0 \leq \gamma < 1/4$. 
	  \item[(ii)]~~there exist $h \in \cH$, $h$ has zero training error on $\cS_{m, \eta}$, and $\risk{\cP}{h} = 0$. However, for any $\gamma > 0$, and for any $h \in \cH$ with zero training error on $\cS_{m, \eta}$, $\radv{\gamma}{h; \cP} \geq 0.1$. 
  \end{enumerate}\vspace{-0.5em}
  Furthermore, the required $c \in \cC$ and $h \in \cH$ above can be computed in $\bigO{\poly{n},\poly{\frac{1}{\frac{1}{2}-\eta}},
  \poly{\frac{1}{\delta}}}$ time.
\end{restatable}

\begin{figure}[t]
    \centering
    \def\svgwidth{0.75\columnwidth}
    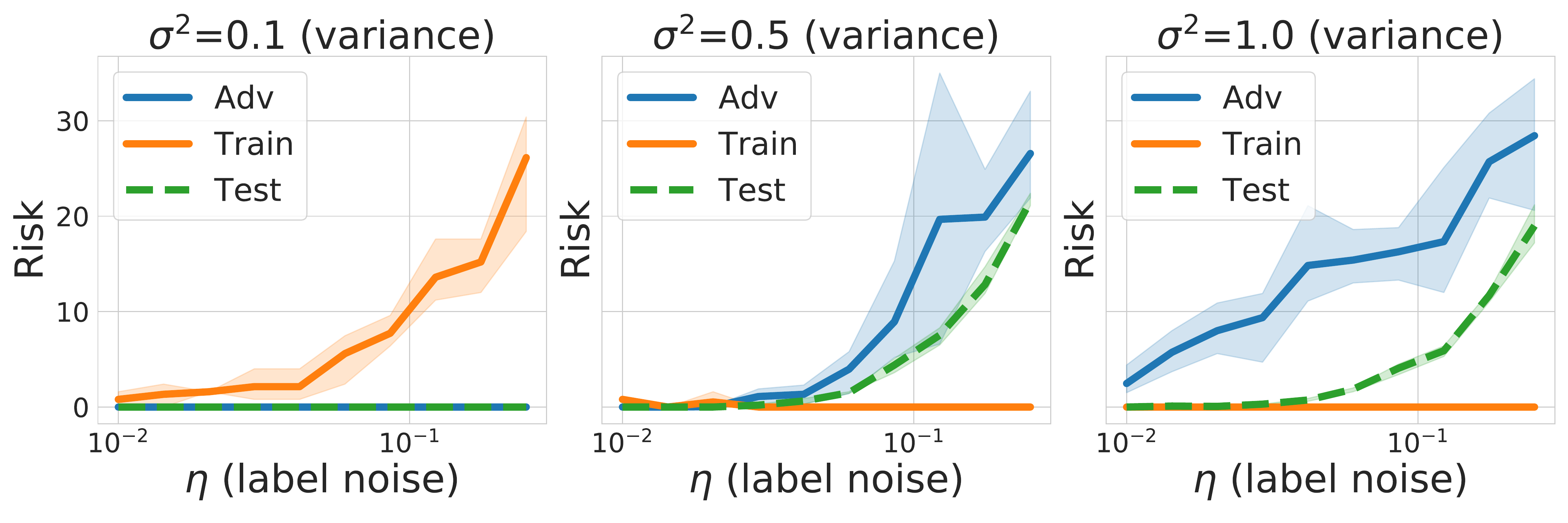
  \caption{ Adversarial
    error increases with  label
    noise~($\eta$) if training error is $0$. Shaded region shows
    $95\%$ confidence interval.} 
  \label{fig:risk_vs_noise}
\end{figure}
We sketch the proof here and present the complete the proof in~\Cref{sec:proof23}; as
in~\Cref{sec:bias-towrds-simpler} we will make use of parity functions, though
the key point is the representations used. Let $\cX = [0, N]$, where $N = 2^n$,
we consider distributions that are supported on intervals $(i - 1/4, i + 1/4)$
for $i \in \{1, \ldots, N-1 \}$~(See~\Cref{fig:thm-3}), but any
such distribution will only have a small number, $O(n)$, intervals on
which it is supported. The \emph{true} class label is given by a function that
depends on the parity of some hidden subsets $S$ of bits in the
bit-representation of the closest integer $i$,  e.g. as
in~\Cref{fig:thm-3} if $S = \{0, 2\}$, then only the least significant and
the third least significant bit of $i$ are examined and the class label is
$1$ if an odd number  of them are $1$ and $0$ otherwise. Despite the noise,
the \emph{correct} label on any interval can be guessed by using the
majority vote and as a result, the correct parity learnt using Gaussian
elimination. (This corresponds to the class $\cC$ in
~\Cref{thm:repre-par-inter}.) On the other hand it is also possible to
learn the function as a union of intervals, i.e. find intervals, $I_1, I_2,
\ldots, I_k$ such that any point that lies in one of these intervals is
given the label $1$ and any other point is given the label $0$. By choosing
intervals carefully, it is possible to fit \emph{all the training data},
including noisy examples, but yet not compromise on \emph{test accuracy}
(Fig.~\ref{fig:thm-3}). Such a classifier, however, will be vulnerable to
adversarial examples by applying Theorem~\ref{thm:inf-label}.  A classifier
such as union of intervals ($\cH$ in Theorem~\ref{thm:repre-par-inter}) is
translation-invariant, whereas the parity classifier is not.  This suggests
that using classifiers, such as neural networks, that are designed to have
too many built-in invariances might hurt its robustness accuracy.

\section{Experimental results}
\label{sec:experimental-results}

In~\Cref{sec:theoretical-setting}, we provided three theoretical settings to
highlight how fitting label noise and sub-optimal representation
learning~(leading to seemingly simpler decision boundaries) hurts adversarial
robustness.  In this section, we provide empirical evidence on synthetic
data inspired by the theory and on the standard datasets:
MNIST~\citep{LBBH:1998}, CIFAR10, and CIFAR100~\citep{krizhevsky2009learning}
to support the theory.

\subsection{Overfitting label noise decreases adversarial accuracy}

\label{sec:exp-overfit-mislbl}

We design a simple binary classification problem, \emph{toy-MNIST}, and show
that when fitting a complex classifier on a training dataset with label noise,
adversarial vulnerability increases with the amount of label noise, and that
this vulnerability is caused by the label noise. The problem is constructed by
selecting two random images from MNIST: one ``0'' and one ``1''. Each
training/test example is generated by selecting one of these images and adding
i.i.d. Gaussian noise sampled from $\cN\br{0,\sigma^2}$. We create a training
dataset of $4000$ samples by sampling uniformly from either class. Finally,
$\eta$ fraction of the training data is chosen randomly and its labels are
flipped. 
We train a neural network with four fully connected layers followed by a
softmax layer and minimize the cross-entropy loss using an SGD optimizer until
the training error becomes zero. Then, we attack this network with a
~\emph{strong} $\ell_\infty$ PGD adversary~\citep{madry2018towards}
with $\epsilon=\frac{64}{255}$ for 
$400$ steps with a step size of $0.01$. 

In~\Cref{fig:risk_vs_noise}, we plot the adversarial error, test
error and training error as the amount of label noise $(\eta)$ varies, for
three different values of sample variance~($\sigma^2$). For low values of
$\sigma^2$, the training data from each class is all concentrated around the
same point; as a result these models are unable to memorize the label noise and
the training error is high. In this case, over-fitting label noise is
impossible and the test error, as well as the adversarial error, is low.
However, as $\sigma^2$ increases, the neural network is flexible enough to use
the ``noise component'' to extract features that allow it to memorize label
noise and fit the training data perfectly. This brings the training error down
to zero, while causing the test error to increase, and the adversarial error
even more so. This is in line with Theorem~\ref{thm:inf-label}. The case when
$\sigma^2=1$ is particularly striking as it exhibits a range of values of
$\eta$ for which test error remains very close to 0 even as the adversarial
error jumps considerably. This confirms the hypothesis that benign overfitting
may not be so benign when it comes to adversarial error.

\begin{figure}[t]
   \begin{subfigure}[t]{0.19\linewidth}
    \centering \def\svgwidth{0.99\columnwidth}
    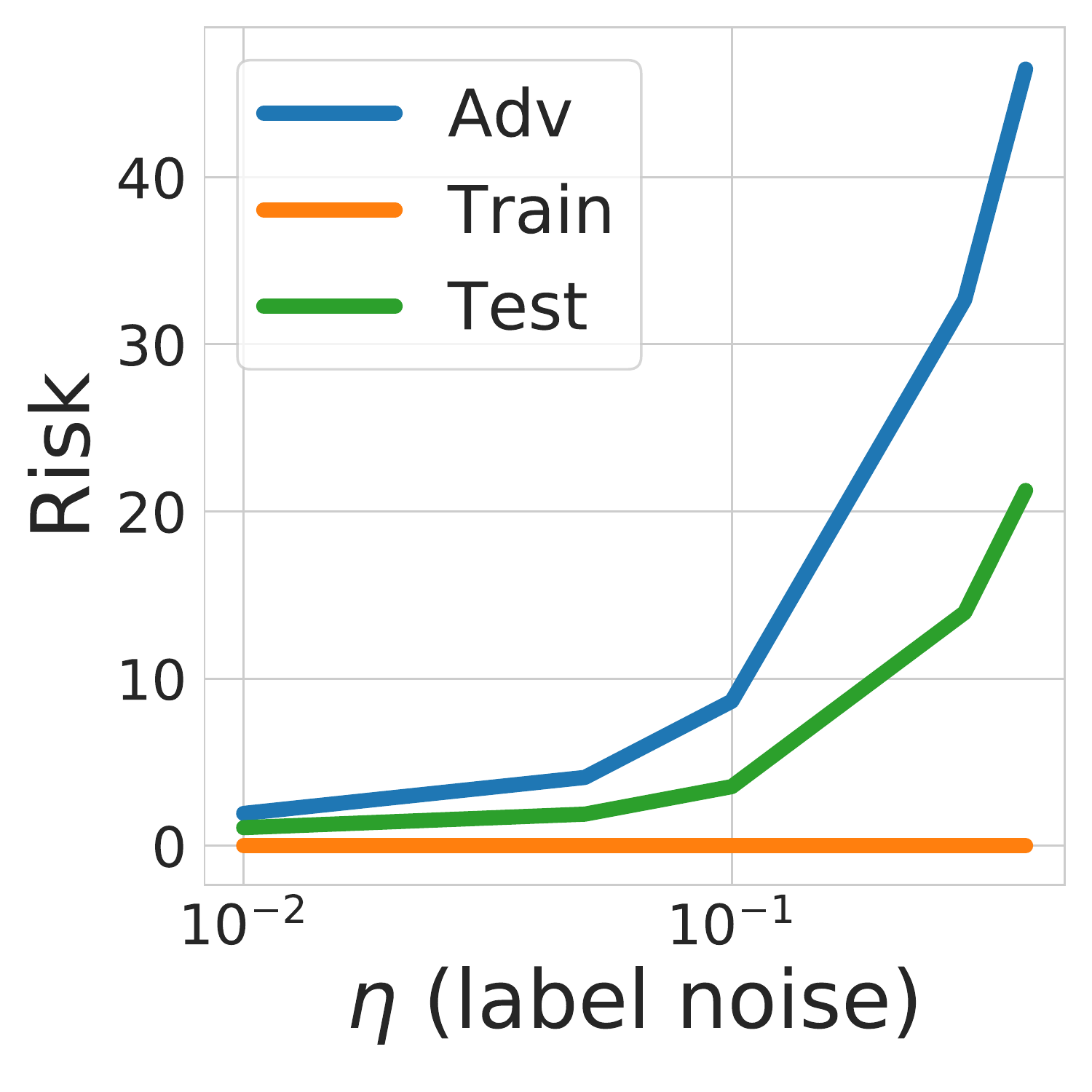
    \caption{$\epsilon=0.01$}
    \label{fig:mnist_lbl_noise_adv}
  \end{subfigure}
   \begin{subfigure}[t]{0.19\linewidth}
    \centering \def\svgwidth{0.99\columnwidth}
    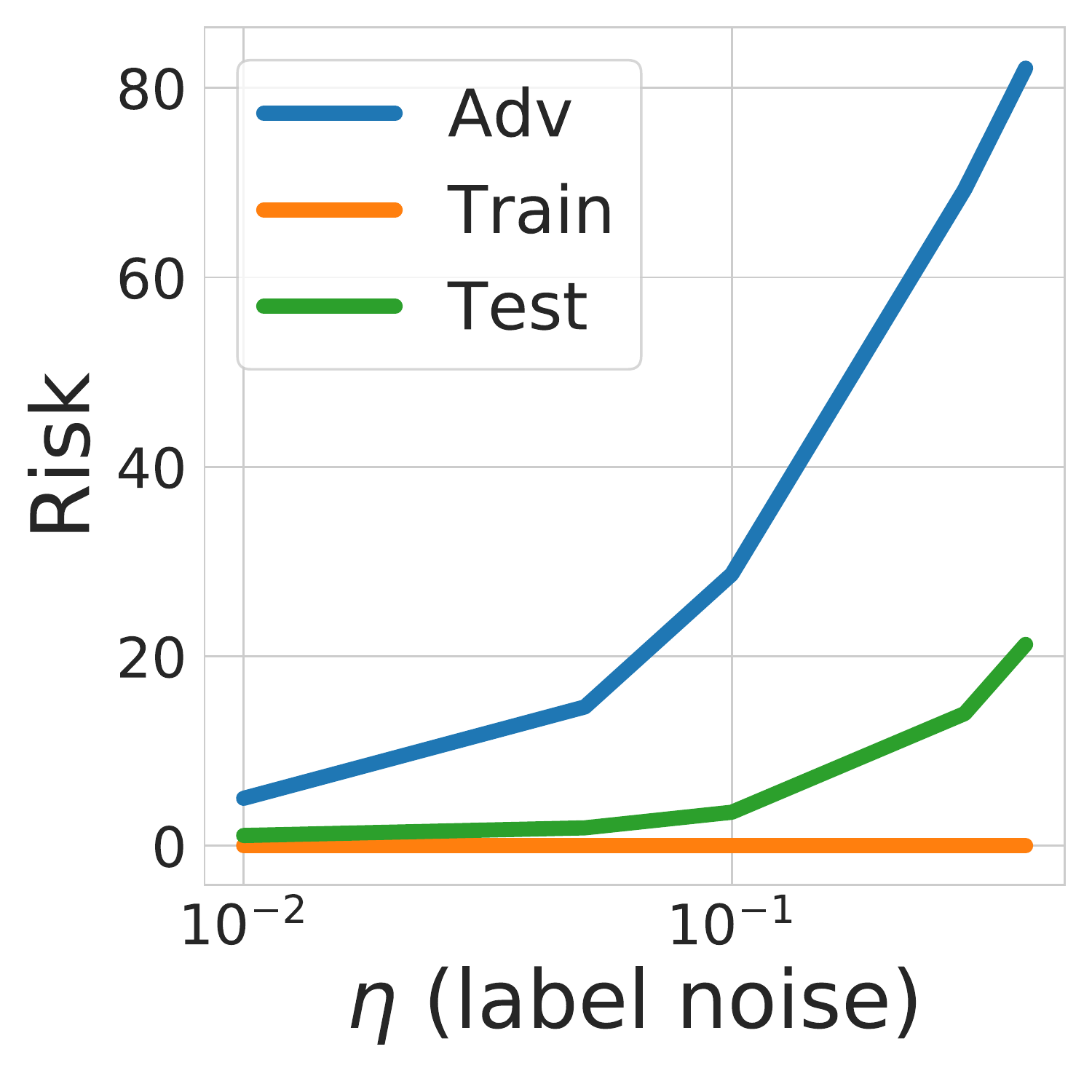
    \caption{$\epsilon=0.025$}
    \label{fig:mnist_lbl_noise_adv}
  \end{subfigure}
  \begin{subfigure}[t]{0.19\linewidth}
    \centering \def\svgwidth{0.99\columnwidth}
    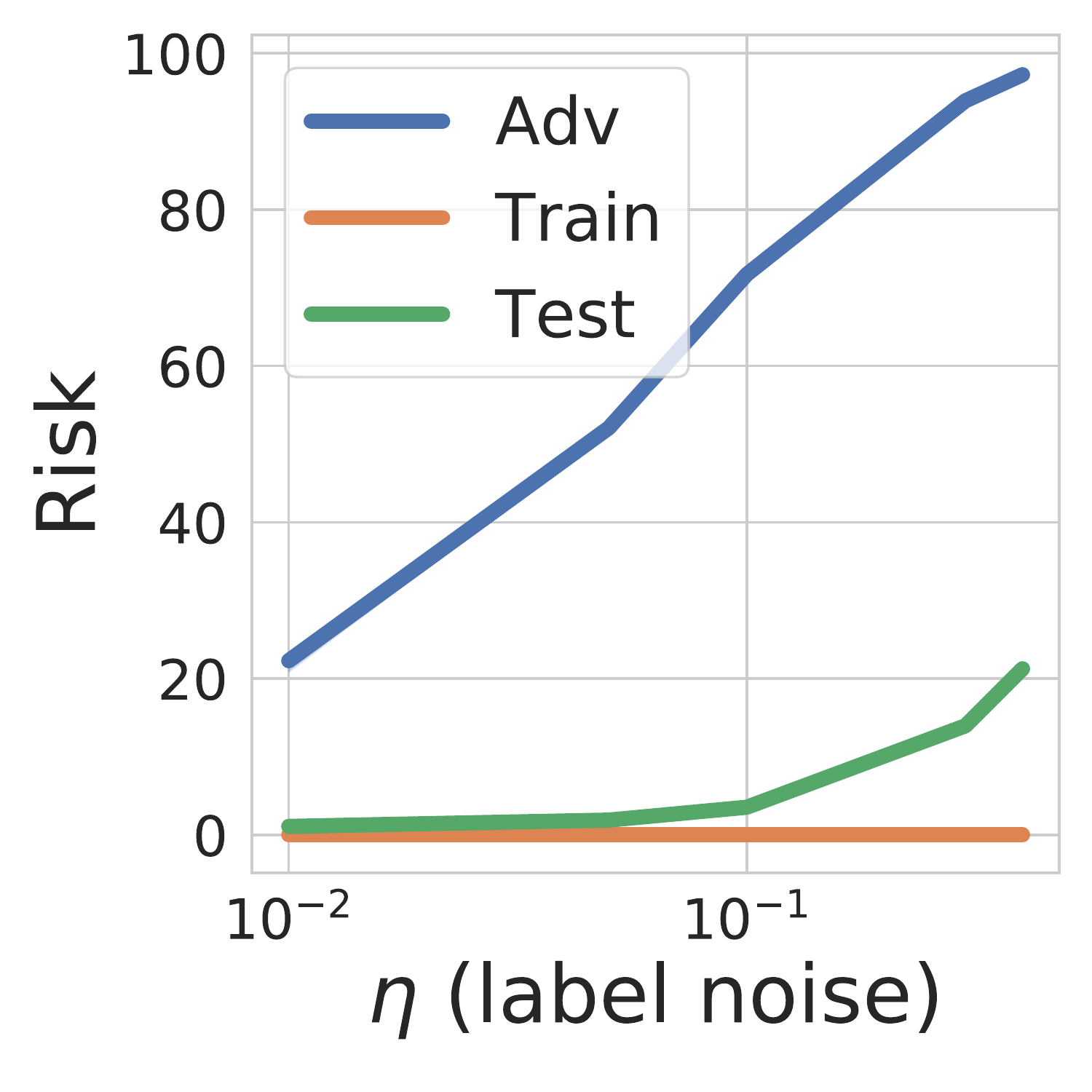
    \caption{$\epsilon=0.05$}
    \label{fig:mnist_lbl_noise_adv}
  \end{subfigure}
   \begin{subfigure}[t]{0.19\linewidth}
    \centering \def\svgwidth{0.99\columnwidth}
    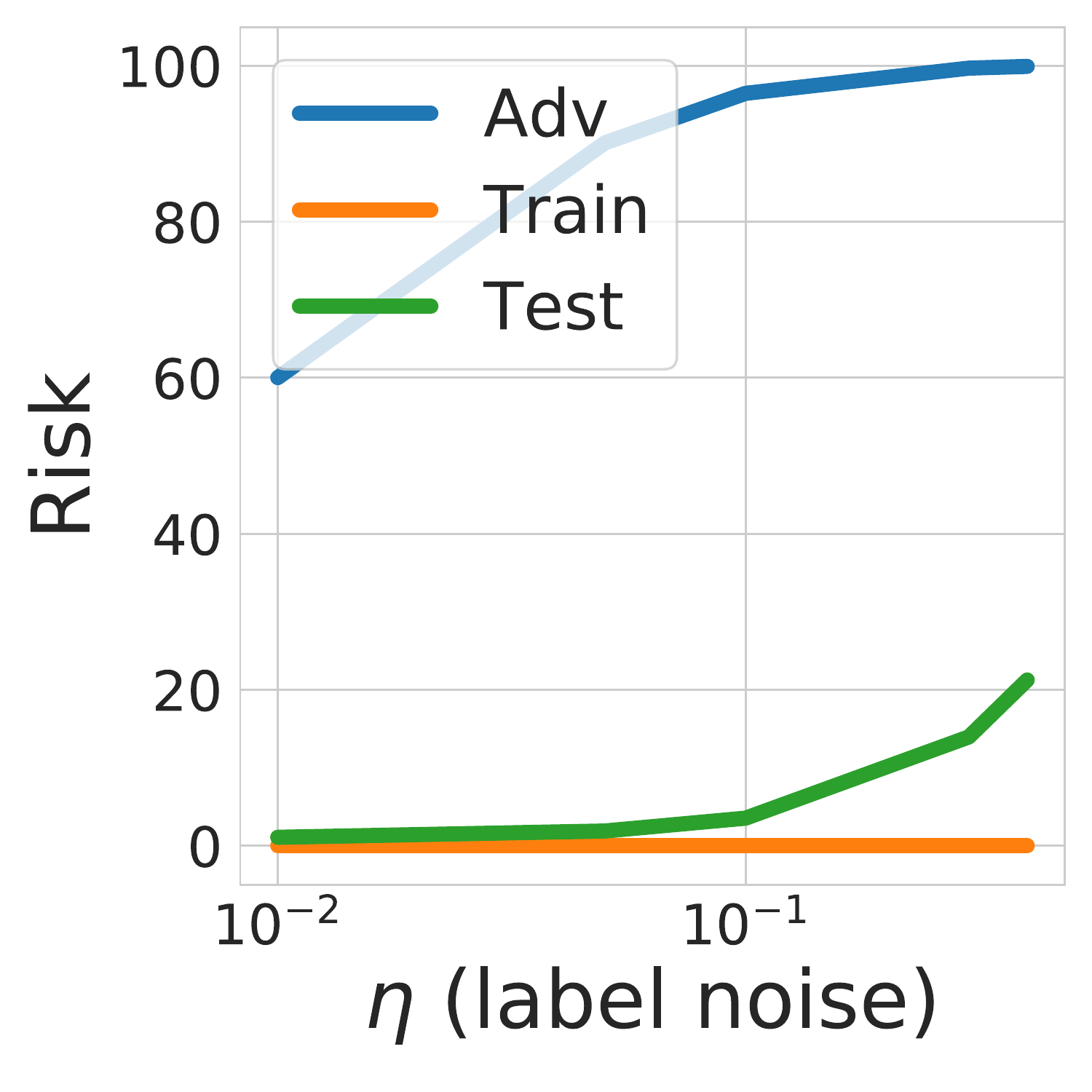
    \caption{$\epsilon=0.1$}
    \label{fig:mnist_lbl_noise_adv}
  \end{subfigure}
   \begin{subfigure}[t]{0.19\linewidth}
    \centering \def\svgwidth{0.99\columnwidth}
    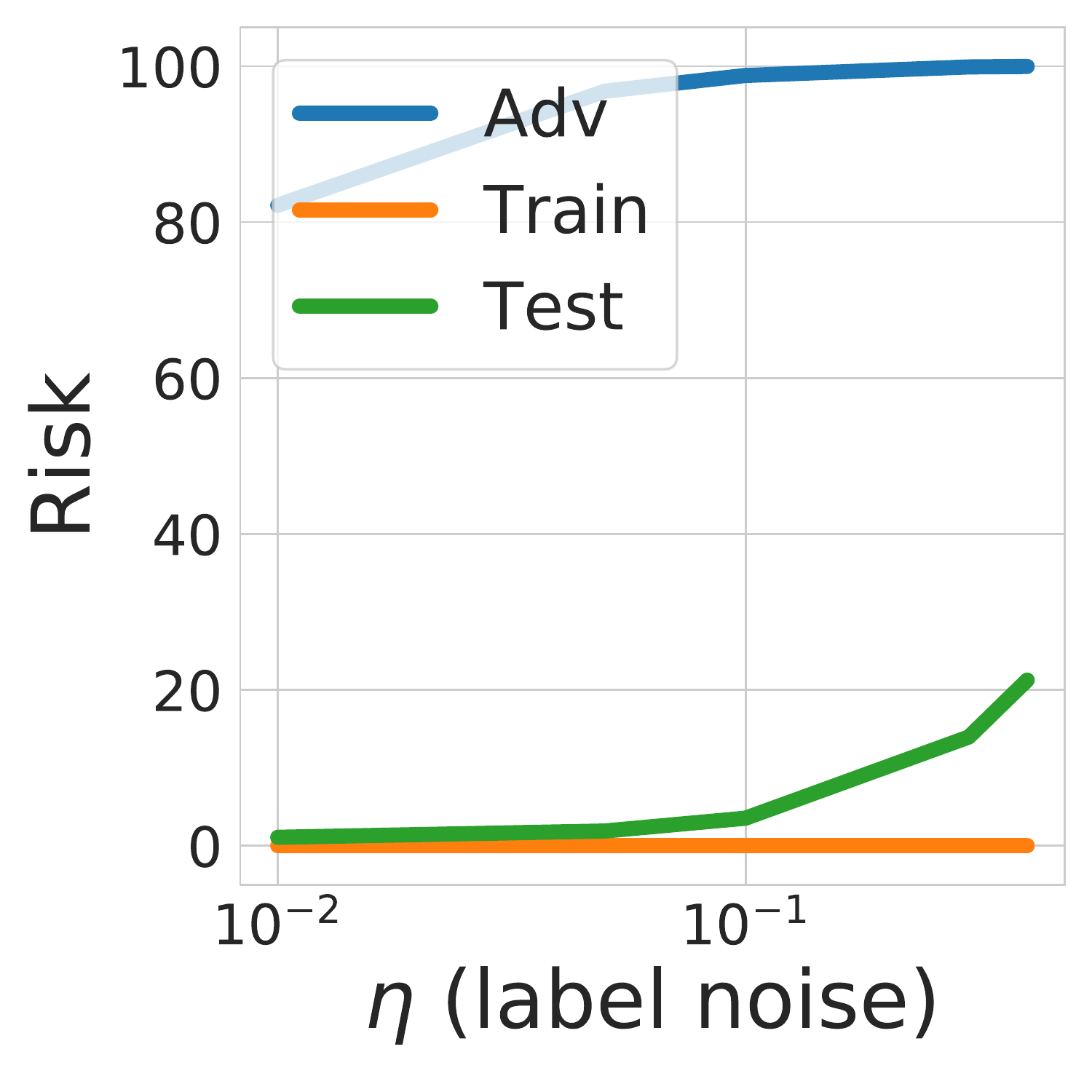
    \caption{$\epsilon=0.2$}
  \end{subfigure}
  \caption{Shows the adversarial for the
    full MNIST dataset for varying levels of adversarial
    perturbation. There is negligible variance between runs and thus
    the shaded region showing the confidence interval is invisible.} 
  \label{fig:risk_vs_noise_app}
\end{figure}

We perform a similar experiment on the full MNIST dataset trained on a ReLU
network with 4 convolutional layers, followed by two fully connected layers.
The first four convolutional layers have 
$32,64, 128, 256$ output filters and $3,4,3,3$ sized
kernels respectively. This is followed by a fully connected layers
with a hidden dimension of $1024$. For varying values of $\eta$, for a
uniformly randomly chosen $\eta$ fraction of the training data we
assigned the class label randomly.  The network is optimized with SGD
with a batch size of $128$, learning rate of $0.1$ for $60$ epochs and the
learning rate is decreased to $0.01$ after 50 epochs.

We compute the natural test accuracy and the adversarial test
accuracy for when the network is attacked with a $\ell_\infty$ bounded
PGD adversary for varying perturbation budget $\epsilon$, with a step
size of $0.01$ and for $20$ steps. \Cref{fig:risk_vs_noise_app} shows that
the effect of over-fitting label noise is even more clearly visible here; for
the same PGD adversary the adversarial error jumps sharply with increasing
label noise, while the growth of natural test error is much slower.

\begin{figure}[t]
  \begin{subfigure}[t]{0.16\linewidth}
    \centering \def\svgwidth{0.99\columnwidth}
    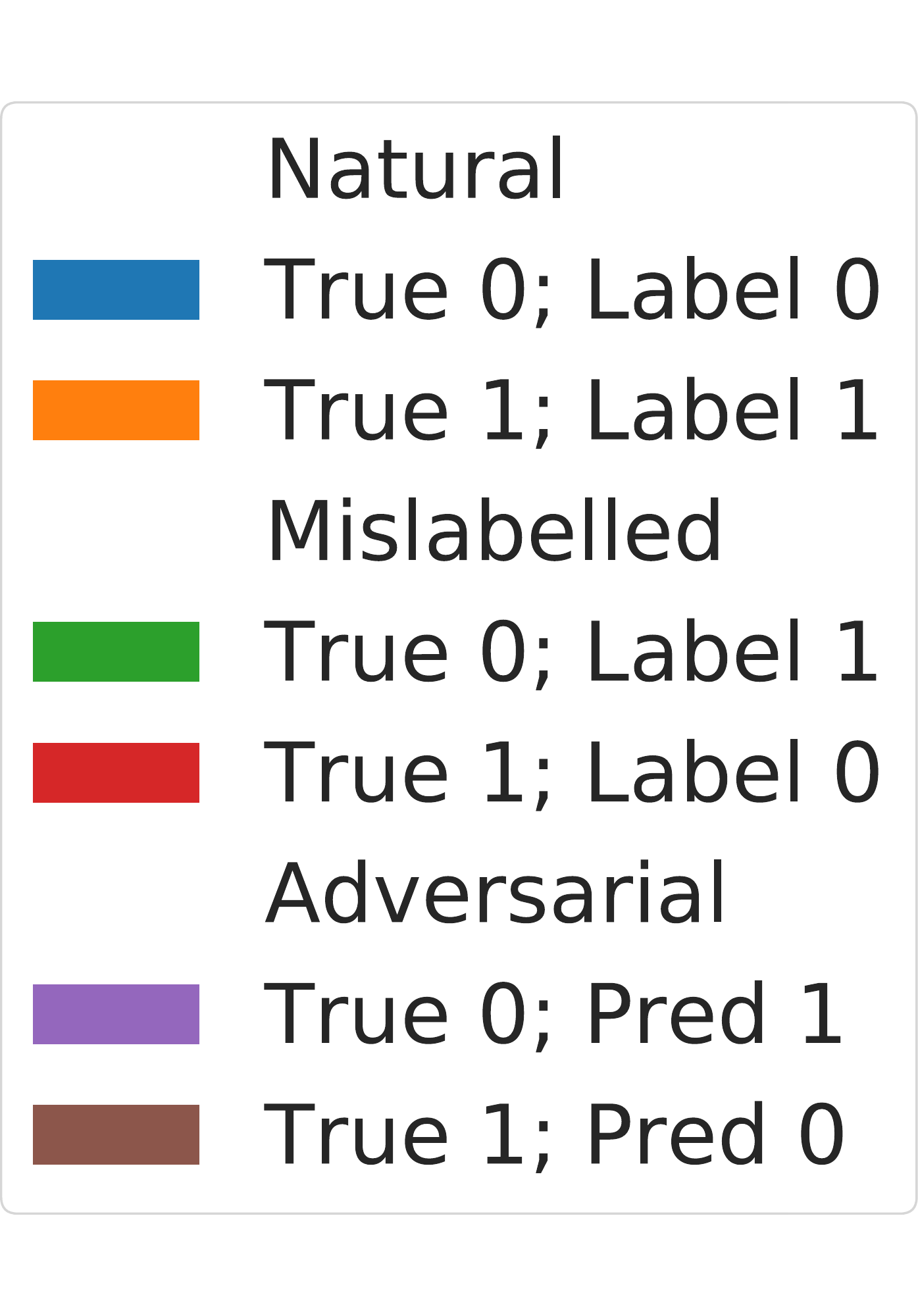
  \end{subfigure}\hfill
                 \begin{subfigure}[t]{0.80\linewidth}
    \centering \def\svgwidth{0.99\columnwidth}
    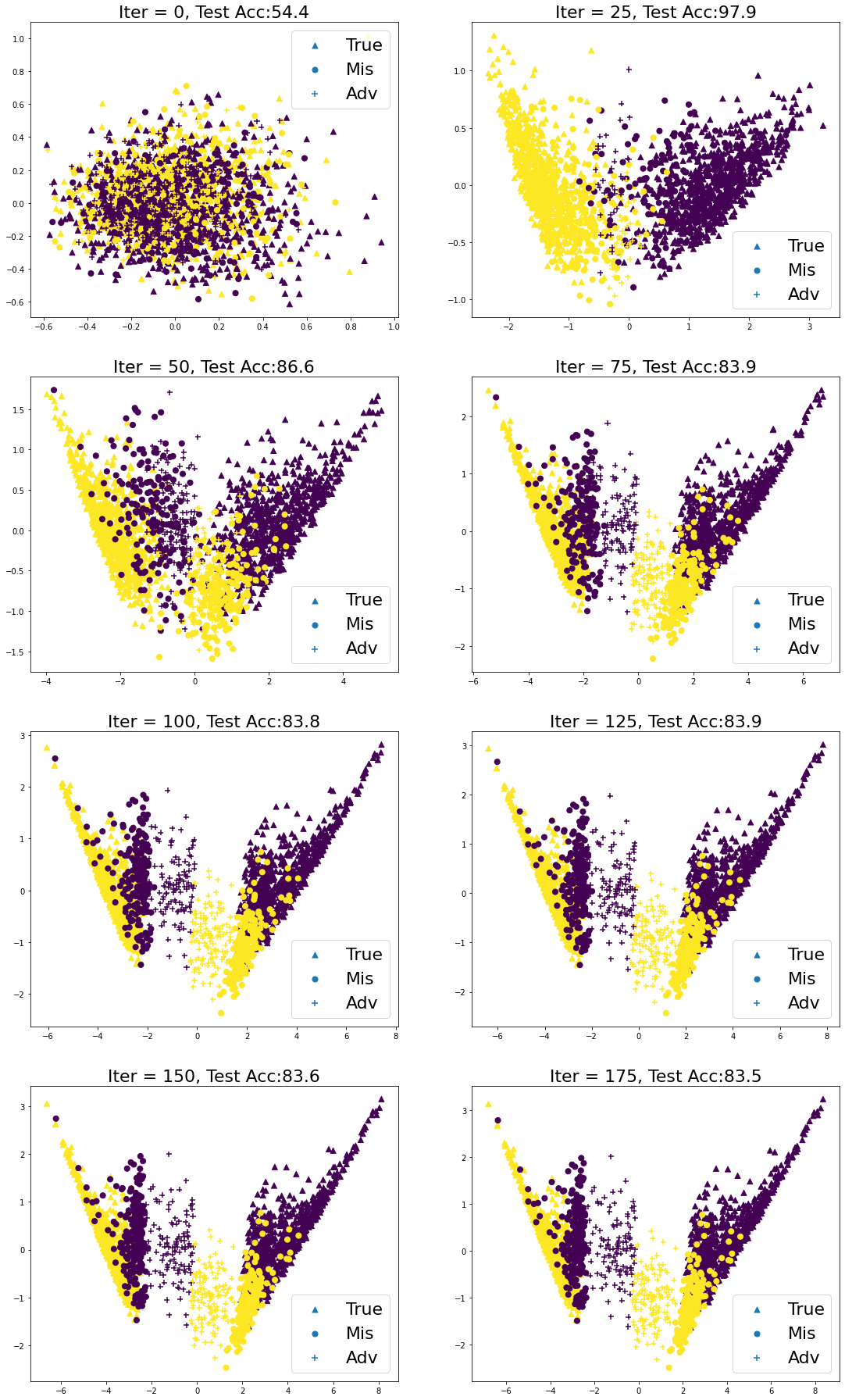
  \end{subfigure}
  \caption{Two dimensional PCA projections of the original correctly
    labelled~(blue and orange), original mis-labelled~(green and red),
  and adversarial examples~(purple and brown) at different stages of training. The correct label
  for~\emph{True 0}~(blue),~\emph{Noisy 0}~(green),~\emph{Adv
    0}~(purple +) are the same i.e. 0 and similar for the other class.}
  \label{fig:represen}
\end{figure}

\textbf{Visualizing through low-dimensional projections}: For the toy-MNIST
problem, we plot a 2-d projection~(using PCA) of the learned
representations~(activations before the last layer) at various stages of
training in~\Cref{fig:represen}.  (We remark that the simplicity of the data
model ensures that even a 1-d PCA projection suffices to perfectly separate the
classes when there is no label noise; however, the representations learned by a
neural network in the presence of noise maybe very different!) We highlight two
key observations: (i) The bulk of adversarial examples~(``$+$''-es) are
concentrated around the mis-labelled training data~(``$\circ$''-es) of the
opposite class. For example, the purple $+$-es~(Adversarially
perturbed: True: 0, Pred:1 ) are very close to the green
$\circ$-es~(Mislabelled: True:0, Pred: 1). This provides empirical validation for the hypothesis that if
there is a mis-labelled data-point in the vicinity that has been fit by the
model, an adversarial example can be created by moving towards that data point
as predicted by~\Cref{thm:inf-label}. (ii) The mis-labelled training data take
longer to be fit by the classifier.  For example by iteration 20, the network
actually learns a fairly good representation and classification boundary that
correctly fits the clean training data (but not the noisy training data). At this
stage, the number of adversarial examples are much lower as compared to
Iteration 160, by which point the network has completely fit the noisy training
data. Thus early stopping helps in avoiding \emph{memorizing} the label noise,
but consequently also reduces adversarial vulnerability. Early stopping has
indeed been used as a defence in quite a few recent papers in context of
adversarial robustness~\citep{Wong2020Fast,hendrycks2019pretraining}, as well
as learning in the presence of label-noise~\citep{Li2019}. Our work provides an explanation regarding \emph{why} early stopping may reduce adversarial vulnerability by avoiding fitting noisy training data.  

\subsection{Robust training avoids memorization of (some) label
noise}
\label{sec:robust-avoid-mem}
\begin{table}[!htb]
  \centering
  \begin{tabular}{|ccc|}\toprule
    $\epsilon$&Train-Acc.~($\%$)&Test-Acc~($\%$)\\\midrule
    0.0&99.98&95.25\\
    0.25&97.23&92.77\\
    1.0&86.03&81.62\\\bottomrule
  \end{tabular}
   \caption{Train and test accuracies on clean dataset for ResNet-50
    models trained using $\ell_2$ adversaries of perturbation
    $\epsilon$. The $\epsilon=0$ setting represents the natural training.}
  \label{tab:accs-robust-models}
\end{table}
Robust training methods like AT~\citep{madry2018towards} and
TRADES~\citep{Zhang2019} are commonly used techniques to increase
adversarial robustness of deep neural networks. However, it has been pointed out that this comes at a cost to clean
accuracy~\citep{Raghunathan2019,tsipras2018robustness}. When trained with these methods, both the training
and test accuracy (on clean data) for commonly used deep learning models drops with increasing strength of the PGD
adversary used (see~\Cref{tab:accs-robust-models}). In this section,
we provide  evidence to show that robust training avoids memorization
of label noise and this also results in the drop of clean train and
test accuracy.

\subsubsection{Robust training ignores label noise} ~\Cref{fig:mislabelled_ds} shows that
label noise is not uncommon in standard datasets like MNIST and
CIFAR10. In fact, upon closely monitoring the mis-classified training
set examples for both~\AT and  TRADES, we found that that neither
predicts correctly on the training set labels for any of the examples
identified in~\Cref{fig:mislabelled_ds}, all examples that have a
wrong label in the training set, whereas natural  training does. Thus, in
line with~\Cref{thm:inf-label}, robust training methods ignore
fitting noisy labels. %

We also observe this in a synthetic experiment on the full MNIST
 dataset where we assigned random
  labels to 15\% of the dataset. A naturally trained CNN model
  achieved $100\%$ accuracy on this dataset whereas an adversarially
  trained model~(standard setting with $\epsilon=0.3$ for $30$ steps)  mis-classified $997$ examples in the training set
  after the same training regime. Out of these $997$ samples, $994$
  belonged to the set of examples whose labels were randomized.

\begin{figure}[t]
  \centering
 \begin{subfigure}[b]{0.53\linewidth}
 \begin{subfigure}[t]{0.99\linewidth}
    \includegraphics[width=0.24\linewidth]{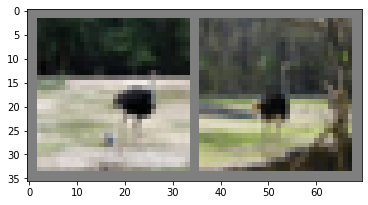}
    \includegraphics[width=0.24\linewidth]{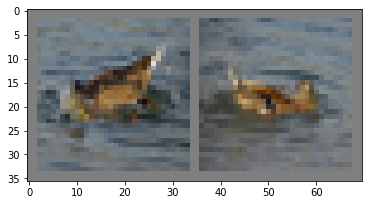}
    \includegraphics[width=0.24\linewidth]{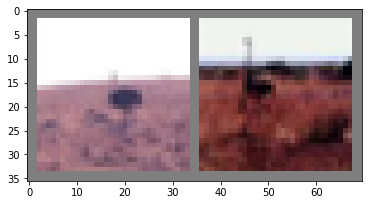}
    \includegraphics[width=0.24\linewidth]{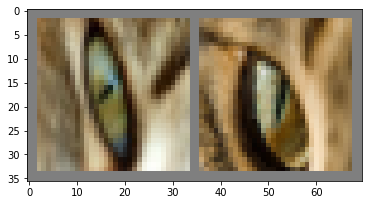}
  \end{subfigure}
  \begin{subfigure}[t]{0.99\linewidth}
    \includegraphics[width=0.24\linewidth]{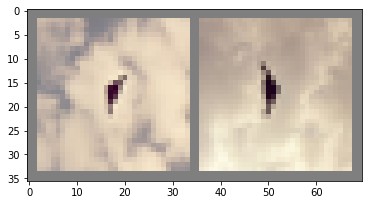}
    \includegraphics[width=0.24\linewidth]{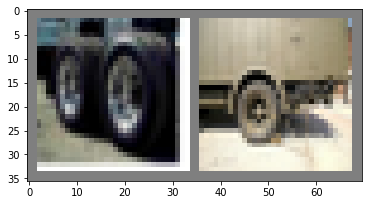}
    \includegraphics[width=0.24\linewidth]{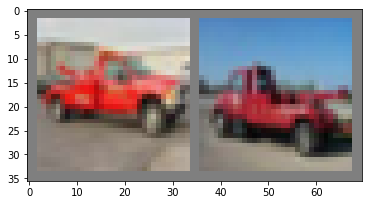}
    \includegraphics[width=0.24\linewidth]{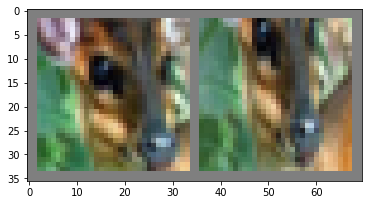}
  \end{subfigure}
  \begin{subfigure}[t]{0.99\linewidth}
    \includegraphics[width=0.24\linewidth]{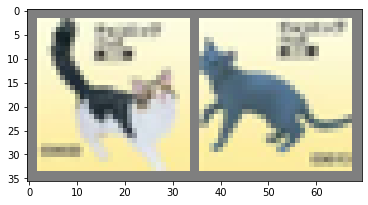}
    \includegraphics[width=0.24\linewidth]{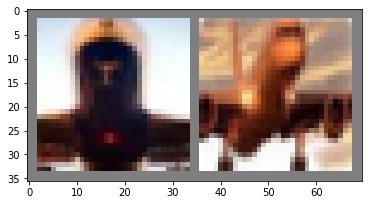}
    \includegraphics[width=0.24\linewidth]{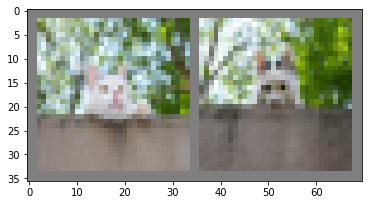}
\includegraphics[width=0.24\linewidth]{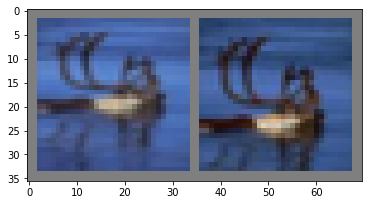}
\end{subfigure}
\caption*{CIFAR10}
\end{subfigure}
\begin{subfigure}[b]{0.45\linewidth}
  \centering \def\svgwidth{0.99\linewidth}
  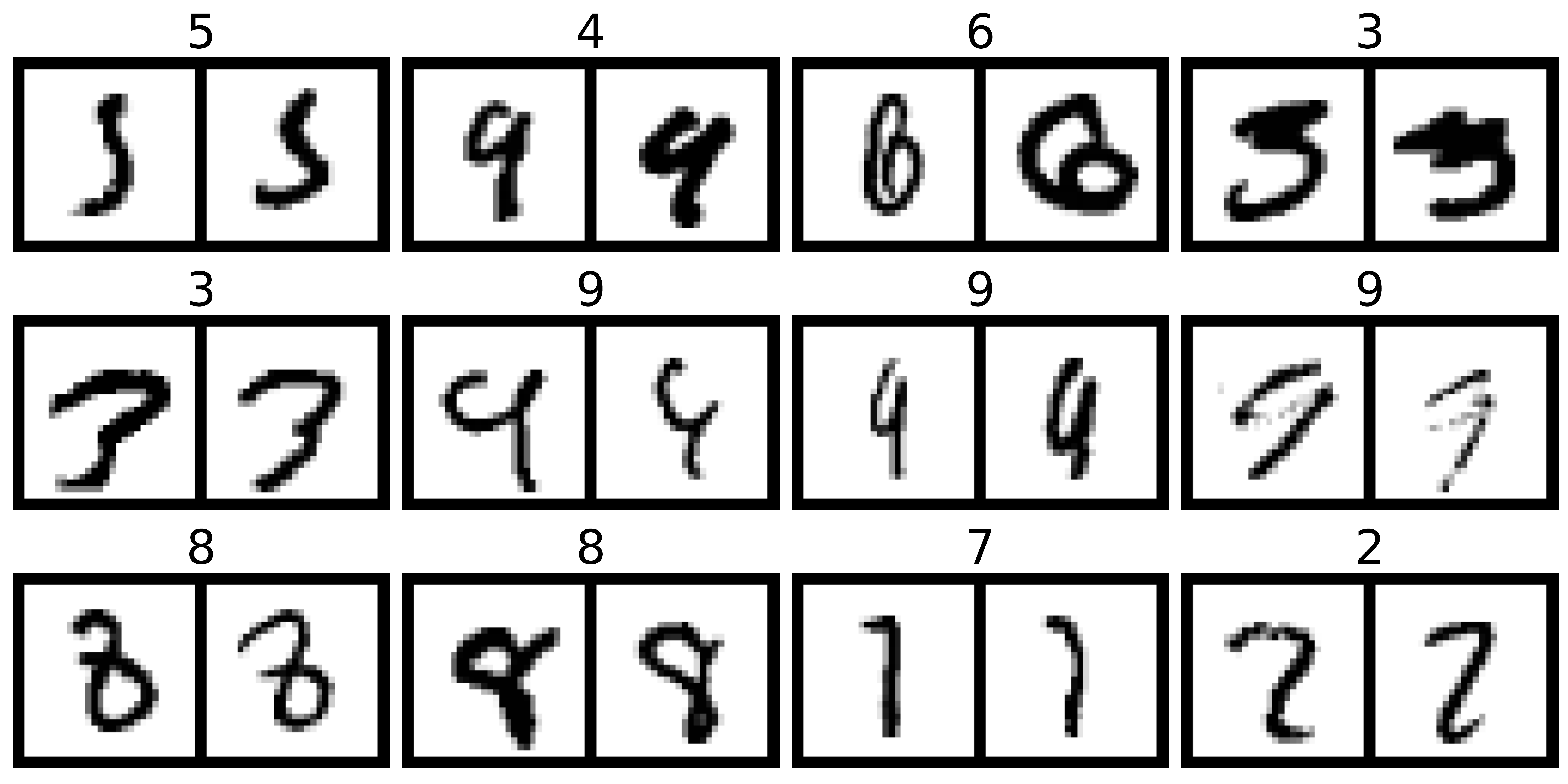
  \caption*{MNIST}
\end{subfigure}
\caption{Each pair is a  training~(left)
    and  test~(right) image mis-classified by the adversarially
    trained model. They were both correctly classified
    by the naturally-trained model.}
  \label{fig:adv-train-test-mis-class}
\end{figure}

\subsubsection{Robust Training ignores rare examples} Next, we show that though ignoring these rare samples helps in
adversarial robustness, it hurts the natural test
accuracy. Our hypothesis is that one of the effects of robust training is
to not \emph{memorize rare examples}, which would otherwise be
memorized by a naturally trained model. The
underlying intuition is that certain examples in the training set
belong to rare sub-populations~(eg. a special kind of cat) and this sub-population is
sufficiently distinct from the rest of the examples of that class in
the training dataset~(other cats in the
dataset). As~\citet{Feldman2019} points out,\emph{ if these sub-populations
are very infrequent in the training dataset, they are
indistinguishable from data-points with label noise with the
difference being that examples from that sub-population are also present in the test-set}.
Natural training by \emph{memorizing} those rare training examples reduces the test error on the
corresponding test examples. Robust training, by not memorizing these
rare samples~(and label noise), achieves better robustness but sacrifices the test accuracy on the
test examples corresponding to those  training points.

\begin{figure}[t]
  \begin{subfigure}[t]{1.0\linewidth}
    \begin{subfigure}[t]{0.19\linewidth}
      \centering \def\svgwidth{0.99\linewidth}
      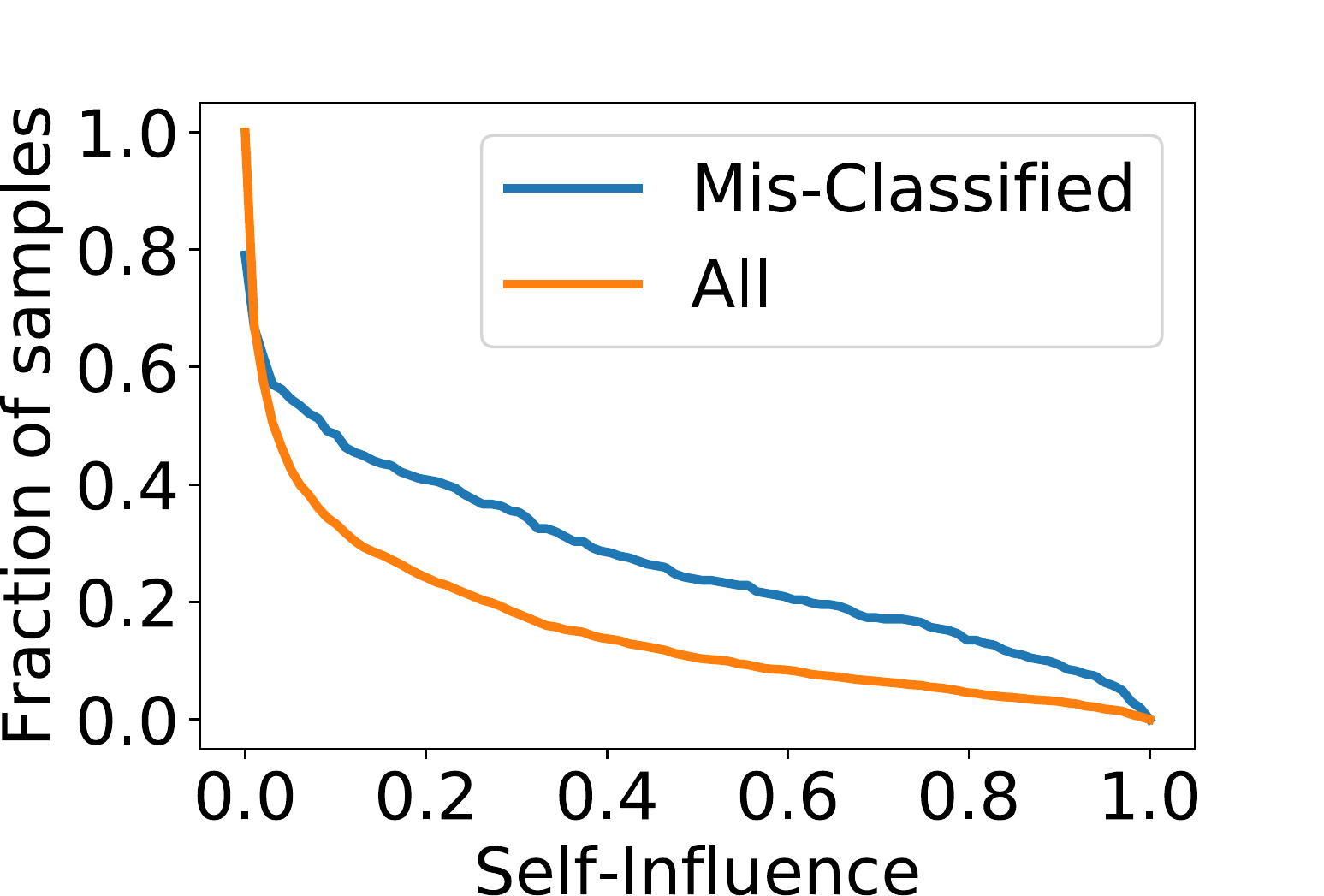
      \caption*{PLANES}    
    \end{subfigure}
    \begin{subfigure}[t]{0.19\linewidth}
      \centering \def\svgwidth{0.99\linewidth}
      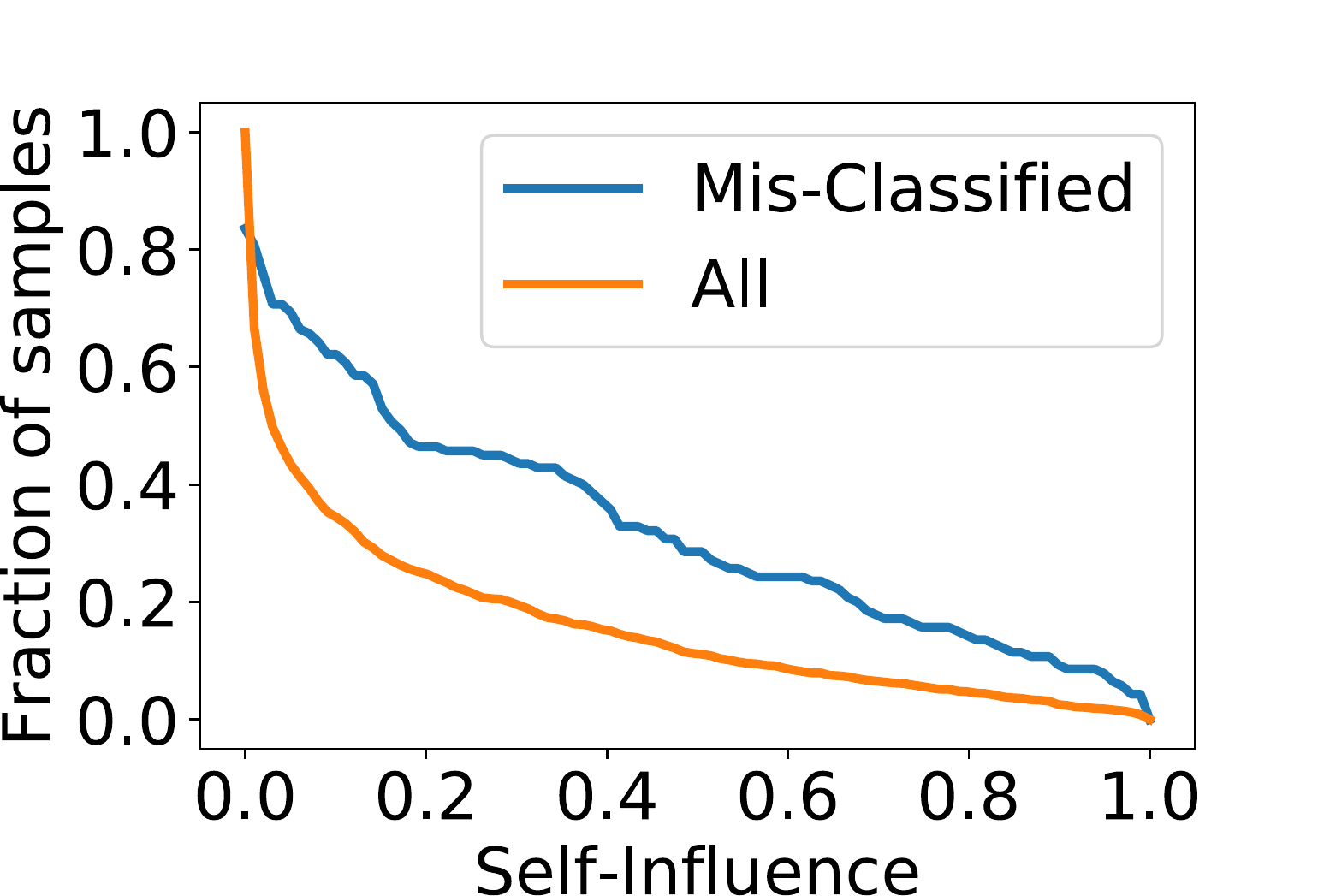
      \caption*{CAR}     
    \end{subfigure}
    \begin{subfigure}[t]{0.19\linewidth}
      \centering \def\svgwidth{0.99\linewidth}
      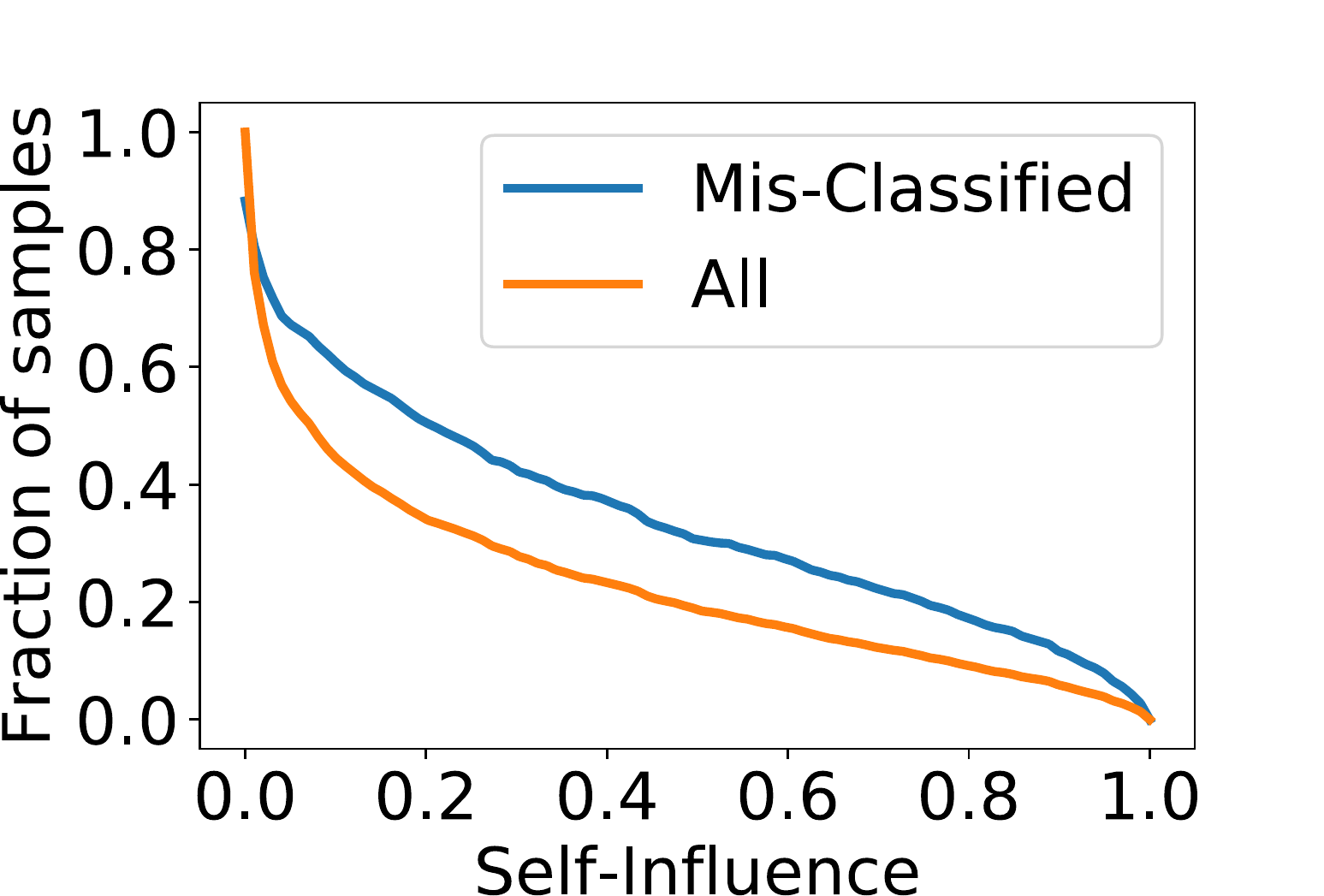
      \caption*{BIRD}     
    \end{subfigure}
    \begin{subfigure}[t]{0.19\linewidth}
      \centering \def\svgwidth{0.99\linewidth}
      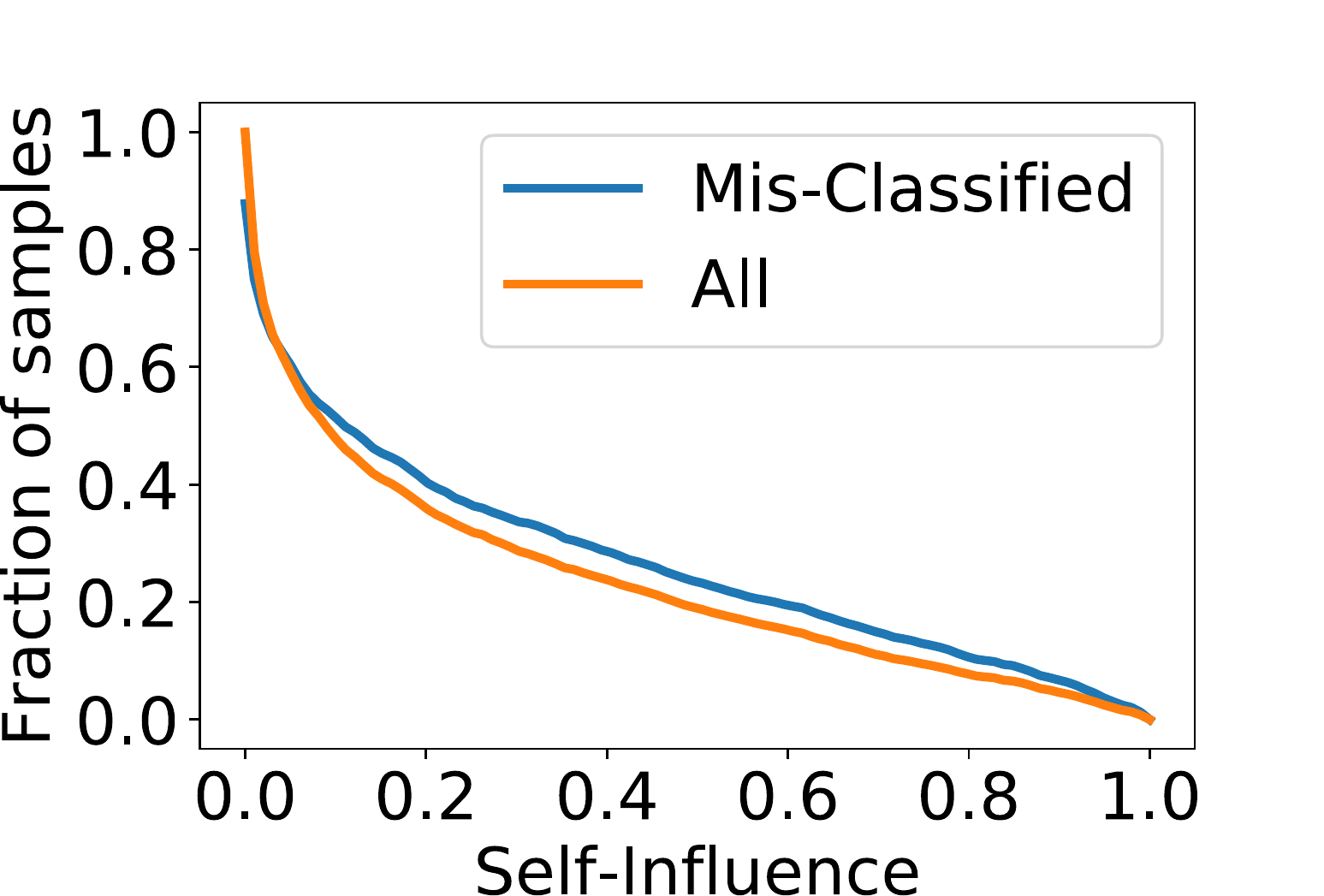
      \caption*{CAT}    
    \end{subfigure}
    \begin{subfigure}[t]{0.19\linewidth}
      \centering \def\svgwidth{0.99\linewidth}
      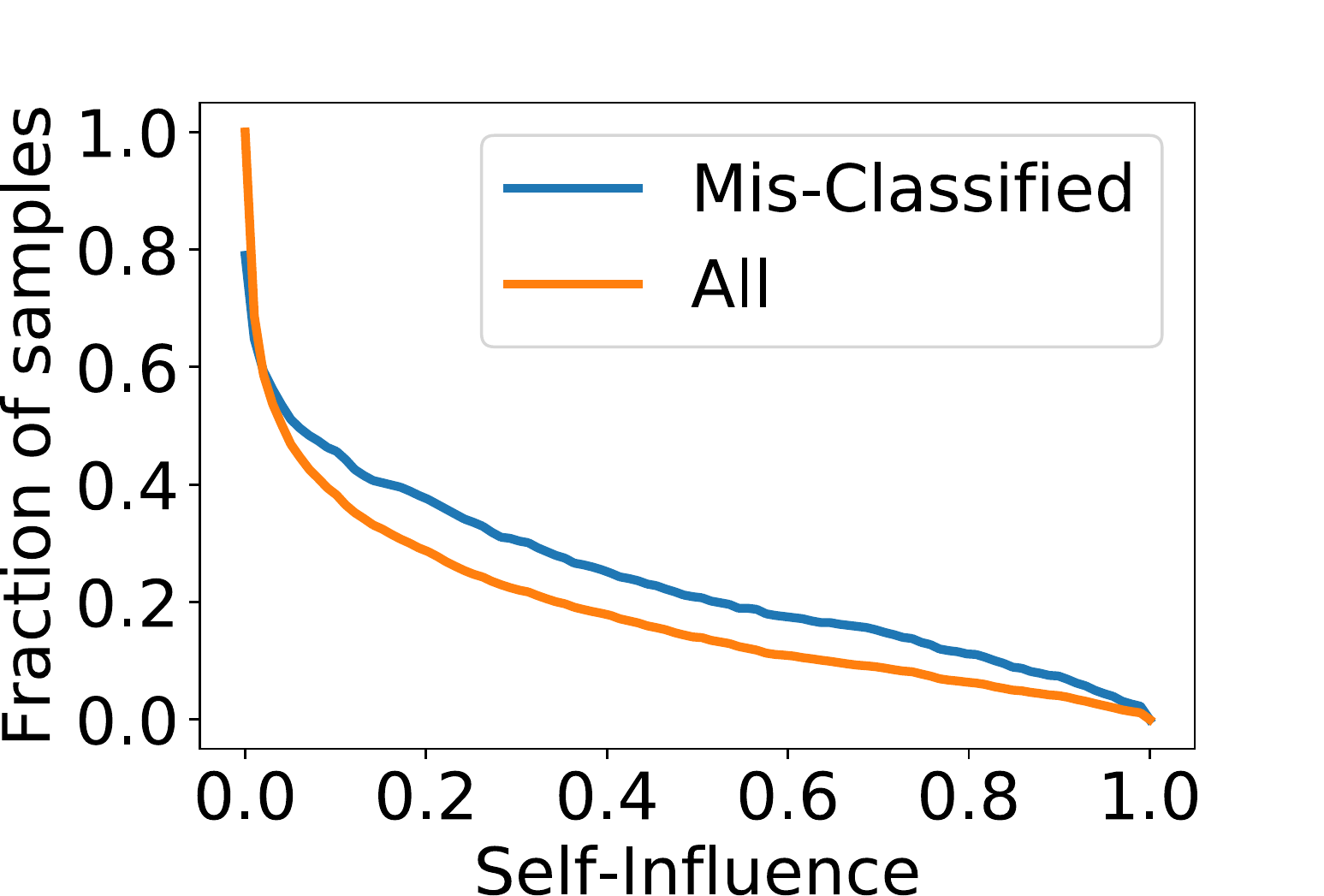
      \caption*{DEER}     
    \end{subfigure}
  \end{subfigure}
    \begin{subfigure}[t]{1.0\linewidth}
    \begin{subfigure}[t]{0.19\linewidth}
      \centering \def\svgwidth{0.99\linewidth}
      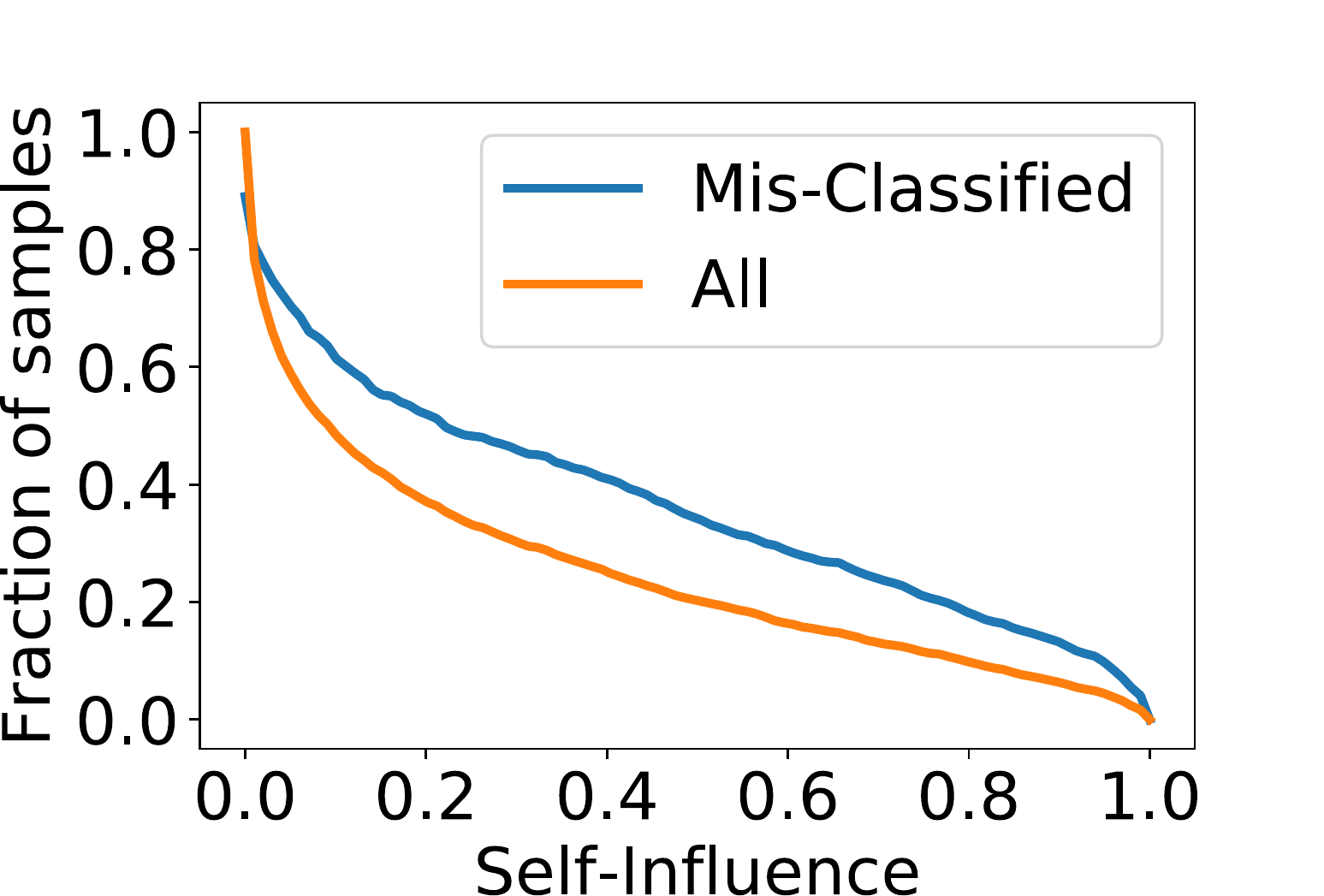
      \caption*{DOG}      
    \end{subfigure}
    \begin{subfigure}[t]{0.19\linewidth}
      \centering \def\svgwidth{0.99\linewidth}
      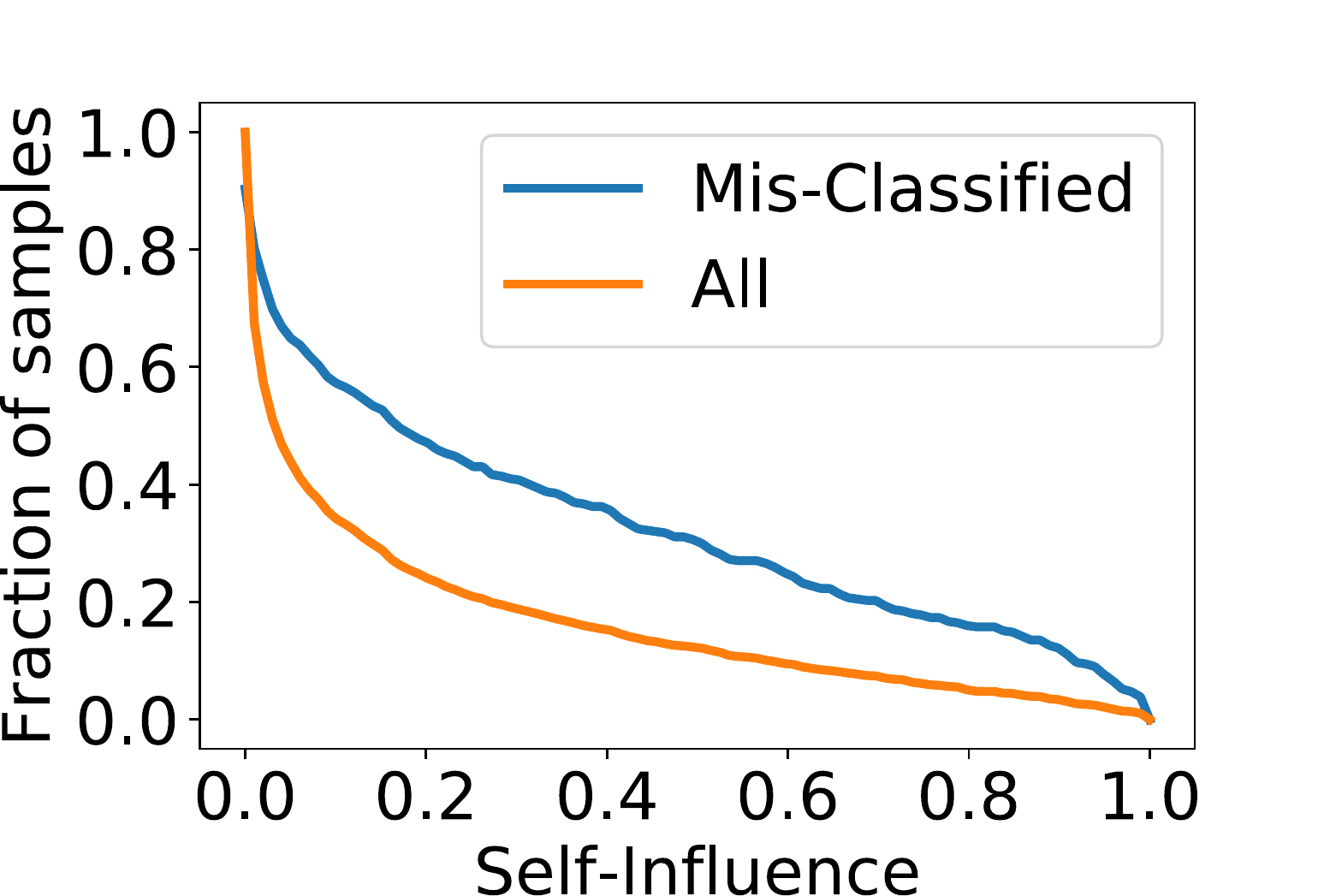
      \caption*{FROG}      
    \end{subfigure}
    \begin{subfigure}[t]{0.19\linewidth}
      \centering \def\svgwidth{0.99\linewidth}
      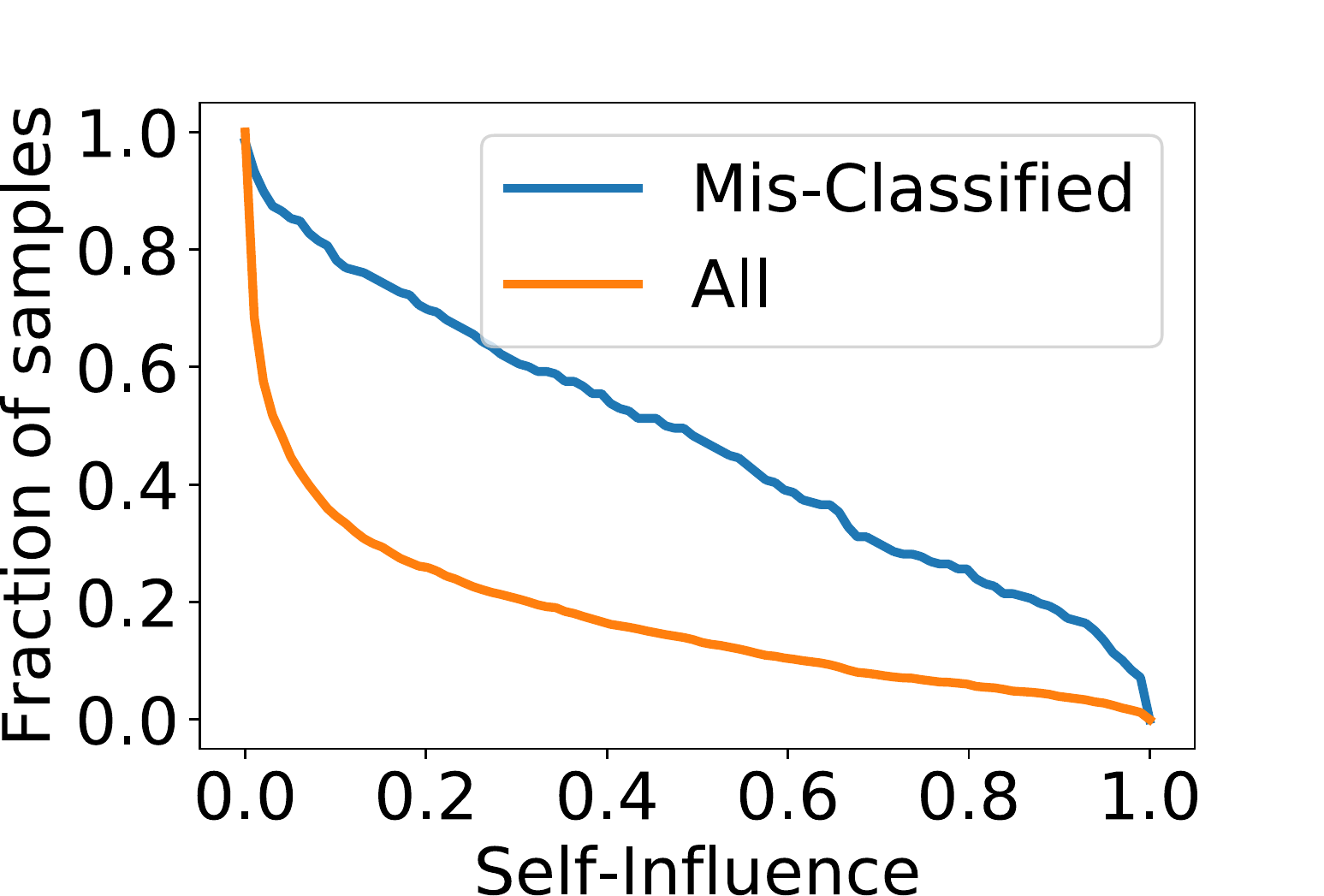
      \caption*{HORSE}     
    \end{subfigure}
    \begin{subfigure}[t]{0.19\linewidth}
      \centering \def\svgwidth{0.99\linewidth}
      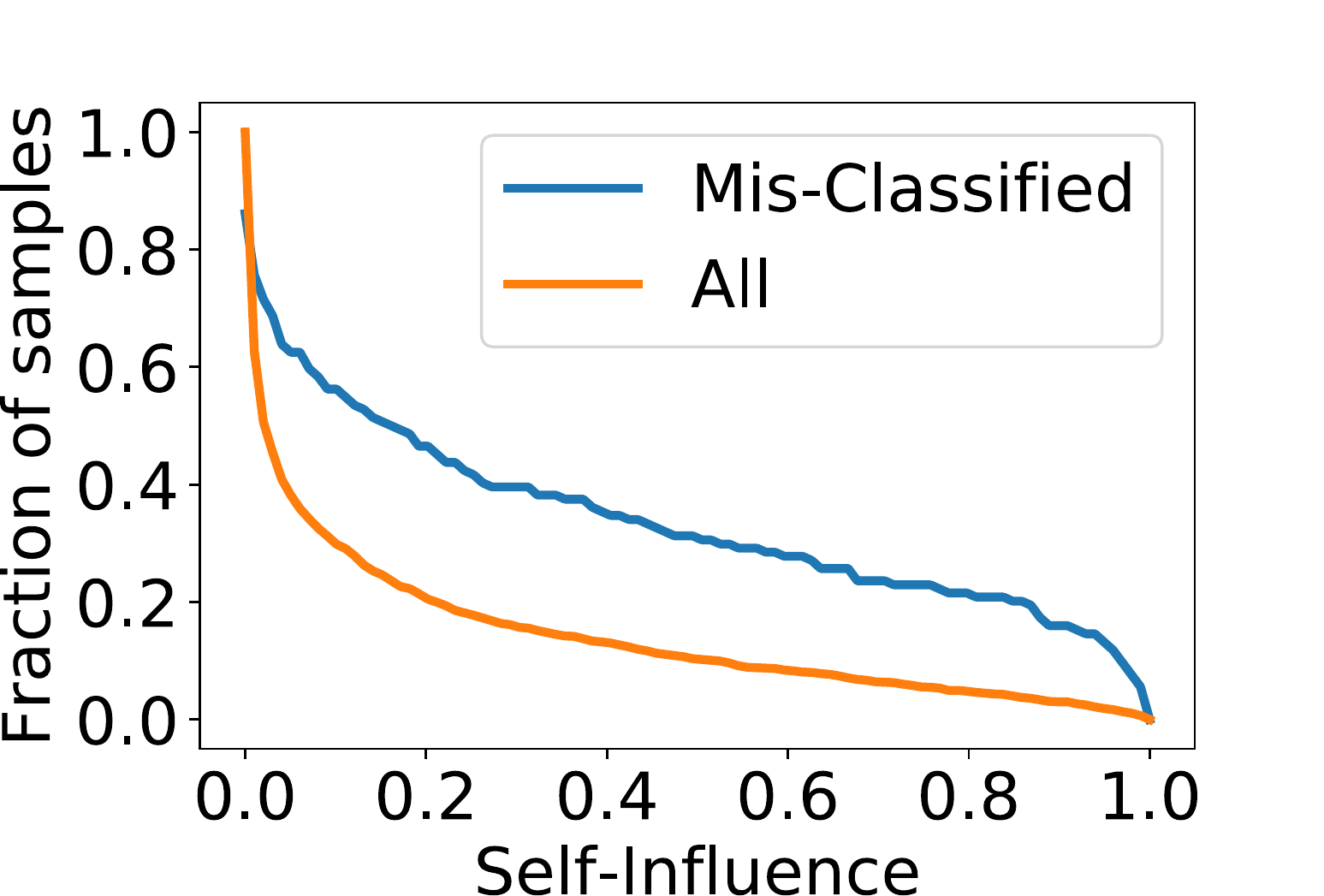
      \caption*{SHIP}
    \end{subfigure}
    \begin{subfigure}[t]{0.19\linewidth}
      \centering \def\svgwidth{0.99\linewidth}
      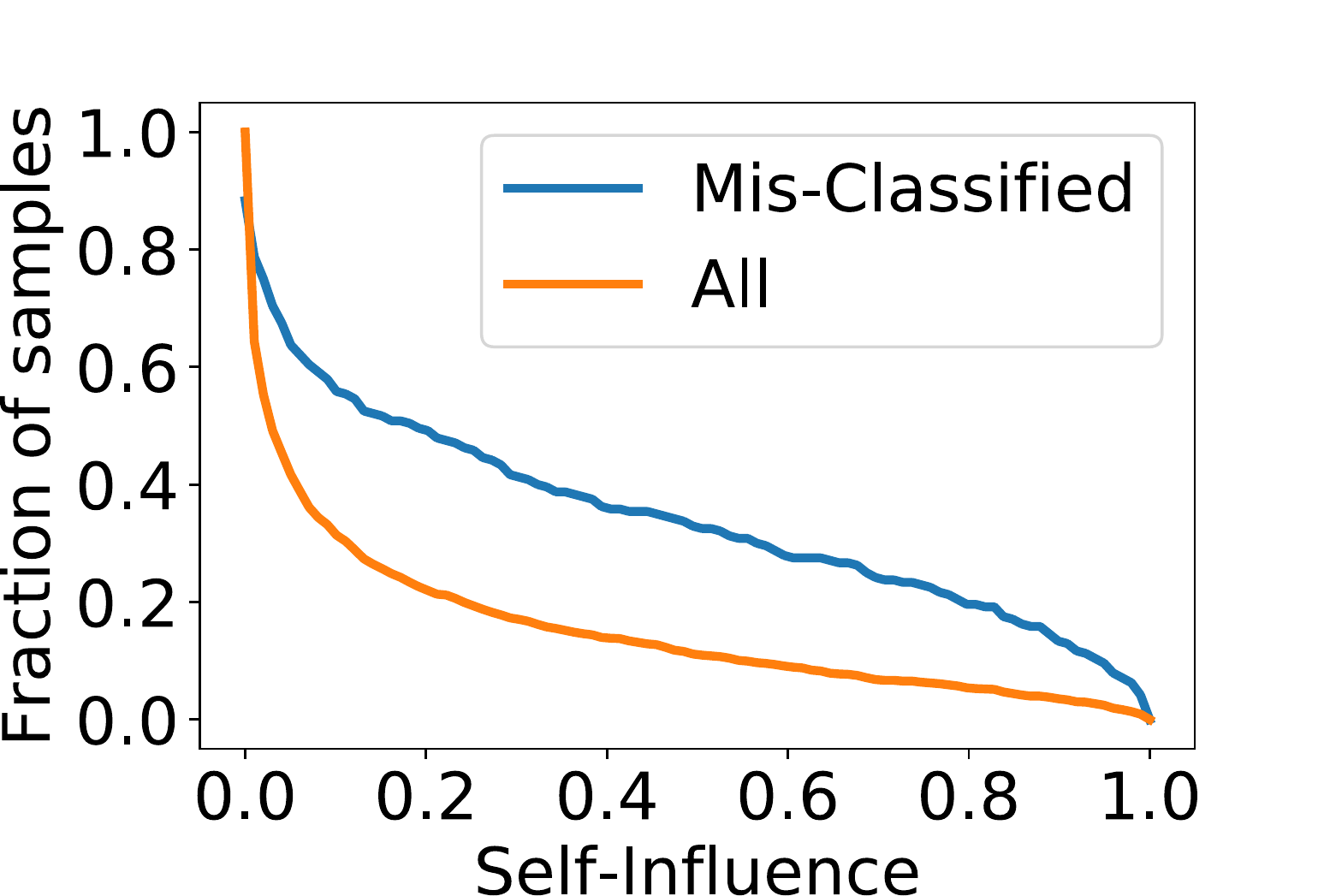
      \caption*{TRUCK}
    \end{subfigure}
  \end{subfigure}
  \caption{Fraction of train points that have a
    self-influence greater than $s$ is plotted versus $s$. The blue
    line represents the points mis-classified by
  an adversarially trained model on CIFAR10. The orange lines shows the
  distribution of self-influence for all points in the CIFAR10 dataset~(of the
  concerned class).  }
\label{fig:cifar_self_influence}
\end{figure}

\paragraph{Experiments on MNIST and CIFAR10} We demonstrate this effect in~\Cref{fig:adv-train-test-mis-class} with
examples from CIFAR10 and MNIST. Each pair of images contains a
mis-classified~(by robustly trained models) test image and the
mis-classified training image ``responsible'' for it (We describe
below how they were identified.).
Importantly both of
these images were correctly classified by a naturally trained
model. Visually, it is evident that the training images are  
extremely similar to the corresponding test image.  Inspecting the
rest of the training set, they are also very different from other
images in the training set. We can thus refer to these as rare sub-populations.

The notion that certain test examples
were not classified correctly due to a particular training examples
not being classified correctly  is measured by the
influence a training image has on the test image~(c.f. defn 3
in~\citet{Zhang2020}). Intuitively, it measures the probability that a
certain test example would be classified correctly if the model were
learned using a training set that \emph{did not contain} the training
point compared to if the training set \emph{did contain} that particular training
point. We obtained the influence of each training image on each test
image for that class from~\citet{Zhang2020}. We found the images
in~\Cref{fig:adv-train-test-mis-class} by manually searching for each test image, the training image that is
misclassified and is visually close to it. Our search space was
shortened with the help of the influence scores each training image has on the classification
of a test image. We searched in the set of top-$10$ most influential
mis-classified train images for each mis-classified test
image. The model used for~\cref{fig:adv-train-test-mis-class} is a \AT model for CIFAR10 with $\ell_2$-adversary with an
$\epsilon=0.25$ and a model trained with TRADES for MNIST with
$\lambda=\frac{1}{6}$ and $\epsilon=0.3$.

A precise notion of measuring if a sample is~\emph{rare} is through the
concept of self-influence. Self
influence of an example with respect to an algorithm~(model, optimizer
etc) can be defined as how unlikely it is for the model learnt
by that algorithm to be correct on an example if it \emph{had not
seen} that example during training compared to if it \emph{had seen} the
example during training. For a precise mathematical definition please
refer to Eq~(1) in~\citet{Zhang2020}. Self-influence for a~\emph{rare example}, that is
unlike other examples of that class, will be high as the rest of the
dataset will \emph{not} provide relevant information that will help the model in 
correctly predicting on that particular
example. In~\cref{fig:cifar_self_influence}, we show that the 
self-influence of training samples that were mis-classified by
adversarially trained models but correctly classified by a naturally
trained model is higher  compared to the distribution of
self-influence on the entire train dataset. In other words, it means
that the  self-influence  of the training examples mis-classified by the
robustly trained models is larger than the average self-influence of
~(all) examples belonging to that class. This supports our hypothesis
that adversarial training excludes fitting 
these rare~(or ones that need to be memorized) samples.

\paragraph{Experiments on a synthetic setting}This phenomenon is demonstrated
more clearly in a simpler distribution for different NN configurations
in~\Cref{fig:simple_adv_train_app}. We create a binary classification
problem on $\reals^2$. The data is
uniformly  supported on non-overlapping
circles of varying radiuses. All points in one circle have the same
label i.e. it is either blue or red depending on the color of the circle. We
train a shallow network with 2 layers and 1000 neurons in each
layer~(Shallow-Wide NN) and a deep network with 4 layers and 100 neurons
in each layer  using cross entropy loss and SGD. The background color
shows the decision region of the learnt neural
network.~\Cref{fig:simple_adv_train_app} shows that the adversarially
trained (\AT) models ignore the smaller circles~(i.e. rare
sub-populations) and tries to get a larger margin around the circles
it does classify correctly whereas the naturally trained~(NAT) models
correctly predicts every circle but ends up with very small margin
around a lot of circles.

\begin{figure}[t]
  \begin{subfigure}[t]{0.48\linewidth}
    \begin{subfigure}[t]{0.495\linewidth}
      \centering \def\svgwidth{0.999\linewidth}
      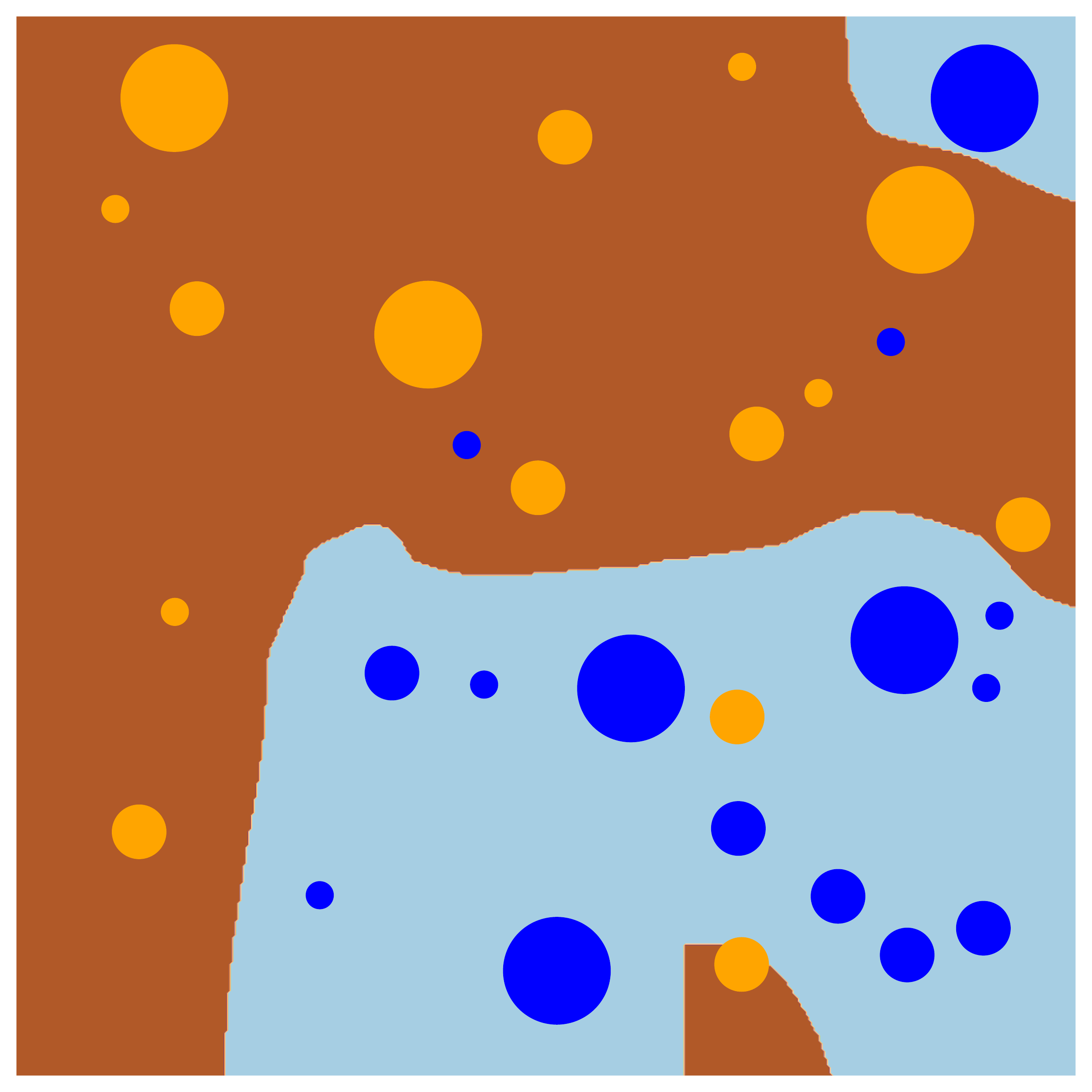
      \caption*{AT}
    \end{subfigure}
    \begin{subfigure}[t]{0.48\linewidth}
      \centering \def\svgwidth{0.999\linewidth}
      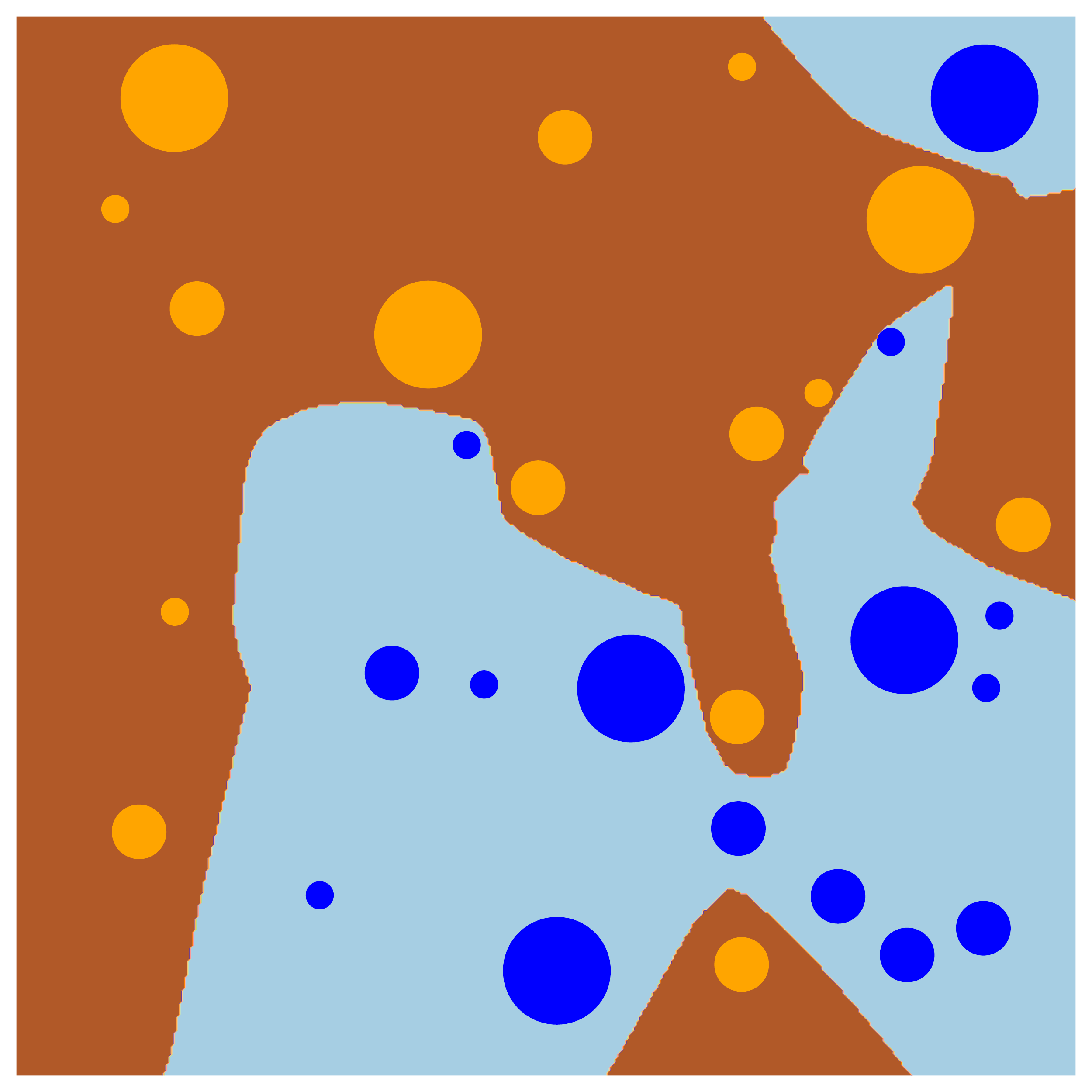
      \caption*{NAT}
    \end{subfigure}\caption{Shallow Wide NN}
  \end{subfigure}
  \begin{subfigure}[t]{0.48\linewidth}
    \begin{subfigure}[t]{0.495\linewidth}
      \centering \def\svgwidth{0.999\linewidth}
      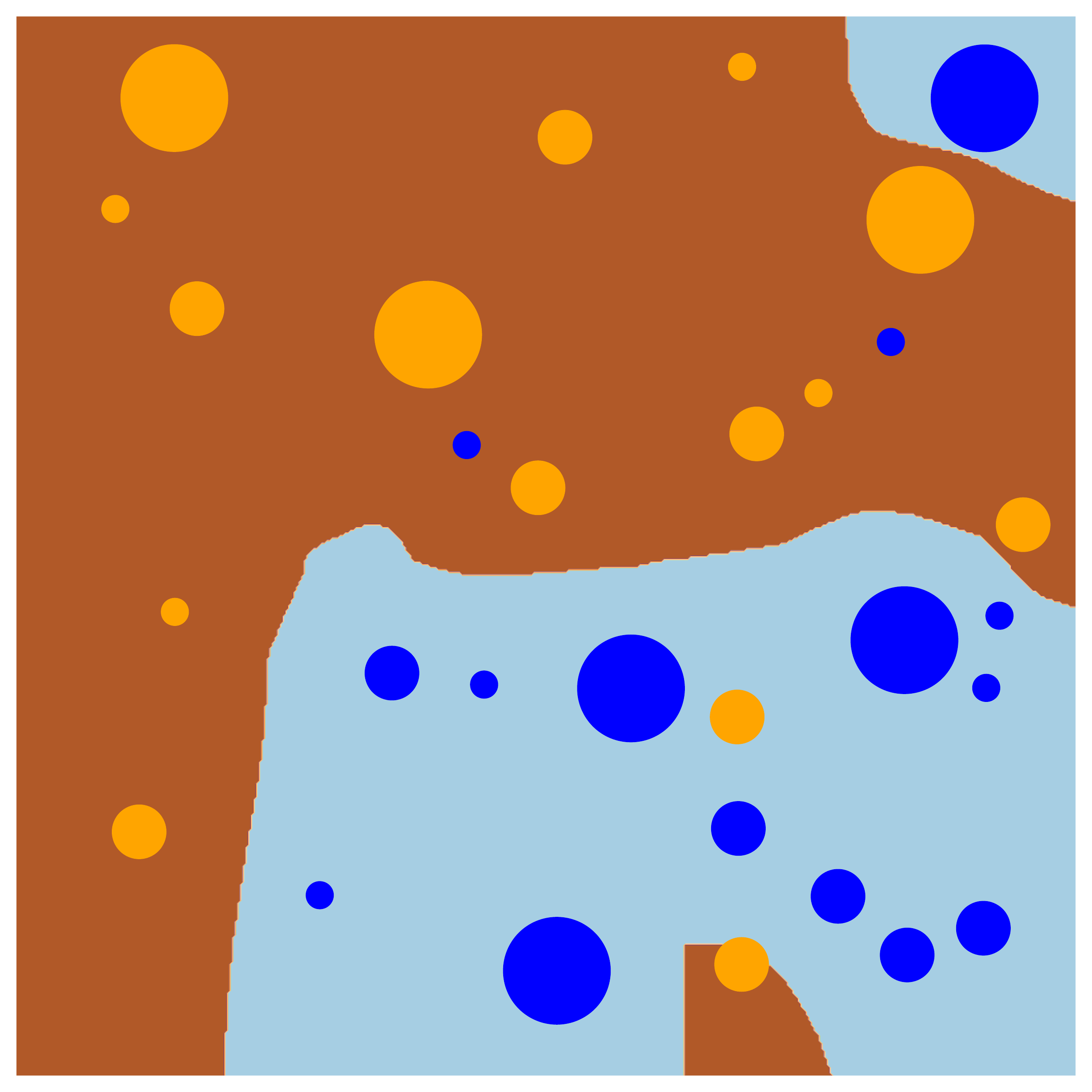
      \caption*{AT}
    \end{subfigure}
    \begin{subfigure}[t]{0.48\linewidth}
      \centering \def\svgwidth{0.999\linewidth}
      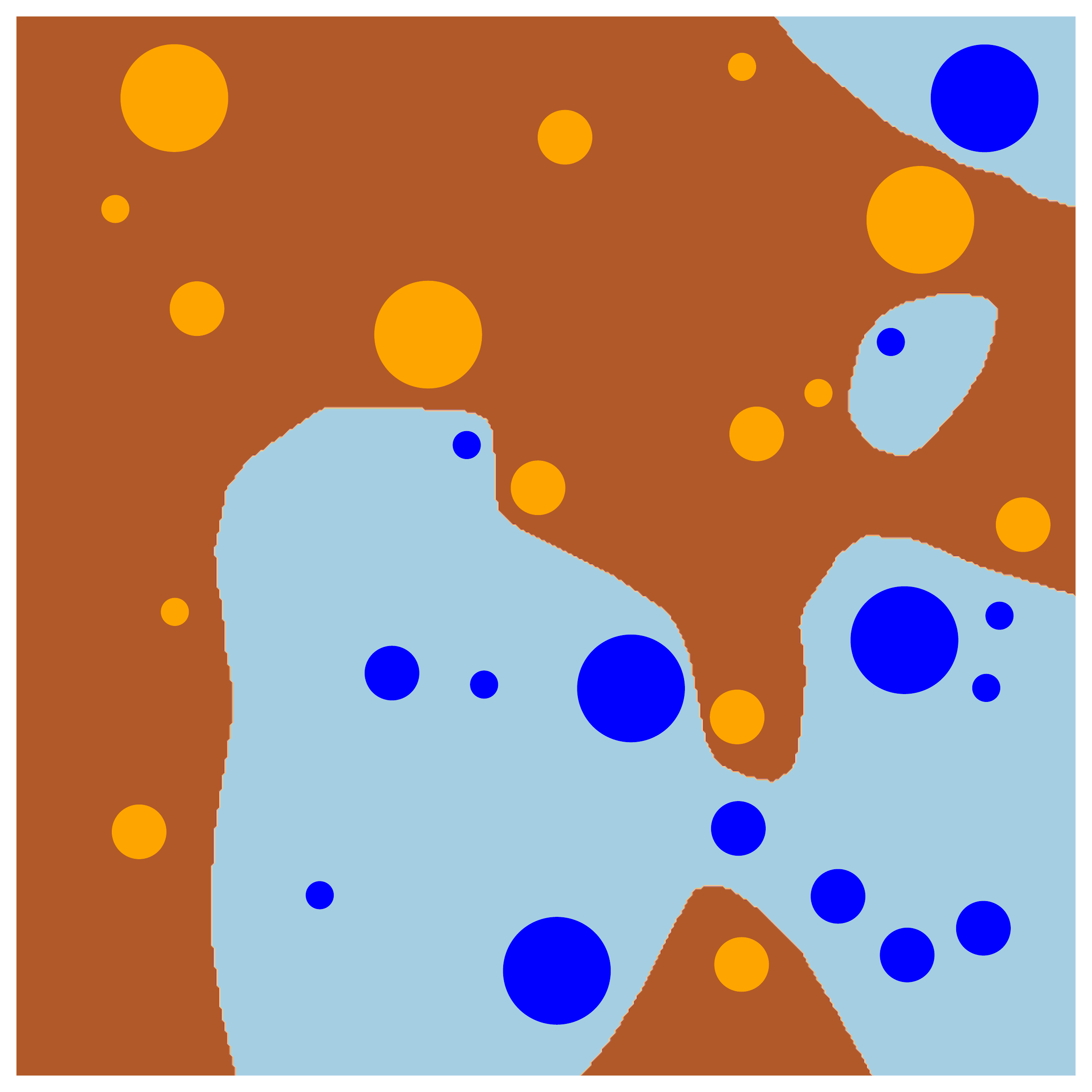
      \caption*{NAT}
    \end{subfigure}\caption{Deep NN}
  \end{subfigure}
  \caption{Adversarial training~(AT) leads to larger margin, and thus
    adversarial robustness around high density regions~(larger
    circles) but causes training error on low
    density sub-populations~(smaller circles) whereas naturally
    trained models~(NAT) minimizes the training error but leads to
    regions with very small margins.}
  \label{fig:simple_adv_train_app}
\end{figure}

\subsection{Complexity of decision boundaries}
\label{sec:robust-train-creat}
When neural networks are trained they create classifiers
whose decisions boundaries are much simpler than they need to be for
being adversarially robust. A few recent
papers~\citep{Nakkiran2019,schmidt2018adversarially} have discussed
that robustness might require more complex classifiers. In~\Cref{thm:parity_robust_repre_all,thm:repre-par-inter} we
discussed this theoretically and also why this might not violate the
traditional wisdom of Occam's Razor. In particular, complex decision
boundaries does  not necessarily mean more complex classifiers in
statistical notions of complexity like VC dimension. In this section,
we show through a simple experiment how the decision boundaries of
neural networks are not ``complex'' enough to provide large enough margins
and are thus adversarially much more vulnerable than is possible.

We train three different neural networks with ReLU activations, a shallow network~(Shallow
NN) with 2 layers and 100 neurons in each layer, a shallow network with 2 layers
and 1000 neurons in each layer~(Shallow-Wide NN), and a deep network with
4 layers and 100 neurons in each layer. We train them for 200 epochs
on a binary classification problem as constructed in~\Cref{fig:dec_reg_app_fig}.  The distribution is supported on
blobs and the color of each blob represent its label. On the right side, we have the decision
boundary of a large margin classifier, which is simulated using a
1-nearest neighbour.

From~\Cref{fig:dec_reg_app_fig}, it is evident that
the decision boundaries of neural networks trained with standard
optimizers have far \emph{simpler} decision boundaries than is needed to be
robust~(eg. the 1- nearest neighbour is much more robust than the
neural networks.)

\subsubsection{Accounting for fine grained sub-populations leads to better
  robustness}
\label{sec:fine-coarse}

We hypothesize that learning more meaningful representations by
  accounting for  fine-grained sub-populations within each class
  may lead to better robustness. We use the theoretical setup presented
in~\Cref{sec:bias-towrds-simpler,fig:complex_simple}. However,
if each of the circles belonged to a separate class then  the
decision boundary would have to be necessarily more complex as it
needs to, now, separate the balls that were previously within the same
class. We test this hypothesis with two experiments. First, we test it on the
the distribution defined in~\Cref{thm:parity_robust_repre_all} where
for each ball with label $1$, we assign it a different label~(say
$\alpha_1,\cdots, \alpha_{k}$) and similarly for balls with label $0$,
we assign it a different label~($\beta_1,\cdots,\beta_k$). Now, we
solve a multi-class classification problem for $2k$ classes with a
deep neural network  and then later aggregate the results by reporting
all $\alpha_i$s as $1$ and all $\beta_i$s as $0$.The resulting
decision boundary is drawn in~\Cref{fig:mc_parity} along with the
decision boundary for natural training and~\AT. Clearly, the decision
boundary for~\AT is the most complex and has the highest margin~(and
robustness) followed by the multi-class model and then the naturally
trained model.

Second, we also repeat the experiment with CIFAR-100. We train a ResNet50~\citep{HZRS:2016} on the fine
labels of CIFAR100 and then aggregate the fine labels corresponding to
a coarse label by summing up the logits. We call this model the
\emph{Fine2Coarse} model and compare the adversarial risk of this network to a
ResNet-50 trained directly on the coarse
labels. Note that the model is end-to-end differentiable as the only
addition is a layer to aggregate the logits corresponding to the fine
classes pertaining to each coarse class. Thus PGD adversarial attacks
can be applied out of the
box.~\Cref{fig:fine2coarse} shows that for all perturbation
budgets, \emph{Fine2Coarse} has smaller adversarial risk than the
naturally trained model.

\begin{figure}[t]
  \begin{subfigure}[t]{0.24\linewidth}
    \centering \def\svgwidth{0.99\linewidth}
    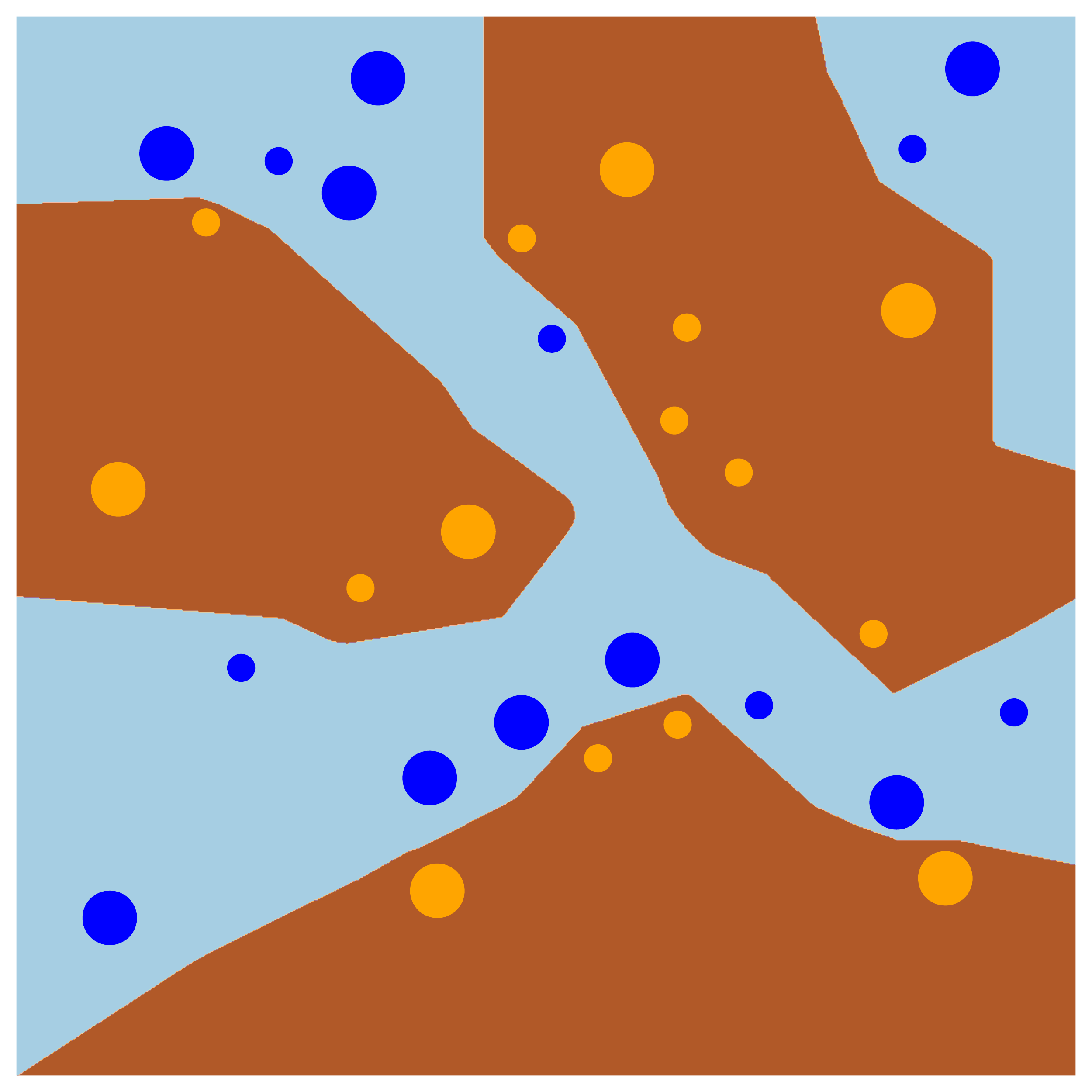
    \caption{Shallow NN}
    \label{fig:dec_shallow}
  \end{subfigure}
  \begin{subfigure}[t]{0.24\linewidth}
    \centering \def\svgwidth{0.99\linewidth}
    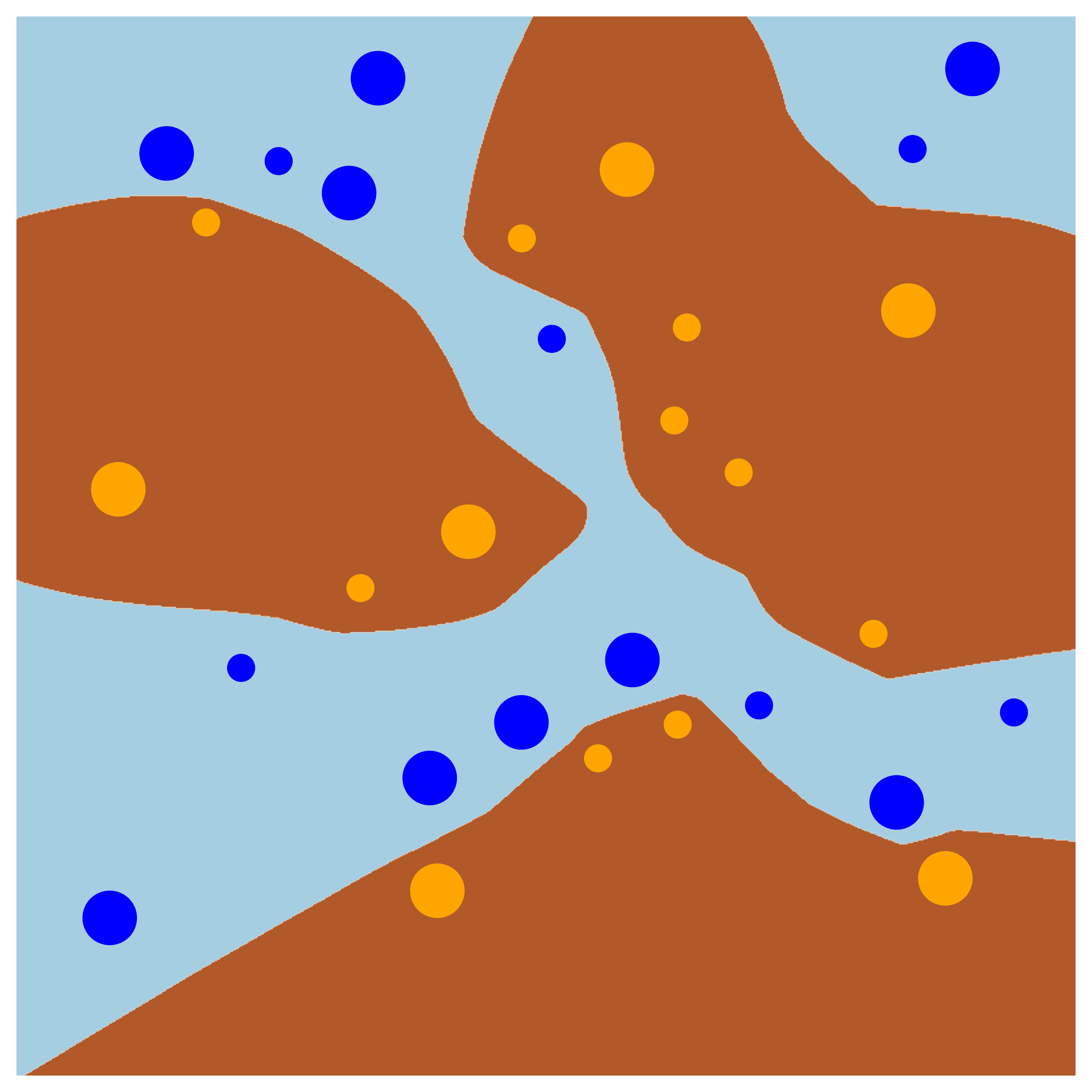
    \caption{Shallow-Wide NN}
    \label{fig:dec_shallow_wide}
  \end{subfigure}
  \begin{subfigure}[t]{0.24\linewidth}
    \centering \def\svgwidth{0.99\linewidth}
    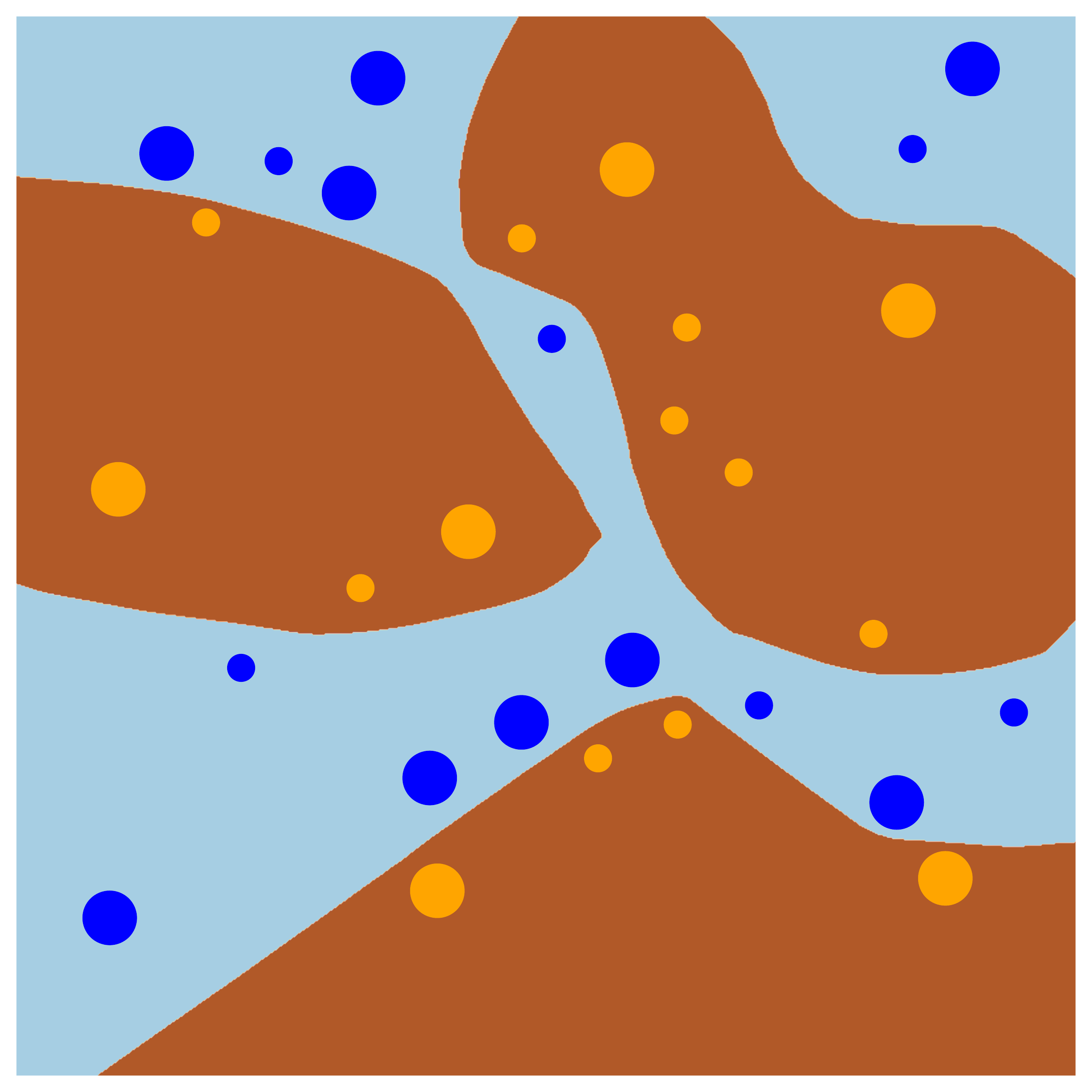
    \caption{Deep NN}
    \label{fig:dec_deep}
  \end{subfigure}
  \begin{subfigure}[t]{0.24\linewidth}
    \centering \def\svgwidth{0.99\linewidth}
    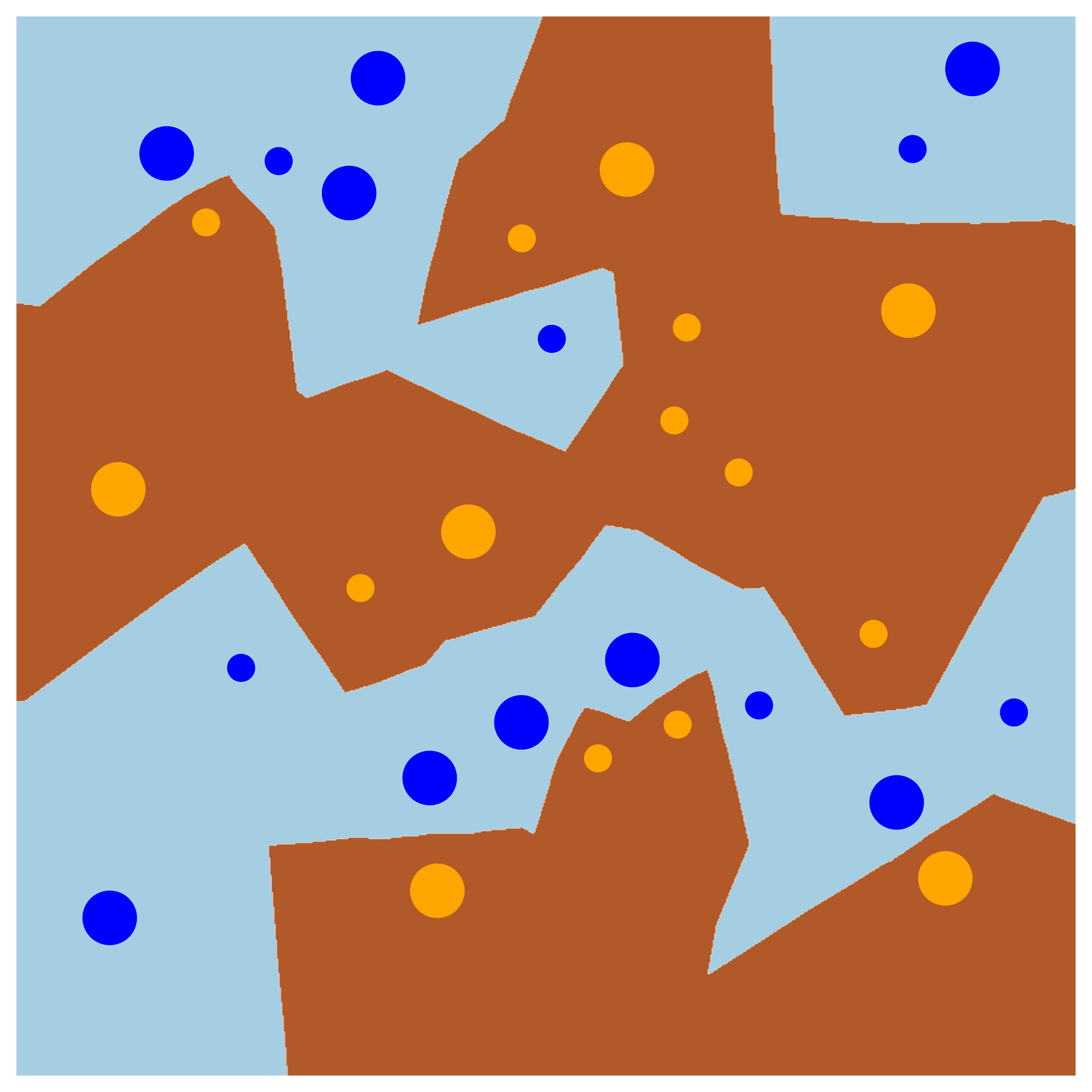
    \caption{Large Margin}
    \label{fig:dec_knn}
  \end{subfigure}
  \caption{Decision boundaries of neural networks are much simpler
    than they should be.}\label{fig:dec_reg_app_fig}
 \end{figure}

\section{Related Work}
\label{sec:related}

~\citep{montasser19a} established that there
are concept classes with finite VC dimensions i.e. are \emph{properly}
PAC-learnable but are only \emph{improperly} robustly PAC learnable. This
implies that to learn the problem with small adversarial error, a different
class of models~(or representations) needs to be used whereas for small
natural test risk, the original model class~(or representation) can be used.
Recent empirical works have also shown evidence towards
this~(eg.~\citep{Sanyal2020LR}).

\citet{pmlr-v97-hanin19a} have shown that
though the number of possible linear regions that can be created by a deep
ReLU network is exponential in depth, in practice for networks trained with
SGD this tends to grow only linearly thus creating much simpler decision
boundaries than is possible  due to sheer expresssivity of
deep networks. Experiments on the data models from our
theoretical settings indeed show that adversarial training indeed produces more
``complex'' decision boundaries

\citet{jacobsen2018excessive} have discussed that excesssive
invariance in neural networks might increase adversarial
error. However, their argument is that excessive invariance
can allow sufficient changes in the semantically important
features without changing the network's prediction. They
describe this as Invariance-based adversarial examples as
opposed to perturbation based adversarial examples. We show
that excessive ~(incorrect) invariance might also result in
perturbation based adversarial examples.

Another contemporary
work~\citep{Geirhos2020} discusses a phenomenon they refer to as~\emph{Shortcut
  Learning} where deep learning models perform very well on
standard tasks like reducing classification error but fail to
perform in more difficult real world situations. We discuss
this in the context of models that have small test
error but large adversarial error and provide and theoretical
and empirical to discuss why one of the reasons for this is
sub-optimal representation learning.

\begin{figure}[t]%
  \begin{subfigure}[b]{0.55\linewidth}
   \begin{subfigure}[t]{0.32\linewidth}
   \centering \def\svgwidth{0.99\linewidth}
   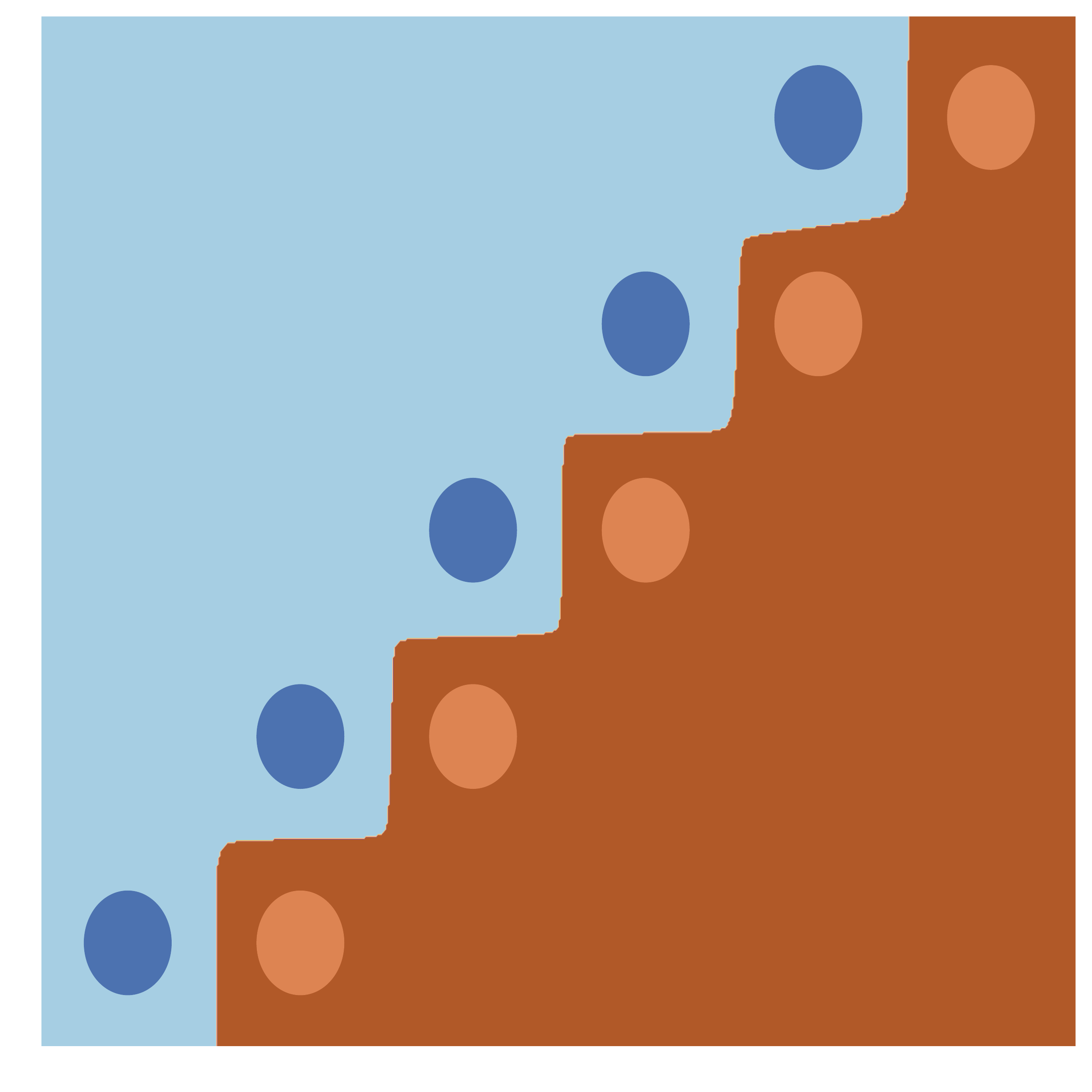
   \caption*{\AT}
  \end{subfigure}
   \begin{subfigure}[t]{0.32\linewidth}
    \centering \def\svgwidth{0.99\linewidth}
    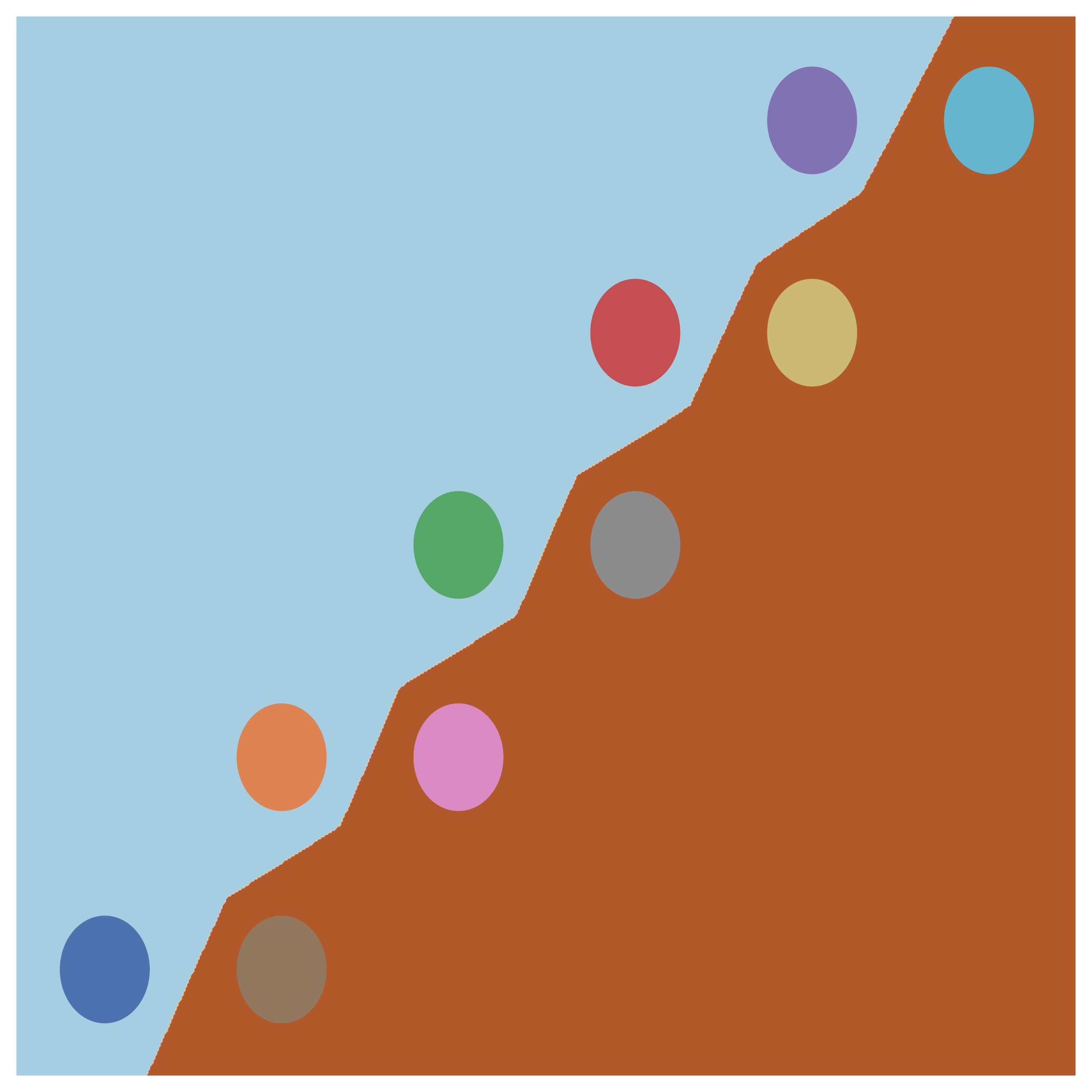
    \caption*{MULTICLASS}
  \end{subfigure}
  \begin{subfigure}[t]{0.32\linewidth}
    \centering \def\svgwidth{0.99\linewidth}
    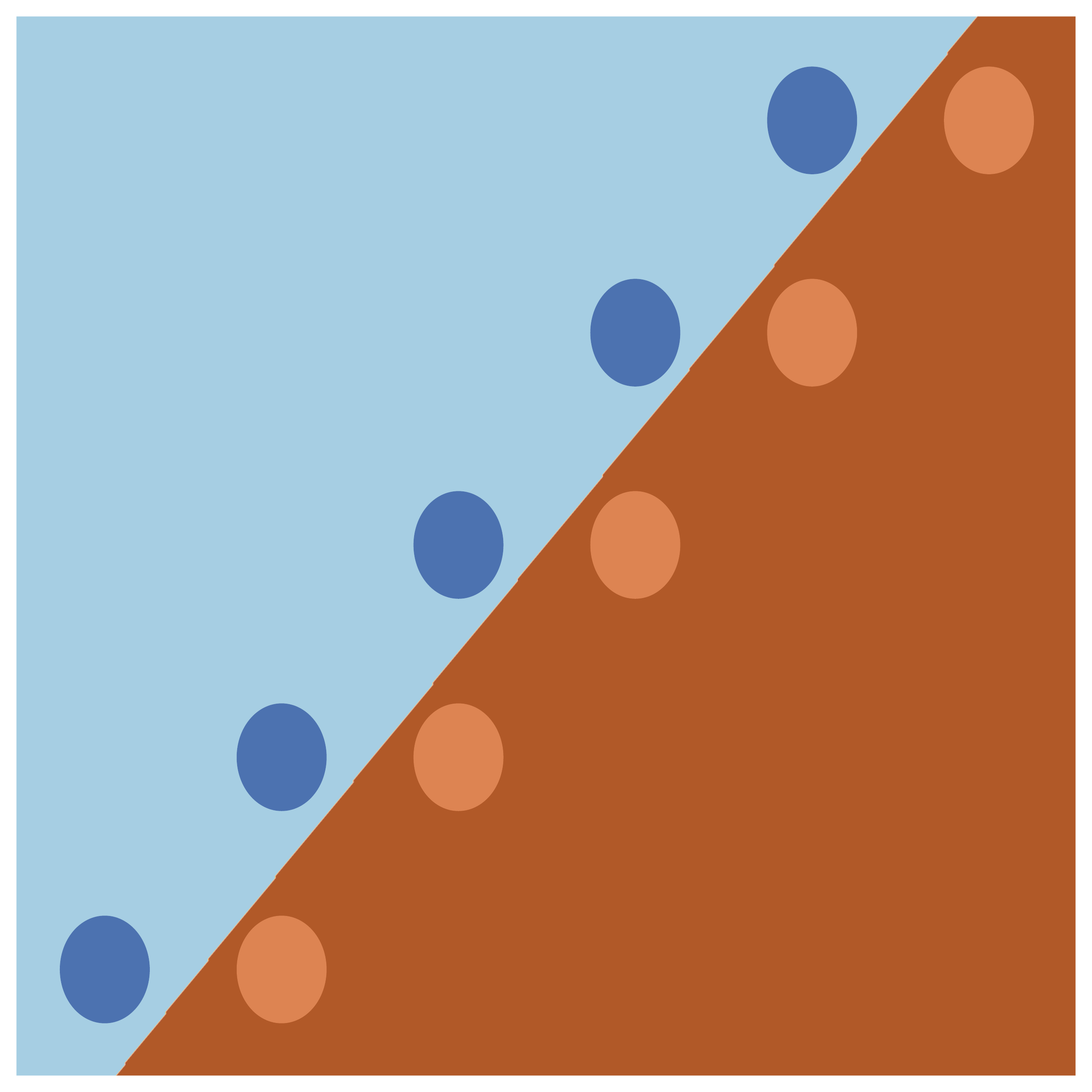
    \caption*{NATURAL}
  \end{subfigure}
  \caption{Decision Region of neural networks are more complex for
    adversarially trained models. Treating it as a multi-class
    classification problem, with natural training~(MULTICLASS), also increases robustness by increasing the
    margin.}
  \label{fig:mc_parity}
\end{subfigure}\hfill
\begin{subfigure}[b]{0.42\linewidth}
    \centering \def\svgwidth{0.85\linewidth}
    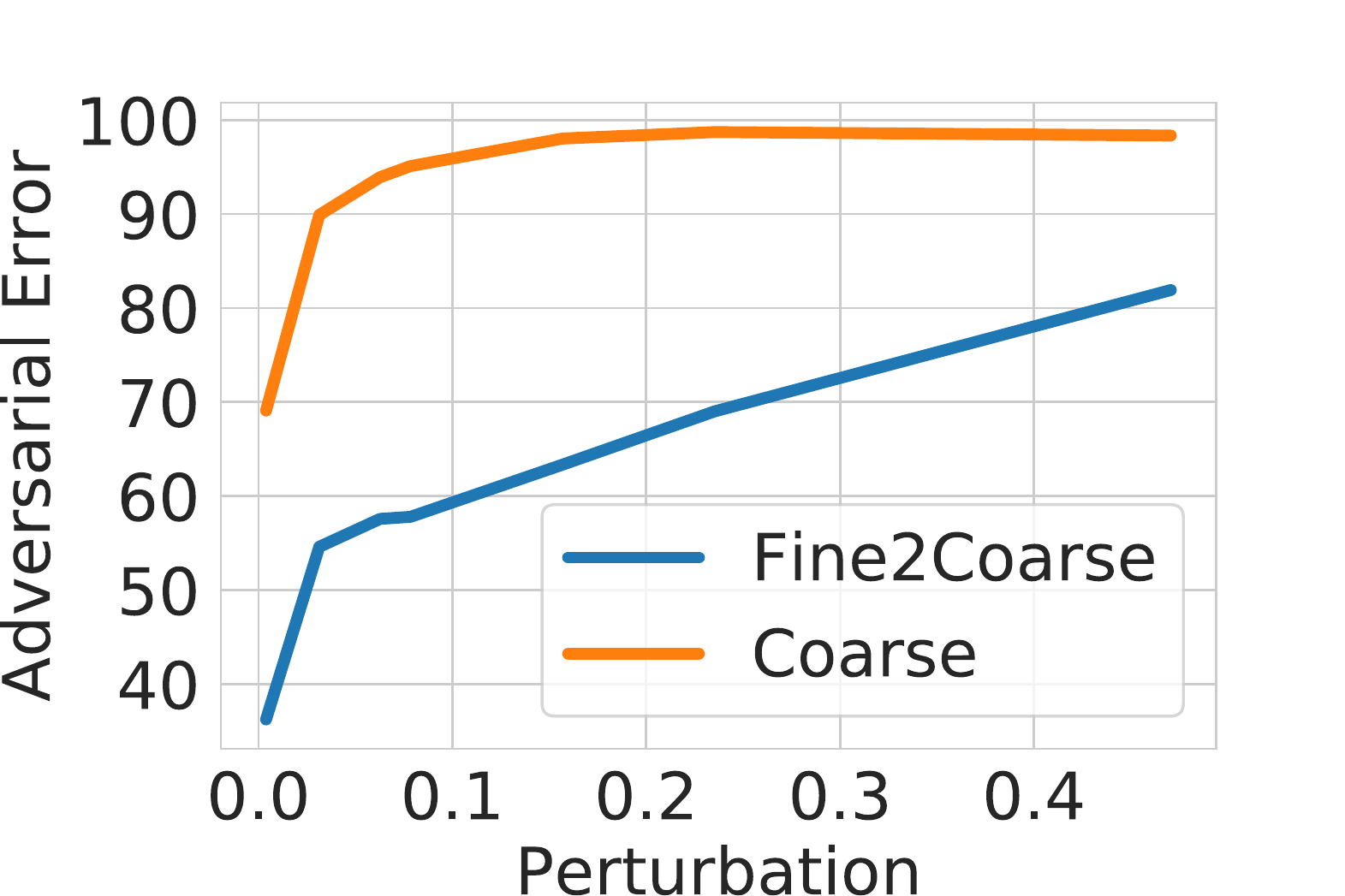
    \caption{\footnotesize Adversarial error on coarse labels of CIFAR-100. %
    }
    \label{fig:fine2coarse}
  \end{subfigure}
  \caption{Assigning a separate class to each sub-population within
    the original class during training
  increases robustness by learning more meaningful representations.}
\end{figure}

\section{Conclusion}
\label{sec:conclusion}
Recent research has largely shone a positive light on interpolation~(zero
training error) by highly over-parameterized models even in the presence of
label noise. While overfitting noisy data may not harm generalisation, we have
shown that this can be severely detrimental to robustness. This raises a new
security  threat where label noise can be inserted into datasets to make the
models learnt from them vulnerable to adversarial attacks without
hurting their test accuracy. As a result, further research into learning without
memorization is ever more important~\citep{sanyal2020stable,shen19e}. Further, we
underscore the importance of proper representation learning in regards to
adversarial robustness. Representations learnt by deep networks often encode a
lot of different invariances, e.g., location, permutation, rotation, etc. While
some of them are useful for the particular task at hand, we highlight that
certain invariances can increase adversarial vulnerability. Thus we believe that
making significant progress towards training robust models with good test error
requires us to rethink representation learning and closely examine the
data on which we are training these models.
\section{Acknowledgement}
\label{sec:ack}
We thank Vitaly Feldman and Chiyuan Zhang for providing us with data
that helped to significantly speed up some parts of this work. We also
thank Nicholas Lord for feedback on the draft. AS acknowledges support
from The Alan Turing Institute under the Turing Doctoral Studentship
grant TU/C/000023. VK is supported in part by the Alan Turing Institute under the EPSRC grant EP/N510129/1.
PHS and PD are supported by the ERC grant ERC-2012-AdG 321162-HELIOS, EPSRC grant Seebibyte
EP/M013774/1 and EPSRC/MURI grant EP/N019474/1. PHS and PD also acknowledges the Royal Academy
of Engineering and FiveAI.
\clearpage
\appendix
\section{Proofs for~\Cref{sec:theoretical-setting}}
\label{sec:proofs-all}
In this section, we present the formal proofs to the theorems stated
in~\Cref{sec:theoretical-setting}.
\subsection{Proof of~\Cref{thm:inf-label}}
\label{sec:proof-21}
\infectedballs*
\begin{proof}[Proof of~\cref{thm:inf-label}]
  From~\cref{eq:balls_density}, for any $\zeta$ and $s\in\zeta$,
  \[\bP_{\br{\vec{x},y}\sim\cD}\bs{\vec{x}\in\cB_{\rho}\br{s}}\ge
    \frac{c_2}{\abs{\zeta}}\]
  As the sampling of the point and the injection of label noise are
  independent events,
  \[\bP_{\br{\vec{x},y}\sim\cD}\bs{\vec{x}\in\cB_{\rho}\br{s}\wedge \vec{x}~\text{gets mislabelled}}\ge
    \frac{c_2\eta}{\abs{\zeta}}\]
  Thus,
  \begin{align*}
    \bP_{\cS_m\sim\cD^m}\bs{\exists\br{\vec{x},y}\in\cS_m:
    \vec{x}\in\cB_{\rho}\br{s}\wedge \vec{x}~\text{is
    mislabelled}}&\ge 
                   1 - \br{1-\frac{c_2\eta}{\abs{\zeta}}}^m\\
                 &\ge 1 - \exp\br{\frac{-c_2\eta m}{\abs{\zeta}}}\\
  \end{align*}
  Substituting $m\ge\frac{\abs{\zeta}}{\eta
    c_2}\log\br{\frac{\abs{\zeta}}{\delta}}$ and applying the union
  bound over all $s\in\zeta$, we get
  \begin{equation}\label{proof:inf-label-1}
    \bP_{\cS_m\sim\cD^m}\bs{\forall s\in\zeta,~\exists\br{\vec{x},y}\in\cS_m:
    \vec{x}\in\cB_{\rho}\br{s}\wedge \vec{x}~\text{is
    mislabelled}}\ge 
1 - \delta \end{equation}
As for all
$\vec{s}\in\reals^d$ and $\forall\vec{x},\vec{z},\in\cB_\rho^p\br{\vec{s}},~\norm{\vec{x}-\vec{z}}_p\le2\rho$, we
have that
\begin{align*}
  \radv{2\rho}{f;\cD}&=\bP_{\cS_m\sim\cD^m}\bs{\bP_{\br{\vec{x},y}\sim\cD}\bs{\exists\vec{z}\in\cB_{2\rho}\br{\vec{x}}~\wedge y\neq f\br{\vec{z}}}}\\
                     &=\bP_{\cS_m\sim\cD^m}\bs{\bP_{\br{\vec{x},y}\sim\cD}\bs{\exists\vec{z}\in\cB_{2\rho}\br{\vec{x}}~\wedge c\br{\vec{z}}\neq
                       f\br{\vec{z}}}}\\
                     &\ge\bP_{\cS_m\sim\cD^n}\bs{\bP_{\br{\vec{x},y}\sim\cD}\bs{\vec{x}\in\bigcup_{s\in\zeta}\cB_{\rho}^p\br{s}\wedge\bc{\exists\vec{z}\in\cB_{2\rho}\br{\vec{x}}:
                       c\br{\vec{z}}\neq f\br{\vec{z}}}}}\\
                     &=\bP_{\cS_m\sim\cD^m}\bs{\bP_{\br{\vec{x},y}\sim\cD}\bs{\exists\vec{s}\in\zeta:\vec{x}\in\cB_{\rho}^p\br{s}\wedge\bc{\exists\vec{z}\in\cB_{\rho}\br{\vec{s}}: c\br{\vec{z}}\neq f\br{\vec{z}}}}}\\
                     &=\bP_{\br{\vec{x},y}\sim\cD}\bs{\vec{x}\in
                       \bigcup_{s\in\zeta}\cB_\rho^p\br{s}}
                       \quad\text{w.p. atleast}~~1-\delta\\
                     &\ge c_1  \quad\text{w.p.}~~1-\delta 
\end{align*}
where $c$ is the
true concept for the distribution $\cD$. The second equality follows
from the assumptions that each of the balls around $\vec{s}\in\zeta$
are pure in their labels. The second last equality follows
from~\cref{proof:inf-label-1} by using the $\vec{x}$ that is
guaranteed to exist in the ball around $\vec{s}$ and be mis-labelled
with probability atleast $1-\delta$. The last equality from Assumption~\cref{proof:inf-label-1}.
\end{proof}

\subsection{Proofs of~\Cref{sec:bias-towrds-simpler}}
\label{sec:proof-22}

\parityrobustrepre*
\begin{proof}[Proof of~\Cref{thm:parity_robust_repre_all}]
  We define a family of distribution $\cD$, such that each
  distribution in $\cD$ is supported on balls of radius $r$ around
  $\br{i,i}$ and  $\br{i+1,i}$ for positive integers $i$. Either all the balls around
  $\br{i,i}$ have the labels $1$ and the balls around $\br{i+1,i}$ have
  the label $0$ or vice versa. ~\cref{fig:complex_simple} shows an
  example where the colors indicate the label.

  Formally, for $r>0$,
  $k\in\bZ_+$, the $\br{r,k}$-1 bit parity class 
    conditional model is defined over
    $\br{x,y}\in\reals^2\times\bc{0,1}$ as follows. First, a label $y$
    is sampled uniformly from $\bc{0,1}$, then and integer $i$ is
    sampled uniformly from the set $\bc{1,\cdots,k}$ and finally
    $\vec{x}$ is generated by sampling uniformly from the $\ell_2$
    ball of radius $r$ around $\br{i+y,i}$.

     In~\cref{thm:linear_parity_all} we first show that a set of $m$
     points sampled iid from any distribution as defined above for
     $r<\frac{1}{2\sqrt{2}}$ is with probability $1$ linear separable
     for any $m$. In   addition, standard VC bounds show that any linear classifier that
     separates $S_{m}$ for large enough $m$ will have
     small test error.~\Cref{thm:linear_parity_all} also proves that
     there exists a range of $\gamma,r$ such that for any distribution
     defined with $r$ in that range, though it is possible to obtain a
     linear classifier with $0$ training and test error, the  minimum
     adversarial risk will be bounded from $0$.

      However while it is possible to obtain a linear classifier with
      $0$ test error,  all such linear classifiers has a large
      adversarial vulnerability. In~\Cref{thm:parity_robust}, we show that there exists a
      different representation for this problem, which also achieves
      zero training and  test error and in addition has zero
      adversarial risk for a range of $r,\gamma$ where the linear
      classifier's adversarial error was atleast a constant.

\end{proof}

\begin{restatable}[Linear Classifier]{lem}{linearparityall}\label{thm:linear_parity_all}

   There exists universal  constants $\gamma_0,\rho$, such
   that for any perturbation $\gamma>\gamma_0$,
     radius $r\ge\rho$, and $k\in\bZ_+$, the following holds. Let $\cD$ be the family of $\br{r,k}$-
  1-bit parity class conditional model, $\cP\in\cD$ and
  $\cS_n=\bc{\br{\vec{x}_1,y_1},\cdots,\br{\vec{x}_n,y_1}}$ be a set
  of $n$ points sampled i.i.d. from
  $\cP$. 
\begin{compactitem}
\item[1)]  For any $n>0$, $S_n$ is linearly separable with probability $1$
  i.e. there exists a $h:\br{\vec{w},w_0}$,
  $\vec{w}\in\reals^2,w_0\in\reals$ such that the linear hyperplane
  $\vec{x}\rightarrow\vec{w}^\top\vec{x}+w_0$ separates $\cS_n$ with
  probability $1$:
  \[\forall \br{\vec{x},y}\in\cS_n\quad
    z\br{\vec{w}^\top\vec{x}+w_0}>0\quad\text{where}~ z=2y-1\]

\item[2)] Further there exists an universal constant $c$ such that for
  any $\epsilon,\delta>0$ with
  probability $1-\delta$ for any $\cS_n$ with
  $n=c\frac{1}{\epsilon^2}\log\frac{1}{\delta}$, any linear classifier
  $\tilde{h}$ that separates $\cS_n$ has
  $\risk{\cP}{\tilde{h}}\le\epsilon$.
 \item [3)] Let $h:\br{\vec{w},w_0}$ be any linear classifier that has
   $\risk{\cP_P}{h}=0$. Then, $\radv{\gamma}{h;\cP}>0.0005$.
\end{compactitem}
\end{restatable}

 We will prove the first part for any $r<\frac{1}{2\sqrt{2}}$ by
 constructing a $\vec{w},w_0$ such that it 
  satisfies the constraints of linear separability. Let
  $\vec{w}=\br{1,-1},~w_0=-0.5$. Consider any point
  $\br{\vec{x},y}\in\cS_n$ and $z=2y-1$. Converting to the polar coordinate system
  there exists a $\theta\in\bs{0,2\pi},j\in\bs{0,\cdots,k}$ such that
  $\vec{x}=\br{j+\frac{z+1}{2}+r\mathrm{cos}\br{\theta},j+r\mathrm{sin}\br{\theta}}$ 
  \begin{align*}
    z\br{\vec{w}^\top\vec{x}+w_0}&=z\br{j+\frac{z+1}{2}+r\mathrm{cos}\br{\theta}-j
                                   - r\mathrm{sin}\br{\theta}-
                                   0.5}&&\vec{w}=\br{1,-1}^\top\\
                                 &=z\br{\frac{z}{2}+0.5+r\mathrm{cos}\br{\theta} -
                                   r\mathrm{sin}\br{\theta}-0.5}\\
                                 &=\frac{1}{2} +
                                   zr\br{\mathrm{cos}\br{\theta}-\mathrm{sin}\br{\theta}}&&\abs{\mathrm{cos}\br{\theta}-\mathrm{sin}\br{\theta}}<\sqrt{2},~~z\in\bc{-1,1}\\
                                 &>\frac{1}{2} - r\sqrt{2}\\
                                 &>0 &&r<\frac{1}{2\sqrt{2}}
  \end{align*}

  Part 2 follows with simple
  VC bounds of linear classifiers.
  
  Let the universal constants $\gamma_0,\rho$ be $0.02$ and
  $\frac{1}{2\sqrt{2}}-0.008$ respectively. Note that there is nothing
  special about this constants except that \emph{some} constant is
  required to bound the adversarial risk away from $0$.  Now, consider
  a distribution $\cP$ 1-bit parity model such that the 
  radius of each ball is atleast $\rho$. This is
  smaller than $\frac{1}{2\sqrt{2}}$ and thus satisfies the linear
  separability criterion.

  Consider $h$ to be a hyper-plane that has $0$ test error.  Let the
  $\ell_2$ radius of adversarial perturbation be
  $\gamma>\gamma_0$. The region of each circle that will be vulnerable
  to the attack will be a circular segment with the chord of the
  segment parallel to the hyper-plane. Let the minimum height of
  all such circular segments be $r_0$. Thus,
  $\radv{\gamma}{h;\cP}$ is greater than the mass of the circular
  segment of radius $r_0$. Let the radius of each ball in the support
  of $\cP$ be $r$.

  Using the fact that $h$ has zero test error; and thus classifies the
  balls in the support of $\cP$ correctly and simple geometry

  \begin{align}
    \frac{1}{\sqrt{2}}&\ge r +\br{\gamma-r_0}+r\nonumber\\
    r_0&\ge 2r + \gamma- \frac{1}{\sqrt{2}}\label{eq:radii}
  \end{align}
  To compute
  $\radv{\gamma}{h;\cP}$ we need to compute the ratio of the area of a circular
  segment of height $r_0$ of a circle of radius $r$ to the area of the
  circle. The ratio  can be written 

  \begin{align}\label{eq:circ-seg}
    A\br{\frac{r_0}{r}} =\frac{{cos}^{-1}\br{1-\frac{r_0}{r}} - \br{1 -
        \frac{r_0}{r}}\sqrt{2\frac{r_0}{r} - \frac{r_0^2}{r^2}}}{\pi}
  \end{align}

  As~\Cref{eq:circ-seg} is increasing with $\frac{r_0}{r}$, we can evaluate 
  \begin{align*}\label{eq:c_1_eq}
    \frac{r_0}{r}&\ge\frac{2r - \frac{1}{\sqrt{2}}+\gamma}{r}&&\text{Using}~\Cref{eq:radii}\\
                 &\ge 2 - \frac{\frac{1}{\sqrt{2}}-0.02}{r}&&\gamma
                                                              >\gamma_0
    = 0.02\\
                 &\ge 2 -
                   \frac{\frac{1}{\sqrt{2}}-0.02}{\frac{1}{\sqrt{2}}-0.008}>0.01&&r>\rho=\frac{1}{2\sqrt{2}}-0.008
  \end{align*}
  Substituting $\frac{r_0}{r}>0.01$ into Eq.~\Cref{eq:circ-seg}, we
  get that $A\br{\frac{r_0}{r}}>0.0005$. Thus, for all
  $\gamma>0.02$, we have $\radv{\gamma}{h;\cP}>0.0005$.

\begin{restatable}[Robustness of parity classifier]{lem}{parityrobust}\label{thm:parity_robust}
  There exists a concept class $\cH$  such that for any
  $\gamma\in\bs{\gamma_0,\gamma_0+\frac{1}{8}}$,
  $k\in\bZ_+$, $\cP$ being the 
  corresponding $\br{\rho,k}$ 1-bit parity class distribution where
  $\rho,\gamma_0$ are the same as in~\Cref{thm:linear_parity_all} there
  exists $g\in\cH$ such that
  \[\risk{\cP}{g} =  0\qquad\radv{\gamma}{g;\cP}=0\]
\end{restatable}

\begin{proof}[Proof of~\Cref{thm:parity_robust}]
  We will again provide a proof by construction.  Consider the
  following class of concepts $\cH$ such that $g_b\in\cH$ is defined
  as \begin{equation}
  \label{eq:2}
  g\br{\br{x_1,x_2}^\top}=\begin{cases}
      1&\text{if} \bs{x_1}+\bs{x_2}  =b \br{\text{mod 2}}\\
      1-b &\text{o.w.}
  \end{cases}
\end{equation} where $\bs{x}$ rounds $x$ to the nearest integer and
$b\in\bc{0,1}$. In~\Cref{fig:complex_simple}, the
green staircase-like classifier belongs to this class. Consider the
classifier $g_1$. Note that by construction $\risk{\cP}{g_1}=0$. The
decision boundary of $g_1$ that are closest to a ball in the support
of $\cP$ centered at $\br{a,b}$ are the lines $x=a\pm 0.5$ and
$y=b\pm 0.5$.

As $\gamma<\gamma_0 + \frac{1}{8}$, the adversarial perturbation is
upper bounded by $\frac{1}{50} + \frac{1}{8}$. The radius of
the ball is upper bounded by $\frac{1}{2\sqrt{2}}$, and as we noted
the center of the ball is at a distance of $0.5$ from the decision
boundary. If the sum of the maximum adversarial perturbation and the
maximum radius of the ball is less than the minimum distance of the
center of the ball from the decision boundary, then the adversarial
error is $0$. Substituting the values, \[\frac{1}{50} +
  \frac{1}{8} + \frac{1}{2\sqrt{2}}
  < 0.499 <\frac{1}{2} \]
This completes the proof.
\end{proof}

\section{Proof of~\Cref{sec:repr-learn-pres}}
\label{sec:proof23}
\robustpossibleful*
\begin{proof}[Proof of~\cref{thm:repre-par-inter}]
  We will provide a constructive proof to this theorem by
  constructing a distribution $\cD$, two concept classes~$\cC$ and $\cH$
  and provide the ERM algorithms to learn the concepts and then
  use~\cref{lem:parity_repre,lem:uni_int_repre} to complete the proof.

  \textbf{Distribution:} Consider the family of distribution $\cD^n$
  such that $\cD_{S,\zeta}\in\cD^n$ is defined on $\cX_\zeta\times\bc{0,1}$ for
  $S\subseteq\bc{1,\cdots,n},\zeta\subseteq\bc{1,\cdots,2^n-1}$  such that the support of $\cX_\zeta$ is a union of
  intervals. 
  \begin{equation}
    \label{eq:dist_union_int}
    \supp{\cX}_\zeta=\bigcup_{j\in\zeta}I_j\text{ where }
    I_j:=\br{j-\frac{1}{4}, j+\frac{1}{4}}
  \end{equation}
  We consider distributions with a relatively small
  support i.e. where $\abs{\zeta}=\bigO{n}$. Each sample $\br{\vec{x},y}~\sim\cD_{S,\zeta}$ is created by sampling
  $\vec{x}$ uniformly from $\cX_\zeta$ and assigning $y=c_S\br{\vec{x}}$ where
  $c_S\in\cC$ is defined below~\cref{eq:parity_concept}. We define the
  family of distributions $\cD =
  \bigcup_{n\in\bZ_+}\cD^n$.  Finally, we create
  $\cD_{S,\zeta}^\eta$ -a noisy version of $\cD_{S,\zeta}$, by flipping $y$ in each sample
  $\br{x,y}$ with probability $\eta<\frac{1}{2}$. Samples from
  $\cD_{S,\zeta}$ can be obtained using the example oracle
  $\mathrm{EX}\br{\cD_{S,\zeta}}$ and samples from the noisy
  distribution can be obtained through the noisy oracle $\mathrm{EX}^\eta\br{\cD_{S,\zeta}}$
  
  \textbf{Concept Class $\cC$:} We define the  concept class $\cC^n$ of concepts
  $c_S:\bs{0,2^n}\rightarrow
  \bc{0,1}$ such that
  \begin{equation}
    \label{eq:parity_concept}
    c_S\br{\vec{x}}=\begin{cases}
      1,
      &\text{if}\br{\langle\bs{\vec{x}}\rangle_b~\mathrm{XOR}~S}~\text{
        is odd.}\\
      0 &~\text{o.w.}
    \end{cases}
  \end{equation}
  where $\bs{\cdot}:\reals\rightarrow\bZ$ rounds a decimal
  to its nearest
  integer,~$\langle\cdot\rangle_b:\bc{0,\cdots,2^n}\rightarrow\bc{0,1}^n$
  returns the binary encoding of the integer,~and
  $\br{\langle\bs{\vec{x}}\rangle_b~\textrm{XOR}~S} = \sum_{j\in S}
  \langle\bs{x}\rangle_b\bs{j}~\textrm{mod}~2$. $\langle\bs{x}\rangle_b\bs{j}$
  is the $j^{\it th}$ least significant bit in the binary encoding of
  the nearest integer to $\vec{x}$. It is essentially the
  class of parity functions  defined on the bits corresponding to the
  indices in $S$ for the binary
  encoding of the nearest  integer to $\vec{x}$. For example, as
in~\Cref{fig:thm-3} if $S = \{0, 2\}$, then only the least significant and
the third least significant bit of $i$ are examined and the class label is
$1$ if an odd number  of them are $1$ and $0$ otherwise.

  \textbf{Concept Class $\cH$:} Finally, we define the concept class
  $\cH=\bigcup_{k=1}^\infty\cH_k$ where $\cH_k$ is the class of union of
$k$ intervals on the real line  $\cH^k$. Each concept $h_I\in\cH^k$
can be written as a set of $k$ disjoint intervals
$I=\bc{I_1,\cdots,I_k}$ on the real line i.e. for $1\le j\le k$,
$I_j=\bs{a,b}$ where $0\le a\le b$ and
\begin{equation}
  \label{eq:union_int-defn_class}
  h_I\br{\vec{x}} =\begin{cases}
    1&\text{if}~\vec{x}\in\bigcup_j I_j\\
    0&\text{o.w.}
  \end{cases}
\end{equation}

Now, we look at the algorithms to learn the concepts from $\cC$ and
$\cH$ that minimize the train error. Both of the algorithms will use a
majority vote to determine the correct~(de-noised) label for each interval, which
will be necessary to minimize the test error. The intuition is that if
we draw a sufficiently large number of samples, then the majority of
samples on each interval will have the correct label with a high
probability. 

~\Cref{lem:parity_repre} proves that there exists an algorithm $\cA$
such that $\cA$ draws
$m=\bigO{\abs{\zeta}^2\frac{\br{1-\eta}}{\br{1-2\eta}^2}\log{\frac{\abs{\zeta}}{\delta}}}$
samples from the noisy oracle $\mathrm{EX}^\eta\br{\cD_{s,\zeta}}$ and with probability $1-\delta$
where the probability is over the randomization in the oracle, returns
$f\in\cC$ such that $\risk{\cD_{S,\zeta}}{f}=0$ and 
$\radv{\gamma}{f;\cD_{S,\zeta}}=0$ for all
$\gamma<\frac{1}{4}$. As~\Cref{lem:parity_repre} states, the algorithm
involves gaussian elimination over $\abs{\zeta}$ variables and
$\abs{\zeta}$ majority votes~(one in each interval) involving a total
of $m$ samples. Thus the
algorithm runs in $\bigO{\poly{m}+\poly{\abs{\zeta}}}$ time. Replacing
the complexity of $m$ and the fact that $\abs{\zeta}=\bigO{n}$, the
complexity of the algorithm is
$\bigO{\poly{n,\frac{1}{1-2\eta},
\frac{1}{\delta}}}$.  

~\Cref{lem:uni_int_repre} proves that there
 exists an algorithm $\widetilde{A}$ such that $\widetilde{A}$ draws \[m>\mathrm{max}\bc{
      2\abs{\zeta}^2\log{\frac{2\abs{\zeta}}{\delta}} 
   \br{8\frac{\br{1-\eta}}{\br{1-2\eta}^2}+1},
   \frac{0.1\abs{\zeta}}{\eta\gamma^2} 
  \log\br{\frac{0.1\abs{\zeta}}{\gamma\delta}}}\] samples and returns
$h\in\cH$ such that $h$ has $0$ training error, $0$ test error and an
adversarial test error of atleast $0.1$. We can replace $\abs{\zeta} =
\bigO{n}$ to get the required bound on $m$ in the theorem. The
algorithm to construct $h$ visits every point atmost twice - once
during the construction of the intervals using majority voting, and
once while accommodating for the mislabelled points.  Replacing
the complexity of $m$, the
complexity of the algorithm is  $\bigO{\poly{n,\frac{1}{1-2\eta},\frac{1}{\gamma},\frac{1}{\delta}}}$. This completes the proof.
\end{proof}

\begin{lem}[Parity Concept Class]\label{lem:parity_repre}
  There exists a  learning algorithm $\cA$ such that given
  access to the noisy example oracle
  $\mathrm{EX}^\eta\br{\cD_{S,\zeta}}$, $\cA$ makes
  $m=\bigO{\abs{\zeta}^2\frac{\br{1-\eta}}{\br{1-2\eta}^2}\log{\frac{\abs{\zeta}}{\delta}}}$
  calls to the oracle and returns a
  hypothesis $f\in\cC$  such that with probability
  $1-\delta$, we have that $\risk{\cD_{S,\zeta}}{f}=0$ and
  $\radv{\gamma}{f;\cD_{S,\zeta}}=0$ for all $\gamma<\frac{1}{4}$.
\end{lem}

\begin{proof}
  The algorithm $\cA$ works as follows. It  makes $m$ calls to the oracle
  $\mathrm{EX}\br{\cD_s^m}$ to obtain a set of
  points~$\bc{\br{x_1,y_1},\cdots,\br{x_m,y_m}}$ where
  $m\ge 2\abs{\zeta}^2\log{\frac{2\abs{\zeta}}{\delta}}\br{8\frac{\br{1-\eta}}{\br{1-2\eta}^2}+1}$
  . Then, it replaces each $x_i$ with $\bs{x_i}$~($\bs{\cdot}$ rounds a
  decimal to the nearest integer) and then removes duplicate
  $x_i$s by preserving the most frequent label $y_i$ associated with each
  $x_i$.
  For example, if $\cS_5 = \bc{\br{2.8,1}, \br{2.9, 0}, \br{3.1,
      1},\br{3.2, 1}, \br{3.9, 0}}$ then after this operation, we will
  have $\bc{\br{3,1}, \br{4,0}}$.

   As   $m\ge 2\abs{\zeta}^2\log{\frac{2\abs{\zeta}}{\delta}}
   \br{8\frac{\br{1-\eta}}{\br{1-2\eta}^2}+1}$, using
   $\delta_2=\frac{\delta}{2}$ and
   $k=\frac{8\br{1-\eta}}{\br{1-2\eta}^2}\log\frac{2\abs{\zeta}}{\delta}$
   in
   ~\cref{lem:min-wt} guarantees that with probability $1-\frac{\delta}{2}$, each
  interval will have atleast
  $\frac{8\br{1-\eta}}{\br{1-2\eta}^2}\log\frac{2\abs{\zeta}}{\delta}$
  samples.

  Then for any specific interval, using
  $\delta_1=\frac{2\abs{\zeta}}{\delta}$ in ~\cref{lem:majority_lem} guarantees that with
  probability atleast $1-\frac{2\abs{\zeta}}{\delta}$, the majority
  vote for the label in that interval will succeed in returning
  the de-noised label. Applying a union bound over all $\abs{\zeta}$ intervals, will
  guarantee that with probability atleast $1-\delta$, the majority
  label of every interval will be the denoised label.

  Now, the problem reduces to solving a parity problem on this reduced
  dataset of $\abs{\zeta}$ points~(after denoising, all points in that
  interval can be reduced to the integer in the interval and the
  denoised label). We know that there exists a polynomial
  algorithm using Gaussian Elimination that finds a consistent
  hypothesis for this problem. We have already guaranteed that there is a
  point in $\cS_m$ from every interval in the 
  support of $\cD_{S,\zeta}$. Further, $f$ is consistent on $\cS_m$ and $f$ is
  constant in each of these intervals by design. Thus, with
  probability atleast  $1-\delta$ we have that  $\risk{\cD_{S,\zeta}}{f}=0$.

  By construction, $f$  makes a constant
  prediction on each interval $\br{j-\frac{1}{2},j+\frac{1}{2}}$ for
  all $j\in\zeta$. Thus, for any perturbation radius
  $\gamma<\frac{1}{4}$ the adversarial risk
  $\radv{\cD_{S,\prime{\zeta}}}{f}=0$. Combining everything, we have shown that there is an algorithm
  that makes $2\abs{\zeta}^2\log{\frac{2\abs{\zeta}}{\delta}}\br{8\frac{\br{1-\eta}}{\br{1-2\eta}^2}+1}$ calls to the
  $\mathrm{EX}\br{\cD_{S,\zeta}^\eta}$ oracle,  runs in time polynomial in $\abs{\zeta},\frac{1}{1-2\eta},\frac{1}{\delta}$ to return
  $f\in\cC$ such that $\risk{\cD_{S,\zeta}}{f}=0$ and
  $\radv{\gamma}{f;\cD_{S,\zeta}}=0$ for $\gamma<\frac{1}{4}$.
\end{proof}

\begin{lem}[Union of Interval Concept Class]\label{lem:uni_int_repre}
   There exists a  learning algorithm $\widetilde{\cA}$ such that given
  access to a noisy example oracle makes
  $m=\bigO{\abs{\zeta}^2\frac{\br{1-\eta}}{\br{1-2\eta}^2}
    \log{\frac{\abs{\zeta}}{\delta}}}$  calls to the oracle and
  returns a hypothesis $h\in\cH$ 
  such that training error is $0$ and with probability
  $1-\delta$, $\risk{\cD_{S,\zeta}}{f}=0$.

  Further for any $h\in\cH$ that has zero training error on
  $m^\prime$ samples drawn from $\mathrm{EX}^\eta\br{\cD_{S,\zeta}}$
  for $m^\prime >  \frac{\abs{\zeta}}{10\eta\gamma^2}
  \log\frac{\abs{\zeta}}{10\gamma\delta}$   and
  $\eta\in\br{0,\frac{1}{2}}$ then 
  $\radv{\gamma}{f;\cD_{S,\zeta}}\ge 0.1$ for all $\gamma>0$.
\end{lem}

\begin{proof}[Proof of~\Cref{lem:uni_int_repre}]
  The first part of the algorithm works similarly
  to~\Cref{lem:parity_repre}. The algorithm $\widetilde{\cA}$ makes
  $m$ calls to the oracle  $\mathrm{EX}\br{\cD_s^m}$ to obtain a set of
  points~$\cS_m = \bc{\br{x_1,y_1},\cdots,\br{x_m,y_m}}$ where
  $m\ge 2\abs{\zeta}^2\log{\frac{2\abs{\zeta}}{\delta}}
  \br{8\frac{\br{1-\eta}}{\br{1-2\eta}^2}+1}$. $\widetilde{\cA}$
  computes $h\in\cH$ as follows. To begin, let the list of
  intervals in $h$ be $I$ and $\cM_z=\bc{}$ Then do the following for every
  $\br{x,y}\in\cS_m$.
  \begin{enumerate}[leftmargin=0.5cm,itemsep=0ex]
  \item let $z := \bs{x}$, 
  \item Let $\cN_z\subseteq\cS_m$ be the set of all $\br{x,y}\in\cS_m$ such that
    $\abs{x-z}<0.5$.
  \item Compute the majority label $\tilde{y}$ of $\cN_z$.
  \item Add all $\br{x,y}\in\cN_z$ such that $y\neq \tilde{y}$ to $\cM_z$
  \item If $\tilde{y}=1$, then add the interval $(z-0.5,z+0.5)$ to $I$.
  \item Remove all elements of $\cN_z$ from $\cS_m$ i.e. $\cS_m:=\cS_m\setminus\cN_z$.
  \end{enumerate}
For reasons similar to~\Cref{lem:parity_repre}, as   $m\ge 2\abs{\zeta}^2\log{\frac{2\abs{\zeta}}{\delta}}
   \br{8\frac{\br{1-\eta}}{\br{1-2\eta}^2}+1}$, 
   ~\cref{lem:min-wt} guarantees that with probability $1-\frac{\delta}{2}$, each
  interval will have atleast
  $\frac{8\br{1-\eta}}{\br{1-2\eta}^2}\log\frac{2\abs{\zeta}}{\delta}$
  samples. Then for any specific interval, ~\cref{lem:majority_lem} guarantees that with
  probability atleast $1-\frac{2\abs{\zeta}}{\delta}$, the majority
  vote for the label in that interval will succeed in returning
  the de-noised label. Applying a union bound over all intervals, will
  guarantee that with probability atleast $1-\delta$, the majority
  label of every interval will be the denoised label. As each interval
  in$\zeta$ has atleast one point, all the intervals in $\zeta$ with
  label $1$ will  be included in $I$ with probability
  $1-\delta$. Thus, $\risk{\cD_{S,\zeta}}{h}=0$.

  Now, for all $\br{x,y}\in\cM_z$, add the interval $\bs{x}$ to $I$ if
  $y=1$. If $y=0$ then $x$ must lie a interval $(a,b)\in
  I$. Replace that interval as follows $I:= I\setminus(a,b)\cup
  \bc{(a,x),(x,b)}$. As only a finite number of sets with lebesgue
  measure of $0$ were added or deleted
  from $I$, the net test error of $h$ doesn't change and is still
  $0$ i.e.  $\risk{\cD_{S,\zeta}}{h}=0$

  For the second part, we will invoke~\Cref{thm:inf-label}. To avoid
  confusion in notation, we will use $\Gamma$ instead of $\zeta$ to
  refer to the sets in~\Cref{thm:inf-label} and reserve $\zeta$ for
  the support of interval of $\cD_{S,\zeta}$. Let $\Gamma$ be any set of
  disjoint intervals of width $\frac{\gamma}{2}$ such that $\abs{\Gamma}=
  \frac{0.1\abs{\zeta}}{\gamma}$. This is always possible as the total
  width of all intervals in $\Gamma$ is $
  \frac{0.1\abs{\zeta}}{\gamma}\frac{\gamma}{2} =
  0.1\frac{\abs{\zeta}}{2}$ which is less than the total width of the
  support $\frac{\abs{\zeta}}{2}$. $c_1,c_2$ from
  Eq.~\Cref{eq:balls_density} is \[c_1 =
    \bP_{\cD_{S,\zeta}}\bs{\Gamma} =
    \frac{2*0.1\abs{\zeta}}{2\abs{\zeta}} = 0.1,\quad
    c_2=\frac{2\gamma}{2\abs{\zeta}}\abs{\zeta}=\gamma\]

  Thus, if $h$ has an error of zero on a set of $m^\prime$ examples
  drawn from $\mathrm{EX}^{\eta}\br{\cD_{S,\zeta}}$ where $m^\prime>
  \frac{0.1\abs{\zeta}}{\eta\gamma^2}
  \log\br{\frac{0.1\abs{\zeta}}{\gamma\delta}}$, then
  by~\Cref{thm:inf-label}, $\radv{\gamma}{h;\cD_{S,\zeta}}>0.1$.

  Combining the two parts for \[m>\mathrm{max}\bc{
      2\abs{\zeta}^2\log{\frac{2\abs{\zeta}}{\delta}} 
   \br{8\frac{\br{1-\eta}}{\br{1-2\eta}^2}+1},
   \frac{0.1\abs{\zeta}}{\eta\gamma^2} 
  \log\br{\frac{0.1\abs{\zeta}}{\gamma\delta}}}\] it is possible to
obtain $h\in\cH$ such that $h$ has zero training error,
$\risk{h}{\cD_{S,\zeta}}=0$ and $\radv{\gamma}{h;\cD_{S,\zeta}}>0.1$
for any $\gamma>0$.

\end{proof}

\begin{lem}\label{lem:min-wt}
  Given $k\in\bZ_+$ and a distribution $\cD_{S,\zeta}$, for any
  $\delta_2 > 0$ if
  $m>2\abs{\zeta}^2k + 2\abs{\zeta}^2\log{\frac{\abs{\zeta}}{\delta_2}}$ samples are
  drawn from $\mathrm{EX}\br{\cD_{S,\zeta}}$ then with probability
  atleast $1 - \delta_2$ there are atleast $k$ samples in each
  interval $\br{j-\frac{1}{4},j+\frac{1}{4}}$ for all $j\in\zeta$.
\end{lem}
\begin{proof}[Proof of~\cref{lem:min-wt}]
   We will repeat the following procedure $\abs{\zeta}$ times once for
   each interval in $\zeta$ and show that with probability
   $\frac{\delta}{\abs{\zeta}}$ the $j^{\it{th}}$ run will result in
   atleast $k$ samples in the $j^{\it th}$ interval.

   Corresponding to each interval in $\zeta$, we will sample atleast  $m^\prime$ samples where $m^\prime=2\abs{\zeta}k +
  2\abs{\zeta}\log{\frac{\abs{\zeta}}{\delta_2}}$.  If $z_i^j$ is the 
  random variable that is $1$ when the $i^{\it th}$ 
  sample belongs to the $j^{\it th}$ interval, then $j^{\it th}$
  interval has atleast $k$ points out of the $m^\prime$ points sampled
  for that interval with probability  less
  than $\frac{\delta_2}{\abs{\zeta}}$.
  \begin{align*}
    \bP\bs{\sum_i z_{i}^j \le k} &= \bP\bs{\sum_i z_{i}^j \le \br{1 -
    \delta}\mu} &&\delta = 1 - \frac{k}{\mu}, \mu = \bE\bs{\sum_i
                   z_i^j}\\
    &\le \exp\br{-\br{1-\frac{k}{\mu}}^2\frac{\mu}{2}} &&\text{By
                                                       Chernoff's
                                                       inequality}\\
    &\le
      \exp\br{-\br{\frac{m^\prime}{2\abs{\zeta}}-k+\frac{k^2\abs{\zeta}}{2m^\prime}}}
                && \mu=\frac{m^\prime}{\abs{\zeta}}\\
    &\le
      \exp\br{k-\frac{m^\prime}{2\abs{\zeta}}}\le \frac{\delta_2}{\abs{\zeta}}
  \end{align*} where the last step follows from $m^\prime>2\abs{\zeta}k +
  2\abs{\zeta}\log{\frac{\abs{\zeta}}{\delta_2}}$. With probability
  atleast $\frac{\delta}{\abs{\zeta}}$, every interval will have
  atleast $k$ samples. Finally, an union
  bound over each interval gives the desired result. As we repeat the
  process for all $\abs{\zeta}$ intervals, the total
  number of samples drawn will be atleast $\abs{\zeta}m^\prime =
  2\abs{\zeta}^2k +  2\abs{\zeta}^2\log{\frac{\abs{\zeta}}{\delta_2}}$.
\end{proof}

\begin{lem}[Majority Vote]\label{lem:majority_lem}
For a given $y\in\bc{0,1}$, let $S=\bc{s_1,\cdots,s_m}$ be a set of size $m$ where each element is $y$ with
probability $1-\eta$ and $1-y$ otherwise. If
$m>\frac{8\br{1-\eta}}{\br{1-2\eta}^2}\log\frac{1}{\delta_1}$ then with
probability atleast $1-\delta_1$ the majority of $S$ is $y$.
\end{lem}
\begin{proof}[Proof of~\cref{lem:majority_lem}]
  Without loss of generality let $y=1$. For the majority to be $1$ we
  need to show that there are more than $\frac{m}{2}$ ``$1$''s in $S$
  i.e. we need to show that the following probability is less than $\delta_1$.
  \begin{align*}
    \bP\bs{\sum s_i< \frac{m_1}{2}} &= \bP\bs{\sum s_i <
                                    \frac{m_1}{2\mu}*\mu +\mu -
                                      \mu}&&\mu = \bE\bs{\sum s_i}\\
                                    &= \bP\bs{\sum s_i < \br{1 - \br{1
                                      - \frac{m_1}{2\mu}}}\mu}\\
                                    &\le
                                      \exp{\br{-\frac{\br{1-2\eta}^2}{8\br{1-\eta}^2}\mu}}
                                          &&\text{By Chernoff's
                                             Inequality}\\
                                    &=\exp{\br{-\frac{\br{1-2\eta}^2}{8\br{1-\eta}}m}}
                                          &&\because \mu=\br{1-\eta}m\\
                                    &\le \delta_1
                                          &&\because m>\frac{8\br{1-\eta}}{\br{1-2\eta}^2}\log{\frac{1}{\delta_1}}
  \end{align*}
\end{proof}

\clearpage
\bibliographystyle{abbrvnat}
\bibliography{adv_gen}
\end{document}